\documentclass[11pt]{article}

\usepackage[round]{natbib}

\usepackage{latexsym,amssymb,amsmath} % for \Box, \mathbb, split, etc.
\usepackage{path}
\usepackage{verbatim}
\usepackage{amsfonts}
\usepackage{amsfonts}
\usepackage{array}

\usepackage{mathrsfs}
\usepackage{amsfonts}
\usepackage{amsmath}
\usepackage{amssymb}
\usepackage{pifont}
\usepackage{amsthm}
\usepackage{bm}
\usepackage{appendix}
\usepackage{multirow}
\usepackage{pifont}
\usepackage{graphicx} % more modern
\usepackage{subfigure}
\usepackage{color}
\usepackage[table]{xcolor}
\usepackage{thmtools}
\usepackage{thm-restate}
\usepackage{enumitem}

\usepackage{graphicx}
\usepackage{epsfig}
\usepackage{subfigure}

\usepackage{tcolorbox}

\usepackage[pagebackref=true,colorlinks=true,citecolor=blue,bookmarks=false]{hyperref}

\definecolor{gray}{rgb}{0.85,0.85,0.85}
\definecolor{yama}{rgb}{0.98, 0.87, 0.68}
\definecolor{lightskyblue}{rgb}{0.53, 0.81, 0.98}

\graphicspath{{figures/}}

\usepackage[lined,boxed,ruled, commentsnumbered]{algorithm2e}

\newtheorem{lemma}{Lemma}

\newtheorem{proposition}{Proposition}
\newtheorem{definition}{Definition}
\newtheorem{remark}{Remark}
\newtheorem{assumption}{Assumption}

\newcommand{\argmin}{\mathop{\arg\min}}

\newcommand{\tabincell}[2]{\begin{tabular}{@{}#1@{}}#2\end{tabular}}


\newcommand{\junk}[1]{{}}

\newlength{\fwtwo} 
\DeclareGraphicsExtensions{.jpg}

\title{\bf Sharper Analysis for Minibatch Stochastic Proximal Point Methods: Stability, Smoothness, and Deviation\vspace{0.3in}}
\author{
  \textbf{Xiao-Tong Yuan} \ and \ \textbf{Ping Li} \\\\
   Cognitive Computing Lab \\
  Baidu Research \\
   No. 10 Xibeiwang East Road, Beijing 100193, China \\
   10900 NE 8th St. Bellevue, Washington 98004, USA\\
   E-mail: \texttt{\{xtyuan1980, pingli98\}@gmail.com}
  }
\date{\vspace{0.3in}}

\begin{document}

\maketitle

\begin{abstract}\vspace{0.2in}
\noindent The stochastic proximal point (SPP) methods have gained recent attention for stochastic optimization, with strong convergence guarantees and superior robustness to the classic stochastic gradient descent (SGD) methods showcased at little to no cost of computational overhead added. In this article, we study a minibatch variant of SPP, namely M-SPP, for solving convex composite risk minimization problems. The core contribution is a set of novel excess risk bounds of M-SPP derived through the lens of algorithmic stability theory. Particularly under smoothness and quadratic growth conditions, we show that M-SPP with minibatch-size $n$ and iteration count $T$ enjoys an in-expectation fast rate of convergence consisting of an $\mathcal{O}\left(\frac{1}{T^2}\right)$ \emph{bias} decaying term and an $\mathcal{O}\left(\frac{1}{nT}\right)$ \emph{variance} decaying term. %Whilst for arbitrary convex functions, the rate of convergence is dominated by an $\mathcal{O}((nT)^{-1/2})$ variance decaying term.
In the small-$n$-large-$T$ setting, this result substantially improves the best known results of SPP-type approaches by revealing the impact of noise level of model on convergence rate. In the complementary small-$T$-large-$n$ regime, we provide a two-phase extension of M-SPP to achieve comparable convergence rates. Moreover, we derive a near-tight high probability (over the randomness of data) bound on the parameter estimation error of a sampling-without-replacement variant of M-SPP. Numerical evidences are provided to support our theoretical predictions when substantialized to Lasso and logistic regression models.
\end{abstract}

\vspace{0.2in}

\newpage

\section{Introduction}

We consider the following problem of regularized risk minimization over a closed convex subset $\mathcal{W}\subseteq \mathbb{R}^p$:
\begin{equation}\label{equat:problem}
\min_{w \in \mathcal{W}} R(w):= R^{\ell}(w) + r(w), \ \  \text{where } \ R^{\ell}(w):= \mathbb{E}_{z \sim \mathcal{D}} [\ell(w;z)],
\end{equation}
where $\ell:\mathcal{W} \times \mathcal{Z} \mapsto R^+$ is a non-negative convex loss function whose value $\ell(w;z)$ measures the loss of a hypothesis, parameterized by $w\in \mathcal{W}$, evaluated over a data sample $z\in \mathcal{Z}$, $\mathcal{D}$ represents a distribution over $\mathcal{Z}$, and $r: \mathcal{W} \mapsto \mathbb{R}^+$ is a data-independent non-negative convex function whose value $r(w)$ measures certain complexity of the hypothesis. We are particularly interested in the situation where the composite population risk $R$ is strongly convex around its minimizers, though in this setting the terms $R^{\ell}$ and $r$ are not necessarily required to be so simultaneously. For an instance, the $\ell_1$-norm regularizer $r(w)=\mu \|w\|_1$ or its grouped variants are often used for sparse generalized linear models learning with quadratic or logistic loss functions~\citep{van2008high,ravikumar2009sparse, negahban2012unified}.

In statistical machine learning, it is usually assumed that the estimator only has access to, either as a batch training set or in an online/incremental manner, a collection $S=\{z_i\}_{i=1}^N$ of i.i.d. random data instances drawn from $\mathcal{D}$. The goal is to compute a stochastic estimator $\hat w_S$ based on the knowledge of $S$, hopefully that it generalizes well as a near minimizer of the population risk. More precisely, we aim at deriving a suitable law of large numbers, i.e., a sample size vanishing rate $\delta_N$ so that the \emph{excess risk} at $\hat w_S$ satisfies $R(\hat w_S) - R^*\le \delta_N$ in expectation or with high probability over $S$ where $R^*:=\min_{w \in \mathcal{W}}R(w)$ represents the minimal value of composite risk.

In this work, inspired by the recent remarkable success of the stochastic proximal point (SPP) algorithms~\citep{patrascu2017nonasymptotic,asi2019importance,asi2019stochastic,davis2019stochastic} and their minibatch extensions~\citep{wang2017memory,zhou2019efficient,asi2020minibatch}, we provide a sharper generalization performance analysis for a class of minibatch SPP methods for solving the stochastic composite risk minimization problem~\eqref{equat:problem}.

\subsection{Algorithm and Motivation of Study}
\label{ssect:algorithm_motivation}

\emph{Minibatch Stochastic Proximal Point Algorithm.} Let $S_t=\{z_{i,t}\}_{i=1}^n$ be a minibatch of $n$ i.i.d. samples drawn from distribution $\mathcal{D}$ at time instance $t\ge 1$ and denote
\[
R_{S_t}(w):=\frac{1}{n}\sum_{i=1}^n \ell(w; z_{i,t}) + r(w)
\]
as the regularized empirical risk over $S_t$. We consider the Minibatch Stochastic Proximal Point (M-SPP) algorithm, as outlined in Algorithm~\ref{alg:mspp}, for composite risk minimization based on a sequence of data minibatches $S=\{S_t\}_{t=1}^T$. The precision value $\epsilon_t$ in the algorithm quantifies the sub-optimality of $w_t$ for solving the inner-loop regularized ERM over the minibatch $S_t$. The M-SPP algorithm is generic and it encompasses several existing SPP methods as special cases. For example in the extreme case when $n=1$ and $\epsilon_t\equiv0$ M-SPP reduces to a composite variant of the standard SPP method~\citep{bertsekas2011incremental}, as formulated in~\eqref{equat:spp}. In general, the recursion update formulation~\eqref{equat:mspp_iteration} can be regarded as a natural composite extension of the existing minibatch stochastic proximal point methods for statistical estimation~\citep{wang2017memory,asi2020minibatch}.

\begin{algorithm}[tb]\caption{\texttt{Minibatch Stochastic Proximal Point (M-SPP)}}
\label{alg:mspp}
\SetKwInOut{Input}{Input}\SetKwInOut{Output}{Output}\SetKw{Initialization}{Initialization}
\Input{Regularization modulus $\{\gamma_t\}_{t\ge 1}$.}
\Output{$\bar w_T$ as a weighted average of $\{w_t\}_{1\le t\le T}$.}
\Initialization{Specify a value of $w_0$. Typically $w_0=0$. }

\For{$t=1, 2, ...,T$}{

Sample a minibatch $S_t:=\{z_{i,t}\}_{i=1}^n \overset{\text{i.i.d.}}{\sim} \mathcal{D}^n$ and estimate $w_t$ satisfying
\begin{equation}\label{equat:mspp_iteration}
F_t(w_t) \le \min_{w\in \mathcal{W}} \left\{F_t(w):=R_{S_t}(w) + \frac{\gamma_t}{2} \|w - w_{t-1}\|^2\right\} + \epsilon_t,
\end{equation}
where $R_{S_t}(w):=\frac{1}{n}\sum_{i=1}^n \ell(w; z_{i,t}) + r(w)$ and $\epsilon_t\ge0$ measures the sub-optimality of estimation.
}
\end{algorithm}

\emph{Prior results and limitations.} The present study focuses on the generalization analysis of M-SPP for convex composite risk optimization. Recently, it has been shown by~\citet[Theorem 2]{asi2020minibatch} that if the instantaneous loss functions are strongly convex with respect to the parameters, then the M-SPP algorithm converges at the rate of $\mathcal{O}\left(\frac{\log(nT)}{nT}\right)$. Prior to that, \citet[Theorem 5]{wang2017memory} proved an $\mathcal{O}(\frac{1}{nT})$ rate for M-SPP when the individual loss functions are Lipschitz continuous and strongly convex. There results, among others for SPP~\citep{patrascu2017nonasymptotic,davis2019stochastic}, commonly require that each instantaneous loss should be strongly convex which is too stringent to be fulfilled in high-dimensional or infinite spaces. For an instance, the quadratic loss $\ell(w; z)=\frac{1}{2}(w^\top x - y)^2$ over a feature-label pair $z=(x,y)$ is convex but in general not strongly convex, although the population risk $R^{\ell}(w)=\frac{1}{2}\mathbb{E}(y-w^\top x)^2$ is strongly convex provided that the covariance matrix of random feature $x$ is non-degenerate. In the meanwhile, the Lipschitz-loss assumption made for the analysis~\cite[Theorem 5]{wang2017memory} limits its applicability to smooth losses like quadratic loss, not to mention an interaction between Lipschitz continuity and strong convexity~\citep{agarwal2012information,asi2019stochastic}.

The above mentioned deficiencies of prior results motivate us to investigate the convergence behavior of M-SPP for composite risk minimization beyond the setting where each individual loss is strongly convex and Lipschitz continuous. From the perspective of optimization, smoothness is essential for establishing strong convergence guarantees for solving the inner-loop strongly convex risk minimization subproblems in~\eqref{equat:mspp_iteration_exact}, e.g., with variance reduced stochastic algorithms~\citep{johnson2013accelerating,xiao2014proximal} or communication-efficient distributed optimization algorithms~\citep{shamir2014communication,zhang2015disco,yuan2020convergence}. Aiming at covering such an important yet less understood problem regime, we focus our study on analyzing the convergence behavior of M-SPP when the convex loss functions are smooth and the risk function exhibits quadratic growth property (see Assumption~\ref{assump:quadratic_growth} for a formal definition).

\subsection{Our Contributions and Main Results}
\label{ssect:contributions}
The main contribution of the present work is a sharper non-asymptotic convergence analysis of the M-SPP algorithm through the lens of algorithmic stability theory~\citep{bousquet2002stability,feldman2018generalization}. Let $ W^*:=\{w\in \mathcal{W}: R(w)=R^*\}$ be the set of minimizers of the composite population risk $R$. We are particularly interested in the regime where the loss function $\ell$ is convex and smooth but not necessarily Lipschitz (e.g., quadratic loss), while the population risk $R$ satisfies the quadratic growth condition, i.e., $R(w) - R^* \ge \frac{\lambda}{2}\min_{w^* \in W^*} \|w - w^*\|^2$, $\forall w \in \mathcal{W}$, for some $\lambda>0$, which can be satisfied by strongly convex objectives, and various other statistical estimation problems~\citep[see, e.g.,][]{karimi2016linear,drusvyatskiy2018error}. For the family of $L$-smooth loss functions, with $\gamma_t=\mathcal{O}(\lambda \rho t)$ for an arbitrary scalar $\rho\in(0,0.5]$ and $\epsilon_t\equiv 0$, we show in Theorem~\ref{thrm:fast_rate_smooth} that the excess risk at the weighted average output $\bar w_T = \frac{2}{T(T+1)}\sum_{t=1}^T t w_t$ is in expectation upper bounded by the following bound:
\begin{equation}\label{equat:contribution_bound_smooth}
R(\bar w_T) - R^* \lesssim \frac{\rho\left[R(w_0) - R^*\right]}{T^2} + \frac{LR^*}{\rho \lambda nT}.
\end{equation}
In this composite bound, the first \emph{bias} component associated with initial gap $R(w_0) - R^*$ has a decaying rate $\mathcal{O}\left(\frac{1}{T^2}\right)$ and the second \emph{variance} component associated with $R^*$ converges at the rate of $\mathcal{O}\left(\frac{1}{\lambda nT}\right)$. {The variance decaying rate actually matches the corresponding optimal rates of the SGD-type methods for strongly convex optimization~\citep{rakhlin2012making,dieuleveut2017harder,woodworth2021even}.} Also, such an $\mathcal{O}\left(\frac{1}{T^2} + \frac{1}{\lambda nT}\right)$ bounds matches those bounds for SPP~\citep{davis2019stochastic} or M-SPP~\citep{wang2017memory} which are in contrast obtained under a substantially stronger assumption that each individual loss function should be strongly convex and Lipschitz as well. In the realizable or near realizable machine learning regimes where $R^*$ equals to or approximates zero, the variance term in~\eqref{equat:contribution_bound_smooth} would be sharper than those bounds of~\citet{wang2017memory,davis2019stochastic}. To our best knowledge, the bound in~\eqref{equat:contribution_bound_smooth} for smooth and convex loss functions is \emph{new} to the SPP-type methods. More generally for arbitrary convex risk functions, we present in Theorem~\ref{thrm:fast_rate_smooth_convex} an $\mathcal{O}(\frac{1}{\sqrt{nT}})$ excess risk bound for exact M-SPP. Further, as shown in Theorem~\ref{thrm:fast_rate_smooth_inexact} and Theorem~\ref{thrm:fast_rate_smooth_convex_inexact}, similar results can be extended to the inexact M-SPP given that the inner-loop sub-optimality is sufficiently small.

In the regime $T \ll n$ which is of special interest for off-line incremental learning with large data batches, setting a near-optimal value $\rho=\sqrt{\frac{T}{n\lambda}}$ in the excess risk bound~\eqref{equat:contribution_bound_smooth} yields an $\mathcal{O}\left(\frac{1}{T\sqrt{\lambda nT}} \right)$ rate of convergence. This rate, in terms of $n$, is substantially slower than the $\mathcal{O}(\frac{1}{\lambda nT})$ rate available for the previous small-$n$-large-$T$ setup. In order to address such a deficiency, we propose a two-phase variant of M-SSP (see Algorithm~\ref{alg:mspp_tp}) to boost its performance in the small-$T$-large-$n$ regime: in the first phase, M-SPP with sufficiently small minibatch-size is invoked over $S_1$ to obtain $w_1$, and then initialized by $w_1$ the second phase applies M-SPP to the rest minibatches. Then in Theorem~\ref{thrm:fast_rate_smooth_initialization} we show that the in-expectation excess risk at the output of the second phase can be accelerated to scale as
\begin{equation}\label{equat:contribution_bound_smooth_twophase}
R(\bar w_T) - R^* \lesssim  \frac{L^2 (R(w_0) - R^*)}{\lambda^2 n^2T^2} + \frac{LR^*}{\lambda nT},
\end{equation}
which holds regardless to the mutual strength of minibatch size $n$ and iteration count $T$.

In addition to the above in-expectation risk bounds, we further derive a high-probability model estimation error bound of M-SPP based on algorithmic stability theory. Our deviation analysis is carried out over a sampling-without-replacement variant of M-SPP (see Algorithm~\ref{alg:mspp_swr}). For population risk with quadratic growth property, up to an additive term on the inner-loop sub-optimality $\epsilon_t$, we establish in Theorem~\ref{thrm:fast_rate_smooth_high_probability} the following deviation bound on the estimation error $D(\bar w_T, W^*)$ that holds with probability at least $1-\delta$ over $S$ while in expectation over the randomness of sampling:
\[
D(\bar w_T, W^*) \lesssim \frac{\sqrt{L\log(1/\delta)}\log(T)}{\lambda \sqrt{nT}} + \sqrt{\frac{\left[R(w_0) - R^*\right]}{\lambda T^2} + \frac{LR^*}{\rho \lambda^2 nT} }.
\]
When $T= \Omega (n)$, up to the logarithmic factors, this above bound matches (in terms of the total sample size $N=nT$) the known minimax lower bounds for statistical estimation even without computational limits~\citep{tsybakov2008introduction}.

To highlight the core contribution of this work, the following three new insights into M-SPP make our results distinguished from the best known of SPP-type methods for convex optimization:
\begin{enumerate}[leftmargin=*]
  \item First and for most, the fast rates in~\eqref{equat:contribution_bound_smooth} and~\eqref{equat:contribution_bound_smooth_twophase} reveal the impact of noise level, as quantified by $R^*$, to convergence rate which has not been previously known for SPP-type methods. These bounds are valid for smooth losses and thus complement the previous ones for Lipschitz losses~\citep{patrascu2017nonasymptotic,wang2017memory,davis2019stochastic}.
  \item Second, the risk bounds in~\eqref{equat:contribution_bound_smooth} and~\eqref{equat:contribution_bound_smooth_twophase} are established under the quadratic growth condition of population risk. This is substantially weaker than the instantaneous-loss-wise strong convexity assumption commonly imposed by prior analysis to achieve the comparable rates for SPP-type methods~\citep{toulis2017asymptotic,wang2017memory,asi2020minibatch}.
  \item Third, we provide a deviation analysis of M-SPP from the viewpoint of uniform algorithmic stability which to our best knowledge has not yet been addressed in the previous study on SPP-type methods.
\end{enumerate}
We should emphasize that, while we provide some insights into the numerical aspects of M-SPP through an empirical study, this work is largely a theoretical contribution.

\subsection{Related Work}

Our work is situated at the intersection of two lines of machine learning research: stochastic optimization and algorithmic stability theory, both of which have been actively studied with a vast body of beautiful and insightful theoretical results established in literature. We next incompletely review some representative work that are closely relevant to ours.

\emph{Stochastic optimization.} Stemming from the pioneering work of~\citet{robbins1951stochastic}, stochastic gradient descent (SGD) methods have been extensively studied to approximately solve a simplified version of the problem~\eqref{equat:problem} with $r\equiv 0$~\citep{zhang2004solving,nemirovski2009robust,rakhlin2012making,bottou2018optimization}. For the composite formulation, a vast body of proximal SGD methods have been developed for efficient optimization in the presence of potentially non-smooth regularizers~\citep{hu2009accelerated,duchi2010composite,ghadimi2012optimal,lan2012optimal,kulunchakov2019generic}.
To handle the challenges associated with stepsize selection and numerical instability of SGD~\citep{nemirovski2009robust,moulines2011non}, a number of more sophisticated  methods including implicit stochastic/online learning~\citep{crammer2006online,kulis2010implicit,toulis2016towards,toulis2017asymptotic} and stochastic proximal point (SPP) methods~\citep{bertsekas2011incremental,patrascu2017nonasymptotic,asi2019importance,asi2019stochastic,davis2019stochastic} have recently been investigated for enhancing stability and adaptivity of stochastic (composite) optimization. For an example, in our considered composite optimization regime, the iteration procedure of vanilla SPP can be expressed as the following recursion form for $\ i\ge 1$:
\begin{equation}\label{equat:spp}
\hat w^{\text{spp}}_i: = \argmin_{w \in \mathcal{W}} \ell(w; z_i) + r(w) + \frac{\gamma_i}{2}\|w-\hat w^{\text{spp}}_{i-1}\|^2,
\end{equation}
where $z_i\sim \mathcal{D}$ is a random data sample, $\gamma_i$ is a regularization modulus and $\|\cdot\|$ stands for the Euclidean norm. In contrast to standard SGD methods which are simple in per-iteration modeling but brittle to stepsize choice, the SPP methods are more accurate in objective approximation which leads to substantially improved stability to the choice of algorithm hyper-parameters while enjoying optimal guarantees on convergence~\citep{asi2019importance,asi2019stochastic}.

An attractive feature of these above (proximal) stochastic optimization methods is that their convergence guarantees directly apply to the population risk and the minimax optimal rates of order $\mathcal{O}(\frac{1}{T})$ are achievable after $T$ rounds of iteration for strongly convex problems~\citep{nemirovski2009robust,agarwal2012information,rakhlin2012making}. For large-scale machine learning, the improved memory efficiency is another practical argument in favor of stochastic over batch optimization methods. However, due to the sequential processing nature, the stochastic optimization methods tend to be less efficient for parallelization especially in distributed computing environment where excessive communication between nodes would be required for model update~\citep{bottou2018optimization}.

\emph{Empirical risk minimization.} At the opposite end of SGD-type and online learning, the following defined (composite) empirical risk minimization (ERM, a.k.a., M-estimation) is another popularly studied formulation for statistical learning~\citep{lehmann2006theory}:
\[
\hat w^{\text{erm}}_S: = \argmin_{w \in \mathcal{W}} \left\{R_S(w):= \frac{1}{N}\sum_{i=1}^N \ell(w; z_i) + r(w) \right\}.
\]
Thanks to the finite-sum structure, a large body of randomized incremental algorithms with linear rates of convergence have been established for ERM including SVRG~\citep{johnson2013accelerating,xiao2014proximal}, SAGA~\citep{defazio2014saga} and Katyusha~\citep{allen2017katyusha}, to name a few. From the perspective of distributed computation, one intrinsic advantage of ERM over SGD-type methods lies in that it can better explore the statistical correlation among data samples for designing communication-efficient distributed optimization algorithms~\citep{jaggi2014communication,shamir2014communication,zhang2015disco,lee2017distributed}. Unlike stochastic optimization methods, the generalization performances of the batch or incremental algorithms are by nature controlled by that of ERM~\citep{bottou2007tradeoffs} which has long been studied with a bunch of insightful results available~\citep{vapnik1999overview,bartlett2005local,srebro2010smoothness,mei2018landscape}. Particularly for strongly convex risk functions, the $\mathcal{O}(\frac{1}{N})$ rate of convergence is possible for ERM~\citep{bartlett2005local,koltchinskii2006local,zhang2017empirical}, though these fast rates are in general dimensionality-dependent for parametric learning models.

It has been recognized that SGD-type and ERM-type approaches cannot dominate each other in terms of generalization, runtime, storage and parallelization efficiency. This motivates a recent trend of trying to propose the so called stochastic model-based methods that can achieve the best of two worlds. Among others, a popular paradigm for such a purpose of combination is \emph{minibatch proximal update} which in each iteration updates the model via (approximately) solving a local ERM over a stochastic minibatch~\citep{li2014efficient,wang2017memory,asi2020minibatch,deng2021minibatch}. This strategy can be viewed as a minibatch extension to the SPP algorithm and it has been shown to attain a substantially improved trade-off between computation, communication and memory efficiency for large-scale distributed machine learning~\citep{li2014efficient,wang2017efficient}.  Alternatively, a number of online extensions of the incremental finite-sum algorithms, such as streaming SVRG~\citep{frostig2015competing} and streaming SAGA~\citep{jothimurugesan2018variance}, have been proposed for stochastic optimization with competitive guarantees to ERM but at lower cost of computation.

\emph{Algorithmic stability and generalization.} Since the seminal work of~\citet{bousquet2002stability}, algorithmic stability has been extensively studied with remarkable success achieved in establishing generalization bounds for strongly convex ERM estimators~\citep{zhang2003leave,mukherjee2006learning,shalev2010learnability}. Particularly, the state-of-the-art risk bounds of strongly convex ERM are offered by approaches based on the notion of uniform stability~\citep{feldman2018generalization,feldman2019high,bousquet2020sharper,klochkov2021stability}. It was shown by~\citet{hardt2016train} that the solution obtained via (stochastic) gradient descent is stable for smooth convex or non-convex loss functions. For non-smooth convex losses, the stability induced generalization bounds of SGD have been established in expectation~\citep{lei2020fine} or deviation~\citep{bassily2020stability}. For learning with sparsity, algorithmic stability theory has been employed to derive the generalization bounds of the popularly used iterative hard thresholding (IHT) algorithm~\citep{yuan2022stability}. Through the lens of uniform algorithmic stability, convergence rates of M-SPP have been studied for convex~\citep{wang2017memory} and weakly convex~\citep{deng2021minibatch} Lipschitz losses. While sharing a similar spirit to~\citet{wang2017memory,deng2021minibatch}, our analysis customized for smooth convex loss functions is considerably different and the resultant fast rates are of special interest in low-noise statistical settings~\citep{srebro2010smoothness}.

\subsection{Notation and Paper Organization}

\emph{Notation.} The key quantities and notations frequently used in our analysis are summarized in Table~\ref{tab:notation}.\vspace{0.2in}

\begin{table}[h]
\centering
\begin{tabular}{|c|c|}
\hline
Notation & Definition  \\
\hline
$n$ & minibatch size \\
$T$ & round of iteration \\
$N$ & total number of samples visited, i.e., $N=nT$ \\
$f$ & hypothesis \\
$\ell$ & loss function \\
$r$ & regularization term \\
$R^{\ell}$ & population risk: $R^{\ell}(w):= \mathbb{E}_{(x,y) \sim \mathcal{D}} [\ell(f_w(x), y)]$ \\
$R$ & composite population risk: $R(w):= R^{\ell}(w) + r(w)$   \\
$R^*$ & the optimal value of composite risk, i.e., $R^*:=\min_{w \in \mathcal{W}}R(w)$ \\
$W^*$ & the optimal solution set of composite risk, i.e., $W^*:=\argmin_{w \in \mathcal{W}}R(w)$ \\
$S_t$ & data minibatch at time instance $t$ \\
$S_{I}$ & The union of data minibatch over $I$, i.e., $S_{I}:=\{S_t\}_{t\in I}$ \\
$R^\ell_{S}$ & empirical risk over $S$, i.e., $R^\ell_{S}(w):=\frac{1}{|S|}\sum_{(x,y)\in S} \ell(f_w(x, y)$ \\
$R_{S}$ & composite empirical risk over $S$, i.e., $R_{S}(w):=R^\ell_{S}(w) + r(w)$ \\
$\epsilon_t$ & precision of minibatch risk minimization at time instance $t$  \\
%$\|\cdot\|$ & a properly defined norm over the hypothesis space $\mathcal{F}$ \\
$\|w\|_1$ & $\ell_1$-norm of a vector $w$, i.e., $\|w\|_1:=\sum_i |[w]_i|$ \\
$\|w\|$ & Euclidean norm of a vector $w$ \\
$D(w,W^*)$ & the distance from $w$ to $W^*$, i.e., $D(w,W^*)=\min_{w^*\in W^*}\|w - w^*\|$ \\
$[T]$ & $[T]:=\{1,...,T\}$ \\
$\mathbf{1}_{\{C\}}$ & the indicator function of the condition $C$\\
\hline
\end{tabular}
\caption{Table of notation. \label{tab:notation}}
\end{table}

\newpage

\noindent\emph{Organization.} The paper proceeds with the material organized as follows:  In Section~\ref{sect:analysis_smooth}, we analyze the risk bounds of exact M-SPP with convex and smooth loss functions and present a two-phase variant to further improve convergence performance. In Section~\ref{sect:analysis_smooth_inexact}, we extend our analysis to the more realistic setting where inexact M-SPP iteration is allowed. In Section~\ref{sect:high_probability}, we study the high-probability estimation error bounds of M-SPP. A comprehensive comparison to some closely relevant results is highlighted in Section~\ref{sect:related_work}. The numerical study for theory verification and algorithm evaluation is provided in Section~\ref{sect:experiment}. The concluding remarks are made in Section~\ref{sect:conclusion}. All the proofs of main results and some additional results on the iteration stability of M-SPP are relegated to appendix.

\section{A Sharper Analysis of M-SPP for Smooth Loss}
\label{sect:analysis_smooth}

In this section, we analyze the convergence rate of M-SPP for smooth and convex loss functions using the tools developed in algorithmic stability theory. In what follows, for the sake of notation simplicity and presentation clarity of core ideas, we assume for the time being that the inner-loop composite ERM in the M-SPP iteration procedure~\eqref{equat:mspp_iteration} has been solved exactly with $\epsilon_t\equiv0$, i.e.,
\begin{equation}\label{equat:mspp_iteration_exact}
w_t = \argmin_{w \in \mathcal{W}} \left\{F_t(w):=R_{S_t}(w) + \frac{\gamma_t}{2} \|w-w_{t-1}\|^2\right\}.
\end{equation}
A full convergence analysis for the inexact variant (i.e., $\epsilon_t>0$) will be presented in the Section~\ref{sect:analysis_smooth_inexact} via a slightly more involved perturbation analysis.

\subsection{Basic Assumptions}

We begin by introducing some basic assumptions that will be used in the analysis to follow. We say a differentiable function $g: \mathcal{W} \mapsto \mathbb{R}$ is $L$-smooth if $\forall s, t\in \mathbb{R}$,
\[
\left|g(w) -g(w') - \langle \nabla g(w), w - w'\rangle \right| \le \frac{L}{2}|w - w'|^2.
\]
As formally stated in the following assumption, we suppose that the individual loss functions are convex and $L$-smooth which can be satisfied, e.g., by the quadratic loss (for regression) and the logistic loss (for prediction).
\begin{assumption}\label{assump:smooth}
The loss function $\ell$ is convex and $L$-smooth with respect to its first argument. Also, we assume that the regularization term $r$ is convex over $\mathcal{W}$.
\end{assumption}
Let us define $D(w,W^*):=\min_{w^*\in W^*}\|w-w^*\|$ as the distance from $w$ to the set $W^*$ of minimizers. The next assumption requires that the population risk has the characterization of quadratic growth away from the set of minimizers~\citep{anitescu2000degenerate,karimi2016linear}.
\begin{assumption}\label{assump:quadratic_growth}
The population risk function $R$ satisfies $R(w) \ge R^* + \frac{\lambda}{2}D^2(w,W^*),  \forall w \in \mathcal{W}$ for some $\lambda>0$.
\end{assumption}
Clearly, the quadratic growth property can be implied by the traditional strong convexity condition (around the minimizers) which is satisfied by a number of popular learning models including linear and logistic regression, generalized linear models, smoothed Huber losses, and various other statistical estimation problems. Particularly, Assumption~\ref{assump:quadratic_growth} holds when $R^{\ell}$ is strongly convex and $r$ is convex. Notice that for risk functions with quadratic growth property, the prior analysis of M-SPP for Lipschitz losses~\citep{wang2017memory} is not generally applicable because Assumption~\ref{assump:quadratic_growth} implies that the Lipschitz constant of loss could be arbitrarily large if the infinite distance $\min_{w^*\in W^*}\|w - w^*\|\rightarrow \infty$ is allowed.

\subsection{Main Results}
\label{ssect:main_results}

The following theorem is our main result on the in-expectation rate of convergence of the exact M-SPP with smooth loss and quadratic growth population risk functions. Recall that $N=nT$ is the total number of data points visited up to the iteration counter $T$.
%Recall that we have assumed $\epsilon_t\equiv 0$ for the time being.

\begin{restatable}{theorem}{FastRateSmooth}\label{thrm:fast_rate_smooth}
Suppose that Assumptions~\ref{assump:smooth} and~\ref{assump:quadratic_growth} hold. Consider $\epsilon_t\equiv 0$ and the weighted average output $\bar w_T = \frac{2}{T(T+1)}\sum_{t=1}^T t w_t$ in Algorithm~\ref{alg:mspp}. Let $\rho \in (0,0.5]$ be an arbitrary scalar.
\begin{itemize}
  \item[(a)] Suppose that $n\ge\frac{64L}{\lambda\rho}$. Set $\gamma_t = \frac{\lambda \rho t}{4}$ for $t\ge1$.  Then for any $T\ge 1$,
\[
\mathbb{E} \left[R(\bar w_T) -  R^*\right] \le \frac{4\rho\left[R(w_0) - R^*\right]}{T^2} + \frac{2^{9}L}{\lambda \rho nT} R^*.
\]
  \item[(b)] {Set $\gamma_t = \frac{\lambda \rho t}{4}+ \frac{16L}{n}$ for $t\ge 1$. Then for any $T\ge 1$,
\[
\mathbb{E} \left[R(\bar w_T) -  R^*\right] \le \left(\frac{4\rho}{T^2} + \frac{2^{8}L}{\lambda nT}\right)[R(w_0) - R^*]  + \left(\frac{2^{16}L^2}{\lambda^2\rho^2n^2T}+\frac{2^{9}L}{\lambda\rho nT}\right)R^*,
\]}
\end{itemize}
\end{restatable}
\begin{proof}
The proof technique is inspired by the uniform stability arguments developed by~\citet{wang2017memory} for Lipschitz and instance-wise strongly convex loss, with several new elements along developed for handling smooth loss and quadratic growth of risk function. As a non-trivial ingredient, we show that it is possible to extend those stability arguments to smooth losses in view of a classical result from~\citet[Lemma 2.1]{srebro2010smoothness} that allows the derivative of a smooth loss to be bounded in terms of its function value. See Appendix~\ref{apdsect:proof_fast_rate_smooth} for a full proof of this result.
\end{proof}
A few remarks on Theorem~\ref{thrm:fast_rate_smooth} are in order.
\begin{remark}\label{remark:SGD_ERM}
In Part (a), the minibatch size is required to be sufficiently large. In this setting, the excess risk bound consists of two components: the first \emph{bias} component associated with initial gap $R(w_0) - R^*$ has a decaying rate $\mathcal{O}(\frac{1}{T^2})$ and the second \emph{variance} component associated with $R^*$ vanishes at a dominate rate of $\mathcal{O}(\frac{1}{\lambda nT})$. The variance term shows that the convergence rate can be improved in the low-noise settings where the factor of $R^*$ is relatively small. Extremely in the separable case with $R^*=0$, the excess risk bound of Theorem~\ref{thrm:fast_rate_smooth} would scale as fast as $\mathcal{O}(\frac{1}{T^2})$.
\end{remark}
\begin{remark}
One disadvantage of the result in Part (a) lie in that the minibatch size is required to be sufficiently larger than the condition number of the population risk $R$. Contrastively, the excess risk bound in Part (b) holds for arbitrary minibatch sizes. The cost, however, is a relatively slower bias decaying term $\mathcal{O}(\frac{1}{T^2}+\frac{1}{\lambda nT})$ which is dominated by $\mathcal{O}(\frac{1}{\lambda nT})$ in the case of $T\gg n$.
\end{remark}
\begin{remark}\label{remark:comparison_of_SOTA}
Let $N=nT$ be the total number of data points accessed. When $T\gg n$, the $\mathcal{O}(\frac{1}{N})$ dominant rates in Theorem~\ref{thrm:fast_rate_smooth} match those prior ones for SPP-type methods~\citep{wang2017memory,davis2019stochastic} which are, however, obtained under the assumption that each individual loss function should be Lipschitz continuous and strongly convex. In comparison to the $\mathcal{O}(\frac{1}{N})$ rate established for SGD with smooth loss~\citep[Theorem 12]{lei2020fine}, our result in Theorem~\ref{thrm:fast_rate_smooth} is stronger and less stringent in the following senses: 1) our bound shows explicitly the impact of $R^*$ which usually represents the noise level of model, and 2) we only require the population risk to have quadratic growth property while the bound of~\citet[Theorem 12]{lei2020fine} not only requires the loss to be Lipschitz but also assumes the empirical risk to be strongly convex.
\end{remark}

%\begin{remark}\label{remark:universal_of_analysis}
%From the perspective of analysis, an appealing feature of M-SPP lies in that the related proof arguments for composite or non-composite objectives, with or without convex constraint, are almost identical mostly due to its minibatch optimization nature.
%\end{remark}

Let us further look into the choice of the scalar $\rho$ in Theorem~\ref{thrm:fast_rate_smooth}. We focus the discussion on the part (a) and similar observations apply to the part (b). We distinguish the discussion in the following two complementary cases regarding the mutual strength of minibatch-size $n$ and round of iteration $T$:

\begin{itemize}[leftmargin=*]
\item \textbf{Case I: Small-$n$-large-$T$.}  Suppose that $n=\mathcal{O}(1)$ and $T\rightarrow \infty$ is allowed. In this case, simply setting $\rho=0.5$ yields the convergence rate of order $\mathcal{O}\left(\frac{1}{T^2} + \frac{1}{\lambda nT} \right)$ in the part (a).
\item \textbf{Case II: Small-$T$-large-$n$.} Suppose that $T=\mathcal{O}(1)$ and $n\rightarrow \infty$ is allowed. In this setup, given that $n\ge \frac{4T}{\lambda}$, then with a roughly optimal choice $\rho=\sqrt{\frac{T}{n\lambda}}$ the excess risk bound in Theorem~\ref{thrm:fast_rate_smooth}(a) will be of the order $\mathcal{O}\left(\frac{1}{T\sqrt{\lambda nT}} \right)$, which is substantially slower than the previous fast rate in Case I. {This is intuitive because M-SPP with large minibatches behaves more like regularized ERM which is known to exhibit slow rate of convergence even for strongly convex problems~\citep{shalev2010learnability,srebro2010smoothness}}.  Nevertheless, such a small-$T$-large-$n$ setup is of special interest for off-line incremental learning with large minibatches and distributed statistical learning~\citep{li2014efficient,wang2017memory,you2020large}. We will address this critical case in the next subsection.
\end{itemize}

\newpage

\subsection{A Two-Phase M-SPP Method}
\label{ssect:two_phase_mspp}

\begin{algorithm}[h]\caption{\texttt{Two-Phase M-SPP (M-SPP-TP)}}
\label{alg:mspp_tp}
\SetKwInOut{Input}{Input}\SetKwInOut{Output}{Output}\SetKw{Initialization}{Initialization}
\Input{Dataset $S = \{S_t\}_{t=1}^T$ in which $S_t:=\{z_{i,t}\}_{i=1}^n \overset{\text{i.i.d.}}{\sim} \mathcal{D}^n$, regularization modulus $\{\gamma_t>0\}_{t\in [T]}$.}
\Output{$\bar w_T$ as a weighted average of $\{w_t\}_{2\le t\le T}$.}
\Initialization{Specify a value of $w_0$. Typically $w_0=0$.}

\tcc{\textcolor[rgb]{0.00,0.00,1.00}{\textbf{Phase-I}}}

Divide sample $S_1$ into disjoint minibatches of equal size $m$;

Run M-SPP over these minibatches to obtain the output $w_1$;

\tcc{\textcolor[rgb]{0.00,0.00,1.00}{\textbf{Phase-II}}}

Initialized with $w_1$, run M-SPP over data minibatches $\{S_t\}_{2\le t\le T}$ with $\{\gamma_t\}_{2\le t\le T}$ to obtain the sequence $\{w_t\}_{2\le t\le T}$.

\end{algorithm}

To remedy the deficiencies mentioned in the previous discussion, we propose a two-phase variant of M-SSP, as outlined in Algorithm~\ref{alg:mspp_tp}, to boost its performance in the small-$T$-large-$n$ regimes. The procedure can be regarded as sort of a restarting argument~\citep{nemirovskii1985optimal,renegar2021simple,zhou2022practical} for M-SPP. More specifically, the Phase-I serves as an initialization step that invokes M-SPP to a uniform division of $S_1$ with minibatch size $m$ to obtain $w_1$. {Then starting from $w_1$, the Phase-II just invokes M-SPP to the consequent large minibatches $\{S_t\}_{t\ge 2}$ which is suitable for large-scale parallelization if applicable.} The following theorem is a consequence of Theorem~\ref{thrm:fast_rate_smooth} to such a two-phase M-SPP procedure.

\begin{restatable}{theorem}{FastRateSmoothInitialization}\label{thrm:fast_rate_smooth_initialization}
Suppose that Assumptions~\ref{assump:smooth} and~\ref{assump:quadratic_growth} hold. Consider $\epsilon_t\equiv 0$ for implementing M-SPP in both Phase-I and Phase-II of Algorithm~\ref{alg:mspp_tp}. Consider the weighted average output $\bar w_T = \frac{2}{(T-1)(T+2)}\sum_{t=2}^T t w_t$ in Phase-II.
\begin{itemize}
  \item[(a)] Suppose that $n\ge\frac{128 L}{\lambda}$. Set $m=\frac{128 L}{\lambda}$ in Phase-I and $\gamma_t = \frac{\lambda t}{8}$ for implementing M-SPP in both Phase-I and Phase II. Then for any $T\ge 2$, $\bar w_T $ satisfies
\[
\mathbb{E} \left[R(\bar w_T) -  R^*\right] \lesssim \frac{L^2\left[R(w_0) -  R^*\right]}{\lambda^2 n^2T^2} + \frac{L}{\lambda nT} R^*.
\]
  \item[(b)] {Set $m=\mathcal{O}(1)$ in Phase-I and $\gamma_t = \frac{\lambda t}{8}+ \frac{16L}{n}$ for implementing M-SPP in both Phase-I and Phase-II. Then for any $T\ge 2$, $\bar w_T $ satisfies
\[
\mathbb{E} \left[R(\bar w_T) -  R^*\right] \lesssim \frac{L^2\left[R(w_0) -  R^*\right]}{\lambda^2 nT} + \frac{L^3}{\lambda^3 nT} R^*.
\]}
\end{itemize}
\end{restatable}

\vspace{-0.1in}

\begin{proof}
See Appendix~\ref{apdsect:proof_fast_rate_smooth_initialization_corollary} for a proof of this result.
\end{proof}

\vspace{-0.1in}

\begin{remark}
The part (a) of Theorem~\ref{thrm:fast_rate_smooth_initialization} suggests that when the minibatch size is sufficiently large, the excess risk bound of two-phase M-SPP has a bias decaying term of scale $\mathcal{O}\left(\frac{1}{n^2T^2}\right)$ and a variance term that decays at the rate of $\mathcal{O}(\frac{1}{nT})$. The rate is valid even when the scales of $T$ relatively small, and thus is stronger than the $\mathcal{O}\left(\frac{1}{T\sqrt{nT}}\right)$ rate implied by Theorem~\ref{thrm:fast_rate_smooth} for the vanilla M-SPP in the small-$T$-large-$n$ regime. It is worth to mention that both the bias and variance components in our bound for M-SPP are faster than those derived for strongly convex ERM~\citep{srebro2010smoothness}.
\end{remark}

\begin{remark}
The excess risk bound in Part (b) of Theorem~\ref{thrm:fast_rate_smooth_initialization} is valid for arbitrary minibatch sizes, but at the cost of a relatively slower $\mathcal{O}(\frac{1}{ nT})$ bias decaying rate.
\end{remark}

\subsection{Results for Arbitrary Convex Risks}

We further analyze the proposed M-SPP algorithm when the loss function $\ell$ is convex and smooth, but without requiring that the composite risk $R$ has quadratic growth property. The following is our main result in such a generic setting.
\begin{restatable}{theorem}{FastRateSmoothConvex}\label{thrm:fast_rate_smooth_convex}
Suppose that Assumption~\ref{assump:smooth} holds. Set $\gamma_t \equiv \gamma \ge  \frac{16 L}{n}$. Let $\bar w_T = \frac{1}{T}\sum_{t=1}^T w_t$ be the average output of Algorithm~\ref{alg:mspp}. Then
\[
\mathbb{E} \left[R(\bar w_T) -  R^*\right] \lesssim \frac{\gamma}{T}D^2(w_0, W^*) + \frac{L}{\gamma n}  R^*.
\]
{Particularly for $\gamma = \sqrt{\frac{T}{n}}+ \frac{16 L}{n}$, it holds that
\[
\mathbb{E} \left[R(\bar w_T) -  R^*\right] \lesssim \left(\frac{1}{\sqrt{nT}} + \frac{L}{nT}\right)D^2(w_0,W^*) + \frac{L}{\sqrt{nT}}  R^*.
\]}
\end{restatable}

%\vspace{-0.1in}

\begin{proof}
See Appendix~\ref{apdsect:proof_fast_rate_smooth_convex} for a proof of this result.
\end{proof}

\begin{remark}
The first bound of Theorem~\ref{thrm:fast_rate_smooth_convex} implies that for any $\epsilon\in(0,1)$, by setting $\gamma=\mathcal{O}\left(\frac{L}{\epsilon n}\right)$, $R(\bar w_t)$ converges to $(1+ \epsilon)R^*$ at the rate of $\mathcal{O}(\frac{1}{nT\epsilon})$. This bound matches the results of~\citet[Theorem 4]{lei2020fine} for smooth SGD method. The second bound of Theorem~\ref{thrm:fast_rate_smooth_convex} further shows that by setting $\gamma=\mathcal{O}(\sqrt{\frac{T}{n}}+ \frac{L}{n})$, the excess risk of $\bar w_T$ decays at the rate of $\mathcal{O}(\frac{1}{\sqrt{nT}})$ for both bias and variance terms, which matches in order the corresponding bound derived for Lipschitz-loss~\cite[Theorem 4]{wang2017memory}. To our knowledge, such a bias-variance composite rate of convergence is new for SPP-type methods with convex and smooth loss functions.
\end{remark}
Analogous to the robustness analysis of SPP~\citep{asi2019importance,asi2019stochastic}, we have also analyzed the iteration stability of M-SPP for convex losses with respect to the choice of regularization modulus $\gamma_t$. The corresponding results, which can be found in Appendix~\ref{sect:robustness}, confirm that the choice of $\gamma_t$ is insensitive to the gradient scale of loss functions for generating a non-divergent sequence of estimation errors.

\section{Perturbation Analysis for Inexact M-SPP}
\label{sect:analysis_smooth_inexact}

In the preceding section, we have analyzed the convergence rates of M-SPP under the assumption that the inner-loop proximal ERM subproblems constructed in its iteration procedure~\eqref{equat:mspp_iteration} are solved exactly, i.e., $\epsilon_t\equiv 0$. To make our analysis more practical, we further provide in this section a perturbation analysis of M-SPP when the inner-loop proximal ERM subproblems are only required to be solved approximately up to certain precision $\epsilon_t >0$. As a starting point, we need to impose the following Lipschitz continuity assumption on the regularization term $r$.
\begin{assumption}\label{assump:lipschitz}
The regularization term $r$ is Lipschitz continuous over $\mathcal{W}$, i.e., $|r(w) - r(w')|\le G\|w - w'\|$, $\forall w, w'\in \mathcal{W}$.
\end{assumption}
For example, the $\ell_1$-norm regularizer $r(w)=\mu\|w\|_1$ satisfies this assumption with respect to Euclidean norm as $|r(w) - r(w')| = \mu|\|w\|_1 - \|w'\|_1| \le \mu \|w - w'\|_1\le \mu\sqrt{p}\|w - w'\|$.

The following theorem is our main result on the rate of convergence of the inexact M-SPP for composite stochastic convex optimization with smooth losses.
\begin{restatable}{theorem}{FastRateSmoothInexact}\label{thrm:fast_rate_smooth_inexact}
Suppose that Assumptions~\ref{assump:smooth},~\ref{assump:quadratic_growth} and~\ref{assump:lipschitz} hold. Let $\rho \in (0,1/4]$ be an arbitrary scalar and set $\gamma_t = \frac{\lambda \rho t}{4}$. Suppose that $n\ge\frac{76 L}{\lambda\rho}$. {Assume that $\epsilon_t\le \frac{\epsilon}{nt^4}$ for some $\epsilon\in [0,1]$. }Then for any $T\ge 1$, the weighted average output $\bar w_t = \frac{2}{T(T+1)}\sum_{t=1}^T t w_t$ of Algorithm~\ref{alg:mspp} satisfies
\[
\mathbb{E} \left[R(\bar w_t) -  R^*\right] \lesssim \frac{\rho}{T^2}(R(w_0) - R^*) + \frac{L}{\lambda\rho nT} R^* + {\frac{\sqrt{\epsilon}}{T^2}\left(\frac{L}{\lambda\rho} + G\sqrt{\frac{1}{\lambda\rho}}\right)}.
\]
\end{restatable}
\begin{proof}
See Appendix~\ref{apdsect:proof_fast_rate_smooth_inexact} for a proof of this result. We would like to highlight that our perturbation analysis for smooth loss is considerably different from that of~\citet{wang2017memory} developed for Lipschitz loss. This is mainly because in the smooth loss case, the change of loss could no longer be upper bounded by the change of prediction, and thus we need to make a more careful treatment to the perturbation caused by inexact minimization of the regularized minibatch empirical risk.
\end{proof}

We provide in order a few remarks on Theorem~\ref{thrm:fast_rate_smooth_inexact}.
\begin{remark}
Theorem~\ref{thrm:fast_rate_smooth_inexact} suggests that the excess risk bound of exact M-SPP in the part (a) of Theorem~\ref{thrm:fast_rate_smooth} can be extended to its inexact version, provided that the inner-loop minibatch ERMs~\eqref{equat:mspp_iteration} are solved to sufficient accuracy, say, $\epsilon_t\le\mathcal{O}\left(\frac{1}{nt^4}\right)$. Similarly, the result in the part (b) of Theorem~\ref{thrm:fast_rate_smooth} for arbitrary minibatch sizes can also be extended to the inexact M-SPP, which is omitted to avoid redundancy. Since the inner-loop minibatch ERMs are strongly convex and the loss functions are smooth, in average the desired accuracy can be attained in logarithmic time $\mathcal{O}\left(\log\left(\frac{1}{\epsilon_t}\right)\right)$ via variance-reduced SGD methods~\citep{xiao2014proximal}.
\end{remark}

\newpage
\begin{remark}
Analogous to the discussions at the end of Section~\ref{ssect:main_results}, by specifying the choice of $\rho$  we can derive a direct consequent result of Theorem~\ref{thrm:fast_rate_smooth_inexact} which more explicitly shows the rate of convergence with respect to $N=nT$. Also for the two-phase M-SPP, in view of Theorem~\ref{thrm:fast_rate_smooth_inexact} we can show that the bound in Theorem~\ref{thrm:fast_rate_smooth_initialization} can be extended to the inexact setting if the minibatch optimization is sufficiently accurate. These extensions are more or less straightforward and thus are omitted.
\end{remark}

In the following theorem, we provide an excess risk bound for the inexact M-SPP when the composite risk $R$ is convex but not necessarily has quadratic growth property.
\begin{restatable}{theorem}{FastRateSmoothConvexInexact}\label{thrm:fast_rate_smooth_convex_inexact}
Suppose that Assumptions~\ref{assump:smooth} and~\ref{assump:lipschitz} hold. Set $\gamma_t \equiv \gamma \ge  \frac{19 L}{n}$. {Assume that $\epsilon_t\le \min\left\{\frac{\epsilon}{n^2t^5}, \frac{2G^2}{9n^2\gamma}\right\}$ for some $\epsilon\in [0,1]$.}
Then the average output $\bar w_T = \frac{1}{T}\sum_{t=1}^T w_t$ of Algorithm~\ref{alg:mspp} satisfies
\[
\mathbb{E} \left[R(\bar w_T) -  R^*\right] \lesssim \frac{\gamma}{T}D^2(w_0, W^*) + \frac{L}{\gamma n} R^* + {\left(\frac{L}{\gamma n } + \frac{\gamma}{LnT} +  \frac{G}{\sqrt{\gamma}nT} \right)\sqrt{\epsilon}}.
\]
Particularly for $\gamma = \sqrt{\frac{T}{n}}+ \frac{19L}{n}$, it holds that
\[
\mathbb{E} \left[R(\bar w_T) -  R^*\right] \lesssim \left(\frac{1}{\sqrt{nT}}+ \frac{L}{nT}\right)D^2(w_0, W^*) + \frac{L}{\sqrt{nT}} R^*+ {\left(\frac{L+G}{\sqrt{nT}} + \frac{1}{nT} \right)\sqrt{\epsilon}}.
\]
\end{restatable}
\begin{proof}
See Appendix~\ref{apdsect:proof_fast_rate_smooth_convex_inexact} for a proof of this result.
\end{proof}
\begin{remark}
Theorem~\ref{thrm:fast_rate_smooth_convex_inexact} confirms that the excess risk bounds established in Theorem~\ref{thrm:fast_rate_smooth_convex} for exact M-SPP are tolerant to sufficiently small sub-optimality $\epsilon_t\le\mathcal{O}(\frac{1}{n^2t^5})$ of minibatch proximal ERM subproblems.
\end{remark}

\section{Performance Guarantees with High Probability}
\label{sect:high_probability}

In the previous two sections, we have analyzed the excess risk bounds of M-SPP in expectation. In this section, we move on to study high-probability guarantees of M-SPP with respect to the randomness of training data, still under the notion of algorithmic stability. To this end, we first introduce a variant of M-SPP which carries out the proximal point update via sampling without replacement over the given data minibatches. We then show that the output of the proposed algorithm is uniformly stable in expectation over the randomness of sampling. As a main result of this section, for strongly convex population risk, we establish a near-optimal high probability (with respect to data) bound on the estimation error $\|\bar w_t - w^*\|$ that holds in expectation over the randomness of inner-data sampling. Additionally, we provide a high-probability generalization bound for arbitrary convex loss.

\newpage

\subsection{Sampling Without Replacement M-SPP}

%\vspace{0.2in}

\begin{algorithm}[h]
\caption{\texttt{Sampling Without Replacement M-SPP (M-SPP-SWoR)}}
\label{alg:mspp_swr}
\SetKwInOut{Input}{Input}\SetKwInOut{Output}{Output}\SetKw{Initialization}{Initialization}
\Input{Dataset $S = \{S_t\}_{t=1}^T$ in which $S_t:=\{z_{i,t}\}_{i=1}^n \overset{\text{i.i.d.}}{\sim} \mathcal{D}^n$, regularization modulus $\{\gamma_t>0\}_{t\in [T]}$.}
\Output{$\bar w_T$ as a weighted average of $\{w_t\}_{1\le t\le T}$..}
\Initialization{Specify a value of $w_0$. Typically $w_0=0$.}

\For{$t=1, 2, ..., T$}{

Uniformly randomly sample an index $\xi_t\in [T]$ without replacement.

Estimate $w_t$ satisfying
\begin{equation}\label{equat:mspp_swr_iteration}
F_t(w_t) \le \min_{w\in \mathcal{W}} \left\{F_t(w):=R_{S_{\xi_t}}(w) + \frac{\gamma_t}{2} \|w-w_{t-1}\|^2\right\} + \epsilon_t,
\end{equation}
where $\epsilon_t\ge0$ measures the sub-optimality.
}
\end{algorithm}

%\vspace{0.2in}

Let us consider the M-SPP-SWoR (M-SPP via Sampling Without Replacement) procedure as outlined in Algorithm~\ref{alg:mspp_swr}. Given a set $S$ of $T$ data minibatches, at each iteration, the algorithm uniformly randomly samples one minibatch from $S$ without replacement for proximal update. After $T$ rounds of iteration, all the minibatches are used to update the model. Since this procedure is merely a random shuffling variant of M-SPP as presented in Algorithm~\ref{alg:mspp}, we can see that all the in-expectation bounds established in the previous sections for M-SPP directly transfer to M-SPP-SWoR under any implementation of shuffling. As we will show shortly in the next subsection that such a random shuffling scheme is beneficial for boosting the on-average algorithmic stability of M-SPP which then leads to strong high-probability guarantees for M-SPP-SWoR.

\subsection{A Uniform Stability Analysis}

Let $S = \{S_t\}_{t\in [T]}$ and $S'=\{S'_t\}_{t\in [T]}$ be two sets of data minibatches. We denote by $S_t \doteq S'_t$ if $S_t$ and $S'_t$ differ in a single data point, and by $S \doteq S'$ if $S$ and $S'$ differ in a single minibatch and a single data point in that minibatch. We introduce the following concept of uniform stability of M-SPP which substantializes the concept of uniform algorithmic stability that serves as a powerful tool for analyzing generalization bounds of statistical estimators and their learning algorithms~\citep{bousquet2002stability,hardt2016train,feldman2019high}.
\begin{definition}[Uniform Stability of M-SPP]\label{def:uniform_stability}
The M-SPP algorithm is said to be $\varrho$-uniformly stable with respect to a mapping $h: \mathcal{W} \mapsto \mathbb{R}^q$ if $\left\|h(\bar w_T) - h(\bar w'_T)\right\| \le \varrho$ for any pair of data sets $S \doteq S'$.
\end{definition}
The following result gives a uniform stability (with respect to identical mapping) bound of the vanilla M-SPP (Algorithm~\ref{alg:mspp}) that holds deterministically, and a corresponding bound for M-SPP-SWoR (Algorithm~\ref{alg:mspp_swr}) that holds in expectation over the randomness of minibatch sampling.

\begin{restatable}{proposition}{UniformStabilityWithoutReplacement}\label{prop:uniform_stability_withoutreplacement}
Suppose that Assumption~\ref{assump:smooth} holds and the loss function is bounded such that $0\le \ell(y',y) \le M$ for all $y,y'$. Let $S=\{S_t\}_{t\in [T]}$ and $S'=\{S'_t\}_{t\in [T]}$ be two sets of data minibatches satisfying $S \doteq S'$. Then

\begin{itemize}
  \item[(a)] The weighted average output $\bar w_T$ and $\bar w'_T$ respectively generated by M-SPP (Algorithm~\ref{alg:mspp}) over $S$ and $S'$ satisfy
\[
\sup_{S, S'} \|\bar w_T - \bar w'_T\| \le \frac{4\sqrt{2LM}}{n\min_{t\in [T]}\gamma_t} +\sum_{t=1}^T 2\sqrt{\frac{2\epsilon_t}{\gamma_t}}.
\]
  \item[(b)] The weighted average output $\bar w_T$ and $\bar w'_T$ respectively generated by M-SPP-SWoR (Algorithm~\ref{alg:mspp_swr}) over $S$ and $S'$ satisfy
\[
\sup_{S, S'} \mathbb{E}_{\xi_{[T]}}\left[\|\bar w_T - \bar w'_T\|\right] \le \sum_{t=1}^T \left\{ \frac{4\sqrt{2LM}}{nT\gamma_t} + 2\sqrt{\frac{2\epsilon_t}{\gamma_t}} \right\}.
\]
\end{itemize}
\end{restatable}
\begin{proof}
See Appendix~\ref{apdsect:proof_uniform_stability_withoutreplacement} for a proof.
\end{proof}
\begin{remark}
Suppose that the sub-optimality $\{\epsilon_t\}_{t\in [T]}$ are sufficiently small. If setting $\gamma_t=\mathcal{O}(t)$ as used for population risks with quadratic growth property, then Proposition~\ref{prop:uniform_stability_withoutreplacement} shows that M-SPP is $\mathcal{O}\big(\frac{1}{n}\big)$-uniformly stable, while in expectation over the randomness of without-replacement sampling, M-SPP-SWoR has an much improved uniform stability parameter scaling as $\mathcal{O}\big(\frac{\log(T)}{nT}\big)$. If setting $\gamma_t\equiv \sqrt{\frac{T}{n}}$ as used for generic convex loss, then M-SPP will be $\mathcal{O}\big(\frac{1}{\sqrt{nT}}\big)$-uniformly stable while M-SPP-SWoR has an identical uniform stability parameter in expectation over sampling.
\end{remark}

In the following theorem, based on the uniform stability bounds in Proposition~\ref{prop:uniform_stability_withoutreplacement}, we derive an upper bound on the estimation error $D(\bar w_T, W^*)$ of M-SPP-SWoR that holds with high probability over data distribution while in expectation over randomly sampling the minibatches for update.
\begin{restatable}{theorem}{FastRateSmoothHighProbability}\label{thrm:fast_rate_smooth_high_probability}
Suppose that Assumptions~\ref{assump:smooth},~\ref{assump:quadratic_growth},~\ref{assump:lipschitz} hold and the loss function $\ell$ is bounded in the interval $(0, M]$. Let $\rho \in (0,1/4]$ be an arbitrary scalar and set $\gamma_t = \frac{\lambda \rho t}{4}$. Suppose that $n\ge\frac{76 L}{\lambda\rho}$. Assume that $\epsilon_t\le \min\left\{ \frac{\epsilon}{nt^4}, \frac{LM}{\lambda \rho n^2T^2 t}\right\}$ for some $\epsilon\in [0,1]$. Then with probability at least $1-\delta$ over $S$, the weighted average output $\bar w_T$ of M-SPP-SWoR (Algorithm~\ref{alg:mspp_swr}) satisfies
\[
\begin{aligned}
&\mathbb{E}_{\xi_{[T]}}\left[D(\bar w_T,  W^*) \right] \\
\lesssim & \frac{\sqrt{LM\log(1/\delta)}\log(T)}{\lambda \rho \sqrt{nT}} + \sqrt{\frac{\rho\left[R(w_0) - R^*\right]}{\lambda T^2} + \frac{L}{\lambda^2 \rho nT} R^* + \frac{\sqrt{\epsilon}}{\lambda T^2}\left(\frac{L}{\lambda\rho} + G\sqrt{\frac{1}{\lambda\rho}}\right)}.
\end{aligned}
\]
\end{restatable}
\begin{proof}
See Appendix~\ref{apdsect:proof_fast_rate_smooth_high_probability} for a proof of this result.
\end{proof}

\newpage

\begin{remark}
We comment on the optimality of the bound in Theorem~\ref{thrm:fast_rate_smooth_high_probability}. Consider $\rho=\mathcal{O}(1)$. The first term of scale $\mathcal{O}\big(\frac{\sqrt{\log(1/\delta)}\log(T)}{\sqrt{nT}}\big)$ represents the overhead of getting generalization with high probability over data. The second term matches the corresponding in-expectation estimation error bound in Theorem~\ref{thrm:fast_rate_smooth_inexact}, which matches the known optimal rates for strongly convex SGD~\citep{rakhlin2012making,dieuleveut2017harder}. In view of the minimax lower bounds for statistical estimation~\citep{tsybakov2008introduction}, the estimation error bound established in Theorem~\ref{thrm:fast_rate_smooth_high_probability} is near-optimal for strongly convex risk minimization.
%This in turn immediately implies that the excess risk bound established in Theorem~\ref{thrm:fast_rate_smooth} is nearly tight.
\end{remark}

Finally, we provide a high-probability generalization bound of M-SPP for arbitrary convex population risk functions.

\begin{restatable}{theorem}{FastRateSmoothConvexHighProbability}\label{thrm:fast_rate_smooth_convex_high_probability}
Suppose that Assumptions~\ref{assump:smooth} and~\ref{assump:lipschitz} hold and the loss function $\ell$ is bounded in the interval $[0, M]$. Set $\gamma_t \equiv \sqrt{\frac{T}{n}}$. Assume that $\epsilon_t\le \frac{LM}{4nT^2\sqrt{nT}}$. Then with probability at least $1-\delta$ over $S$, the average output $\bar w_T = \frac{1}{T}\sum_{t=1}^T w_t$ of M-SPP (Algorithm~\ref{alg:mspp}) satisfies
\[
\left|R(\bar w_T) - R_S(\bar w_T)\right| \lesssim \frac{(LM + G \sqrt{LM})\log(N)\log(1/\delta)}{\sqrt{nT}} + M\sqrt{\frac{\log\left(1/\delta\right)}{nT}}.
\]
\end{restatable}
\begin{proof}
See Appendix~\ref{apdsect:proof_fast_rate_smooth_convex_high_probability} for a proof of this result.
\end{proof}
We remark in passing that using similar uniform stability argument, the high-probability generalization bound in Theorem~\ref{thrm:fast_rate_smooth_convex_high_probability} can be shown to hold for convex and non-smooth loss functions as well. We omit the detailed analysis as it is out of the scope of this paper focusing on smooth losses.

\section{Comparison with Prior Methods}
\label{sect:related_work}

\emph{Comparison with M-SPP and SPP methods.} The M-SPP algorithm considered in this article is a minibatch extension of the SPP methods. The convergence analysis of SPP has received recent wide attention in stochastic optimization community. Specially for finite-sum optimization over $N$ data points, an incremental SPP method was proposed and analyzed in~\citep{bertsekas2011incremental}. For learning with linear prediction models and strongly convex Lipschitz-loss, \citep{toulis2016towards} established a set of $\mathcal{O}(\frac{1}{N^{\gamma}})$ rates of convergence for SPP with suitable $\gamma\in(0.5, 1]$, where $N$ is the iteration counter. For arbitrary convex loss functions, the non-asymptotic convergence performance of SPP was studied with $\mathcal{O}(\frac{1}{\sqrt{N}})$ rate obtained for Lipschitz losses~\citep{patrascu2017nonasymptotic,davis2019stochastic}, $\mathcal{O}(\frac{1}{N})$ for strongly convex and Lipschitz~\citep{davis2019stochastic} or smooth~\citep{patrascu2017nonasymptotic} losses, or $\mathcal{O}\left(\frac{\log(N)}{N}\right)$ rate for strongly convex non-smooth losses~\citep{asi2019stochastic}. Recently, it has been shown that the $\mathcal{O}\left(\frac{\log(N)}{N}\right)$ rate also extends to M-SPP with strongly convex losses~\citep{asi2020minibatch}. The asymptotic and non-asymptotic behaviors of SPP for weakly convex losses (e.g., composite of convex loss with smooth map) have been studied for stochastic optimization with~\citep{duchi2018stochastic} or without~\citep{davis2019stochastic} composite structures. Among others, our work is most closely related to the minibatch proximal update method developed for communication-efficient distributed optimization~\citep{wang2017memory}. Similarly from the viewpoint of algorithmic stability, the $\mathcal{O}(\frac{1}{N^{\gamma}})$ rates were established for that method for Lipschitz-loss with arbitrary convexity ($\gamma=0.5$) or strong convexity ($\gamma=1$). In comparison to these prior results, our convergence results for M-SPP are new in the following aspects:
\begin{itemize}
  \item The convergence rates are derived for smooth losses and they explicitly show the impact of noise level of a statistical model, as encoded in $R^*$, to convergence performance which has not been previously known for SPP-type methods.
  \item The $\mathcal{O}(N^{-1})$ fast rate attained in this article is valid for population risks with quadratic growth property, without requiring each instantaneous loss to be strongly convex.
  \item We provide a near-optimal model estimation error bound of a sampling-without-replacement variant of M-SPP that holds with high probability over the randomness of data while in expectation over the randomness of sampling.
 %\item We provide a unified framework for analyzing M-SPP for composite/non-composite risk minimization problems.
\end{itemize}

\emph{Comparison with SGD and ERM.} Similar to those in Theorem~\ref{thrm:fast_rate_smooth} and Theorem~\ref{thrm:fast_rate_smooth_convex}, the bias-variance composite rates have been known for accelerated SGD for least squares regression~\citep{dieuleveut2017harder}, {or minibatch SGD (M-SGD) for generic convex and smooth learning problems~\citep{woodworth2021even}. While the results are of similar flavor, we came to the path in a distinct algorithmic framework using quite different proof techniques. Particularly, in contrast to~\citet{woodworth2021even}, our analysis neither uses the knowledge of model scale which is typically inaccessible in real problems, nor relies on the restarting arguments for strongly convex problems.} Also for SGD with smooth loss functions, a fast rate of $\mathcal{O}(\frac{1}{N})$ has recently been established via stability theory in the ideally clean case where the optimal population risk is zero~\citep[Theorem 4]{lei2020fine}. With $\gamma=\mathcal{O}(\frac{1}{n})$, the first bound of our Theorem~\ref{thrm:fast_rate_smooth_convex} matches that bound in the context of M-SPP. For strongly convex problems, our results in Theorem~\ref{thrm:fast_rate_smooth} are stronger than~\citep[Theorem 12]{lei2020fine} in the sense that the formers (ours) only require the population risk to have quadratic growth property while the latter requires the loss to be Lipschitz and the empirical risk to be strongly convex. Finally, for convex ERM, similar composite risk bounds have been established by~\citet{srebro2010smoothness,zhang2017empirical} under somewhat more stringent conditions such as bounded domain of interest and huge sample with $N\gg p$.

\newpage

Table~\ref{tab:result_comparison} summaries a comparison of the risk bounds obtained in this work to several prior ones for (M-)SPP, (M-)SGD and ERM.\vspace{0.2in}

\begin{table}[h]
\small
%\renewcommand{\arraystretch}{1.0}
\centering
\begin{tabular}{|c|c|c|c|c|c|}
\hline
\multirow{2}{*}{Method} & \multirow{2}{*}{Literature} &  \multirow{2}{*}{Risk Bound} & \multicolumn{3}{c|}{Conditions} \\
\cline{4-6}
& & & Loss & $R$ & $R_S$ \\
\hline
\multirow{4}{*}{ M-SPP } & \citet{asi2020minibatch} & $\mathcal{O}\left(\frac{\log(N)}{N}\right)$ & s.cvx & --- & --- \\
&\citet{wang2017memory}  & $\mathcal{O}\left(\frac{1}{N}\right)$ & Lip \& s.cvx & --- & --- \\
& Theorem~\ref{thrm:fast_rate_smooth} \textbf{(our work)} & \tabincell{c}{{$\mathcal{O}\left(\frac{1}{T^2} + \frac{R^*}{N}\right)$} or \\  {$\mathcal{O}\left(\frac{1}{T^2} + \frac{1 + R^*}{N}\right)$}} & sm \& cvx & qg & --- \\
& Theorem~\ref{thrm:fast_rate_smooth_convex} \textbf{(our work)} & \tabincell{c}{{$\mathcal{O}\left(\frac{1}{N} + R^*\right)$} or \\  {$\mathcal{O}\left(\frac{1 + R^*}{\sqrt{N}} \right)$}} & sm \& cvx & --- & --- \\
%& \textbf{This work} & \cellcolor{red!30}$\mathcal{O}\left(\sqrt{\frac{1}{T}}+\sqrt{\frac{R^*}{N}}\right)$ & & & \\
\hline
\multirow{3}{*}{ SPP  } & \citet{asi2019stochastic} &$\mathcal{O}\left(\frac{\log(N)}{N}\right)$ & s.cvx & --- & ---\\
& \citet{patrascu2017nonasymptotic} & $\mathcal{O}\left(\frac{1}{N}\right)$ & sm \& s.cvx & --- & \\
& \citet{davis2019stochastic} & $\mathcal{O}\left(\frac{1}{N^2} + \frac{1}{N}\right)$ & Lip \& s.cvx & --- & \\
\hline
\hline
\multirow{2}{*}{{M-SGD}} & \multirow{2}{*}{\citet{woodworth2021even}} & {$\mathcal{O}\left(\frac{1}{T^2} + \frac{1}{N} +  \sqrt{\frac{R^*}{N}} \right)$} & {sm \& cvx} & --- & --- \\
&  & {$\mathcal{O}\left(e^{-T} +  \frac{R^*}{N} \right)$} & {sm \& cvx} &{ qg} & ---\\
\hline
& \citet{dieuleveut2017harder} & $\mathcal{O}\left(\frac{1}{N^2} + \frac{R^*}{N}\right)$ & quadratic & s.cvx & --- \\
\multirow{3}{*}{ SGD } & \citet{lei2020fine} & \tabincell{c}{$\mathcal{O}\left(\frac{1}{N} + R^*\right)$ or \\  $\mathcal{O}\left(\frac{1 + R^*}{\sqrt{N}} \right)$} & sm \& cvx & --- & s.cvx \\
& \citet{rakhlin2012making} & $\mathcal{O}\left(\frac{1}{N}\right)$ & \tabincell{c}{Lip \& \\ sm \& cvx} & s.cvx & --- \\
\hline
\hline
\multirow{2}{*}{\tabincell{c}{ \\ ERM }} & \citet{zhang2017empirical} & \tabincell{c}{$\mathcal{O}\left(\frac{p}{N} + \frac{R^*}{N}\right)$ or \\ $\mathcal{O}\left(\frac{1}{N^2} + \frac{R^*}{N}\right)$ \\ for $N\gtrsim p$ } & sm \& cvx & \tabincell{c}{ Lip \\ \& s.cvx} & --- \\
& \citet{srebro2010smoothness} & $\mathcal{O}\left(\frac{1}{N} + \sqrt{\frac{R^*}{N}} \right)$ & sm \& cvx & --- & --- \\
\hline
\end{tabular}
\caption{Comparison of our risk bounds to some prior results for M-SPP and SPP as well as for SGD and ERM. Recall that $T$ is the iteration count and $N$ is the total number of samples accessed. %All the bounds for quadratic problem and our results for non-quadratic problem hold with high probability over the random draw of local i.i.d. data. The other results are deterministic. The x-mark ``\xmark'' indicates that the related result was not available in the original reference of method.
All the listed bounds hold in expectation. Here we have used the following abbreviations: cvx (convex), s.cvx (strongly convex), Lip (Lipschitz continuous), sm (smooth), qg (quadratic growth). \label{tab:result_comparison}}
\end{table}

\newpage\clearpage

\section{Experiments}
\label{sect:experiment}

We carry out a set of numerical study to demonstrate the convergence performance of minibatch stochastic proximal point methods in (composite) statistical learning problems, to answer the following 3 questions associated with the key theory and algorithms established in this article:
\begin{itemize}\vspace{-0.03in}
  \item Question 1: \emph{How the size of minibatch and noise level of a statistical learning model affect the convergence speed of M-SPP for smooth loss function?} This question is mainly about verifying Theorem~\ref{thrm:fast_rate_smooth} and Theorem~\ref{thrm:fast_rate_smooth_convex_inexact}, and it is answered through a simulation study on Lasso estimation in Section~\ref{ssect:simulation}.\vspace{-0.03in}
  \item Question 2: \emph{Can the two-phase variant of M-SPP improve over M-SPP in the small-$T$-large-$n$ setting?} The simulation results presented in Section~\ref{ssect:simulation} also answer this question related to the verification of Theorem~\ref{thrm:fast_rate_smooth_initialization}.\vspace{-0.03in}
  \item Question 3: \emph{How M-SPP(-TP) methods compare with M-SGD in convergence performance?} The real-data experimental results on logistic regression tasks in Section~\ref{ssect:real_data_experiment} answer this question about algorithm comparison.
\end{itemize}
%While our convergence theory is essentially non-parametric, in practice the search for a solution still often needs to be restricted to a suitable parametric hypothesis to allow efficient computations and reliable estimation. Therefore, in this empirical study we focus on finite-parametric learning models including Lasso regression and logistic regression.

\subsection{Simulation Study}
\label{ssect:simulation}

We first provide a simulation study to verify our theoretical results for smooth losses when substantialize to the widely used Lasso regression model~\citep{wainwright2009sharp} with quadratic loss function $\ell(f_w(x), y)=\frac{1}{2}(y-w^{\top}x)^2$ and $r(f_w)=\mu \|w\|_1$ where $\mu$ is the $\ell_1$-penalty modulus. Given a model parameter $\bar w \in \mathbb{R}^p$ and a feature point $x\in \mathbb{R}^p$ drawn from standard Gaussian distribution $\mathcal{N}(0, I_{p\times p})$, the responses $y$ is generated according to a linear model $y=\bar{w}^{\top}x+\varepsilon$ with a random Gaussian noise $\varepsilon\sim \mathcal{N}(0,\sigma^2)$. In this case, the population risk function can be expressed in a close form as
\begin{align*}
R(w) = \frac{1}{2}\|w - \bar w\|^2 + \frac{\sigma^2}{2} + \mu \|w\|_1.
\end{align*}
Given a set of $T$ random $n$-minibatches $\left\{S_t=\{x_{i,t},y_{i,t}\}_{i\in[n]}\right\}_{t\in [T]}$ drawn from the above data distribution, we aim at evaluating the convergence performance of M-SPP towards the minimizer of $R$ which can be expressed as
\begin{align*}
w^*=(\bar w - \mu)_+ - (-\bar w - \mu)_+,
\end{align*}
where $(\cdot)_+$ is an element-wise function that preserves the positive parts of a vector.

We test with $p = 5000$ and $N=nT=100p$, and consider a well-specified sparse regression model where the true parameter vector $\bar{w}$ is $\bar k$-sparse with $\bar k=0.2p$ and its non-zero entries are sampled from a zero-mean Gaussian distribution. We set $\mu=10^{-3}$ and initialize $w^{(0)} = 0$. {The inner-loop minibatch proximal Lasso subproblems are optimized via a standard proximal gradient descent method, using either of the following two termination criteria: 1) the difference between consecutive objective values is below $10^{-3}$ and 2) the iteration step reaches $1000$.}

\begin{figure}[t]
\mbox{\hspace{-0.25in}
\subfigure[Results under varying $T$. \label{fig:mspp_convergence_T_add}]{
\includegraphics[width=2.3in]{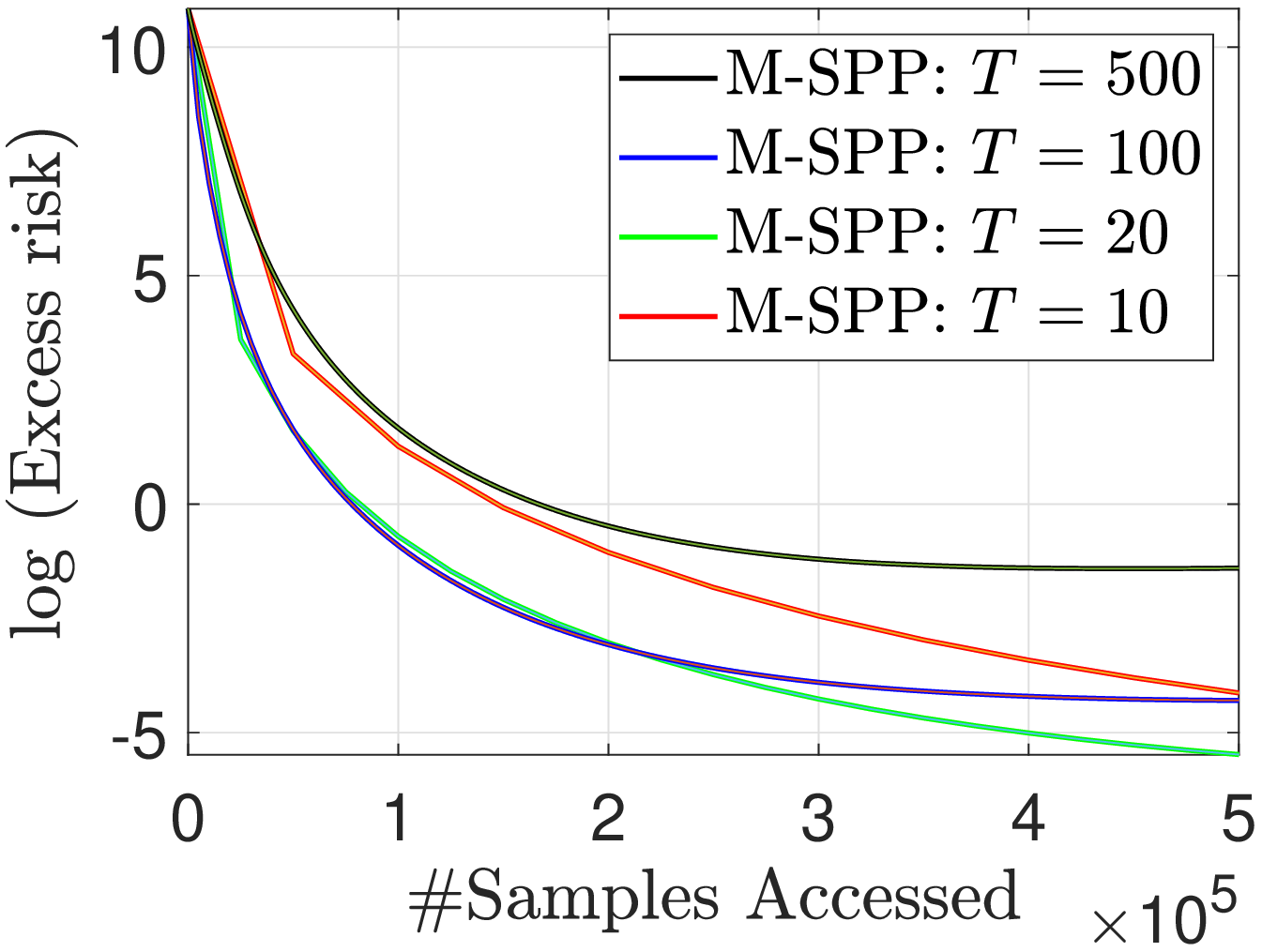}
}\hspace{-0.2in}

\subfigure[Results under varying $\sigma$. \label{fig:mspp_convergence_sigma_add}]{
\includegraphics[width=2.3in]{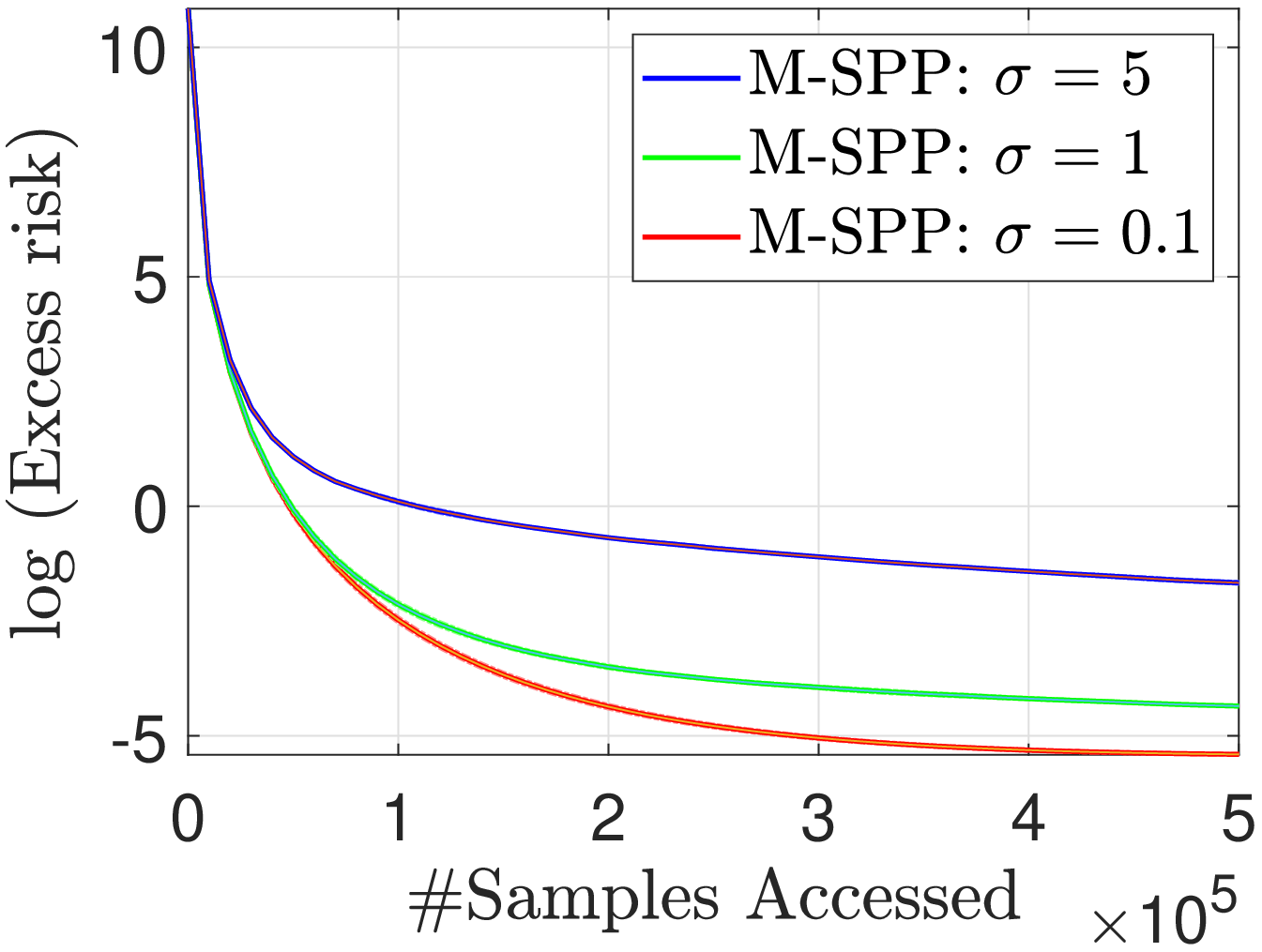}
}\hspace{-0.2in}

\subfigure[M-SPP \emph{versus} M-SPP-TP  \label{fig:two_phase mspp_vs_mspp_add}]{
\includegraphics[width=2.3in]{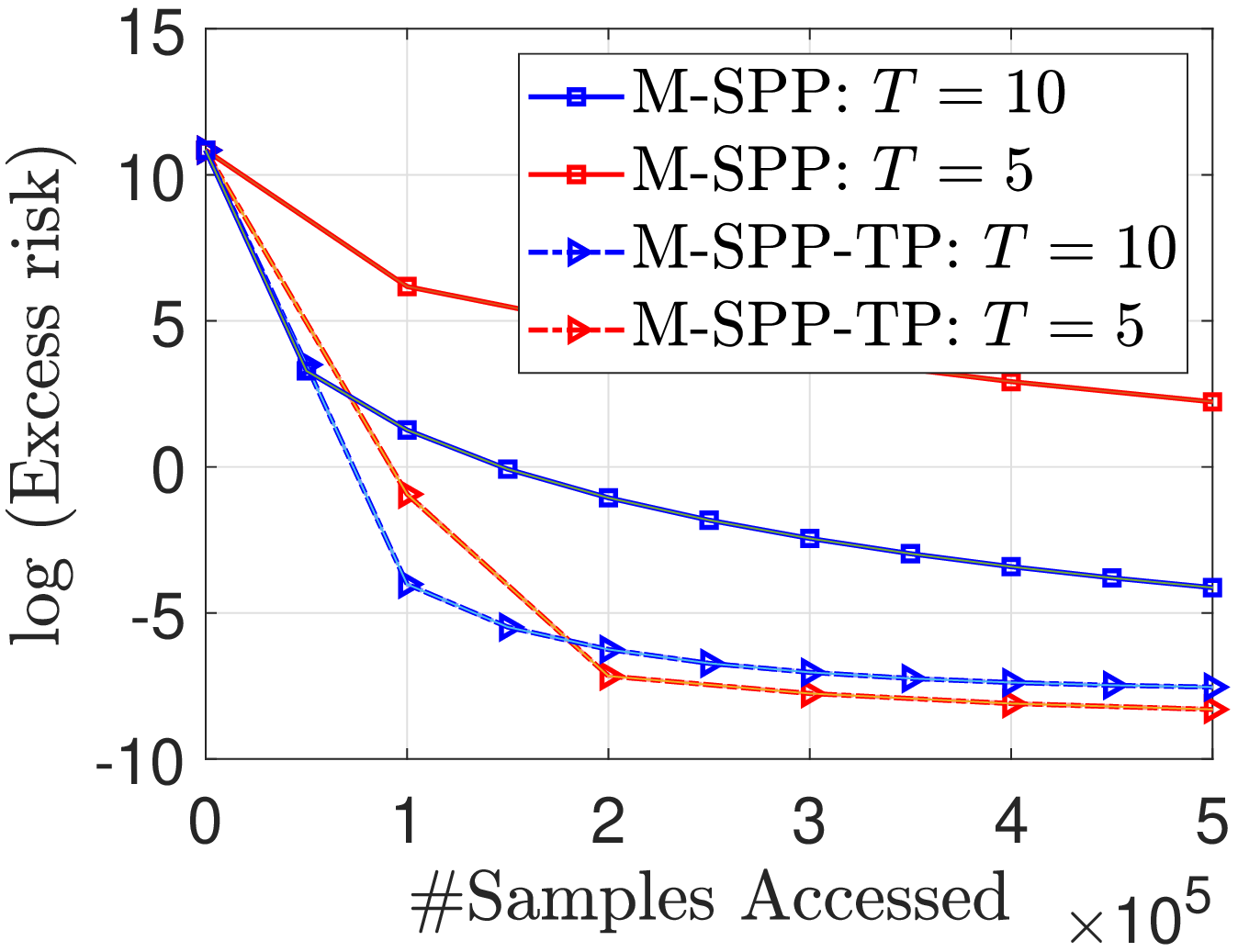}
}
}
\caption{Simulation study on Lasso regression: Convergence performances of M-SPP and M-SPP-TP. {The y-axis represents the logarithmic scale of excess risk.}}
\label{fig:mspp_convergence_add}
\end{figure}

The following two experimental setups are considered for theory verification:
\begin{itemize}%[leftmargin=*]
  \item We fix the noise level $\sigma=0.1$ and study the impact of varying $T\in \{10, 20, 100, 500\}$ on the convergence of M-SPP. {Figure~\ref{fig:mspp_convergence_T_add} shows the evolving curves of excess risk as functions of sample size, in a semi-log layout with y-axis representing the logarithmic scale of excess risk. From this set of curves we can observe a clear trend that in the early stage, M-SPP converges faster when the total number of minibatches is relatively large (say, $T\in \{20,100\}$). This is consistent with the prediction of Theorem~\ref{thrm:fast_rate_smooth} about the impact of $T$ and $n$ on convergence rates. While in the final stage, relatively slower convergence behavior is exhibited under relatively larger $T$ (say, $T\in\{100,500\}$). This observation can be explained by the inexact analysis in Theorem~\ref{thrm:fast_rate_smooth_inexact} which shows that to guarantee the desired convergence rate, the inner-loop proximal ERM update needs to be extremely accurate when $T$ is relatively large. Therefore, the question raised in Question 1 on the impact of minibatch size on convergence rate is answered by this group of results.}

      Also in this setup, we have compared M-SPP and its two-phase variant M-SPP-TP for $T\in\{5,10\}$. The related results are shown in Figure~\ref{fig:two_phase mspp_vs_mspp_add}, which indicate that M-SPP-TP significantly improves the convergence of M-SPP in the small-$T$-large-$n$ cases. This observation supports the result of Theorem~\ref{thrm:fast_rate_smooth_initialization} and answers Question 2 affirmatively.
  \item We fix $T=50$ and study the impact of varying noise level $\sigma\in \{0.1, 1, 5\}$ on the convergence performance. The results are shown in Figure~\ref{fig:mspp_convergence_sigma_add}. From this group of results we can see that faster convergence speed is attained at relatively smaller noise level $\sigma$, while the speed becomes insensitive to noise level when $\sigma$ is sufficiently small (say, $\sigma\le 1$). This is consistent with the predication by Theorem~\ref{thrm:fast_rate_smooth}, keeping in mind the fact that $R^* = \frac{1}{2}\|w^* - \bar w\|^2 + \frac{\sigma^2}{2} + \mu \|w^*\|_1\le \|\bar w\|^2 + \frac{1}{2}\sigma^2$. The question raised in Question 1 on the impact of noise level on convergence performance is answered by this group of results.
\end{itemize}

\subsection{Experiment on Real Data}
\label{ssect:real_data_experiment}

We further compare our methods with M-SGD for binary prediction problems using the logistic loss $\ell(w^\top x, y)=\log(1+\exp(-yw^\top x))$. {Here the M-SGD method is implemented by an SGD solver from SGDLibrary~\citep{kasai2017sgdlibrary}. For M-SPP and M-SPP-TP, the The inner-loop minibatch proximal ERMs are solved by the same SGD solver applied with a fixed SGD-batch-size $10$ and a single epoch of data processing. We initialize $w^{(0)} = 0$ for all the considered methods.}

We use two public data sets for evaluation: the \texttt{gisette} data~\citep{guyon2004result} with $p=5000, N=6000$ and the \texttt{covtype.binary} data~\citep{collobert2001parallel} with $p=54, N=581,012$~\footnote{Both data sets are available at~\url{https://www.csie.ntu.edu.tw/~cjlin/libsvmtools/datasets/}.}. For each data set, we use half of the samples as training set and the rest as test set. We are interested in the impact of minibatch-size $n$ on the prediction performance of model measured by test error. All the considered stochastic algorithms are executed with 10 epochs of data processing, and thus the overall number of minibatches is $T=N/n \times 10$. We replicate each experiment $10$ times over random split of data and report the results in mean-value along with error bar.

\begin{figure}
\mbox{\hspace{-0.25in}
\subfigure[$n=N/5$~\label{fig:mspp_convergence_T_gistee_T_5}]{
\includegraphics[width=2.3in]{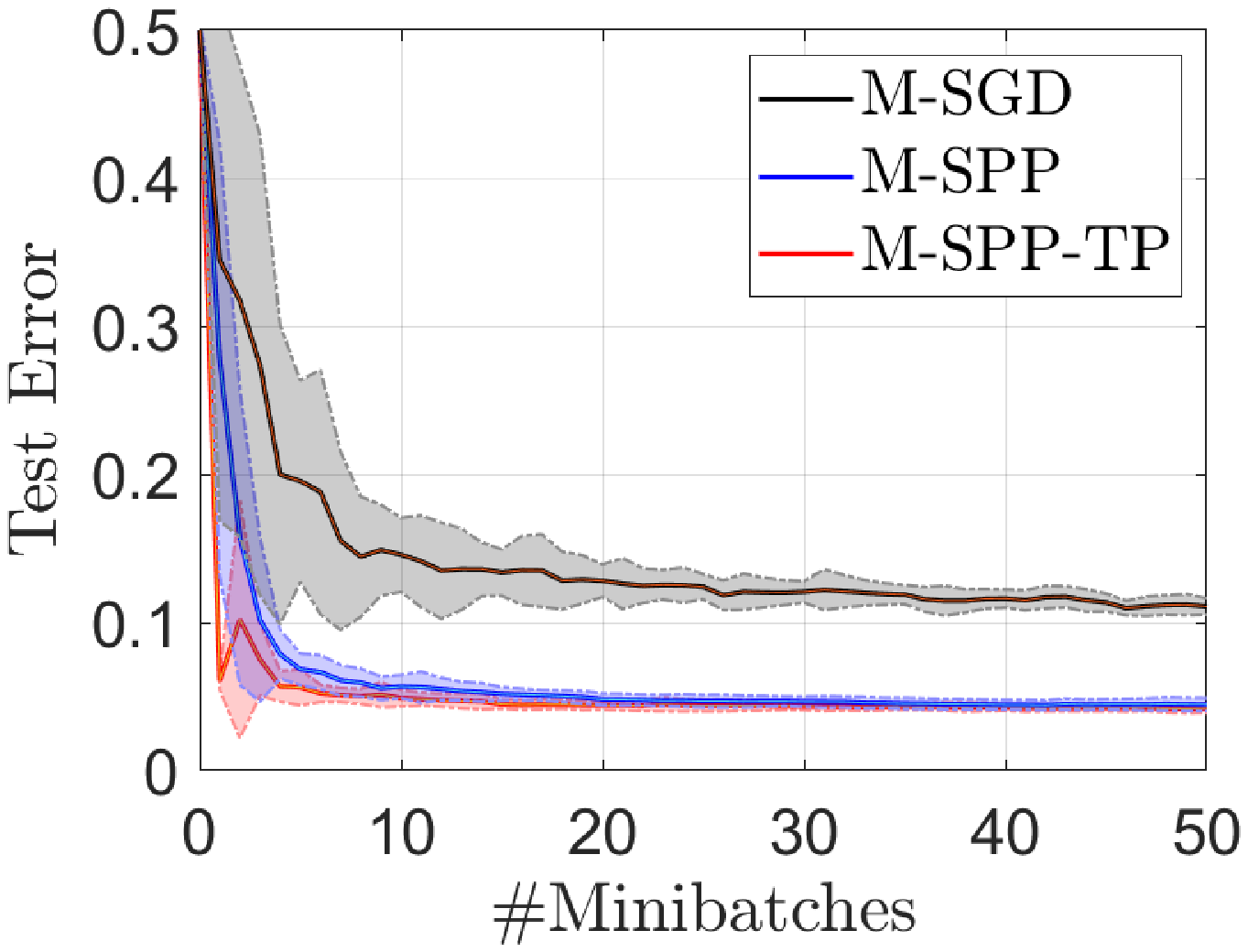}
}\hspace{-0.2in}
\subfigure[$n=N/20$]{
\includegraphics[width=2.3in]{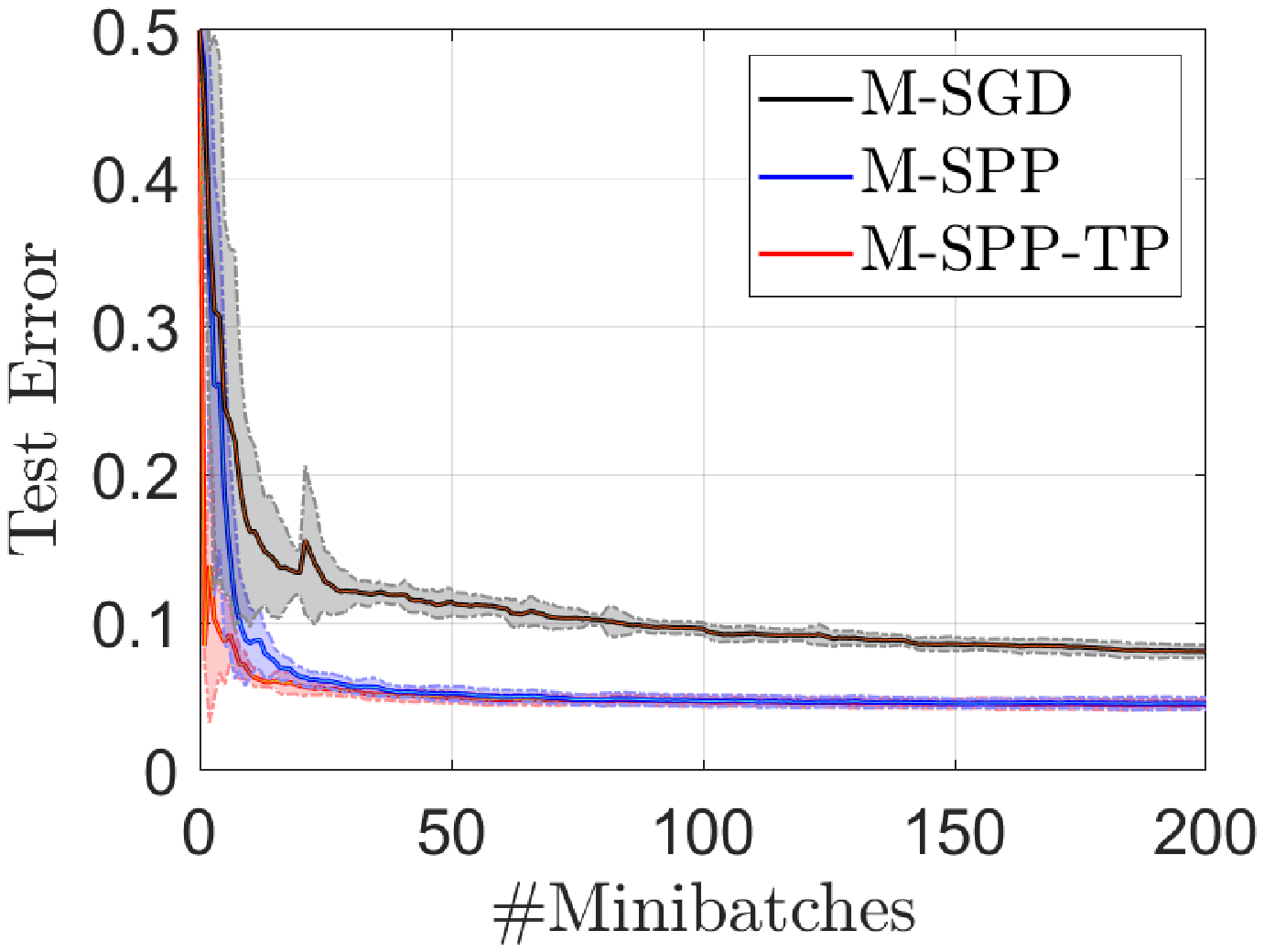}
}\hspace{-0.2in}
\subfigure[$n=N/100$]{
\includegraphics[width=2.3in]{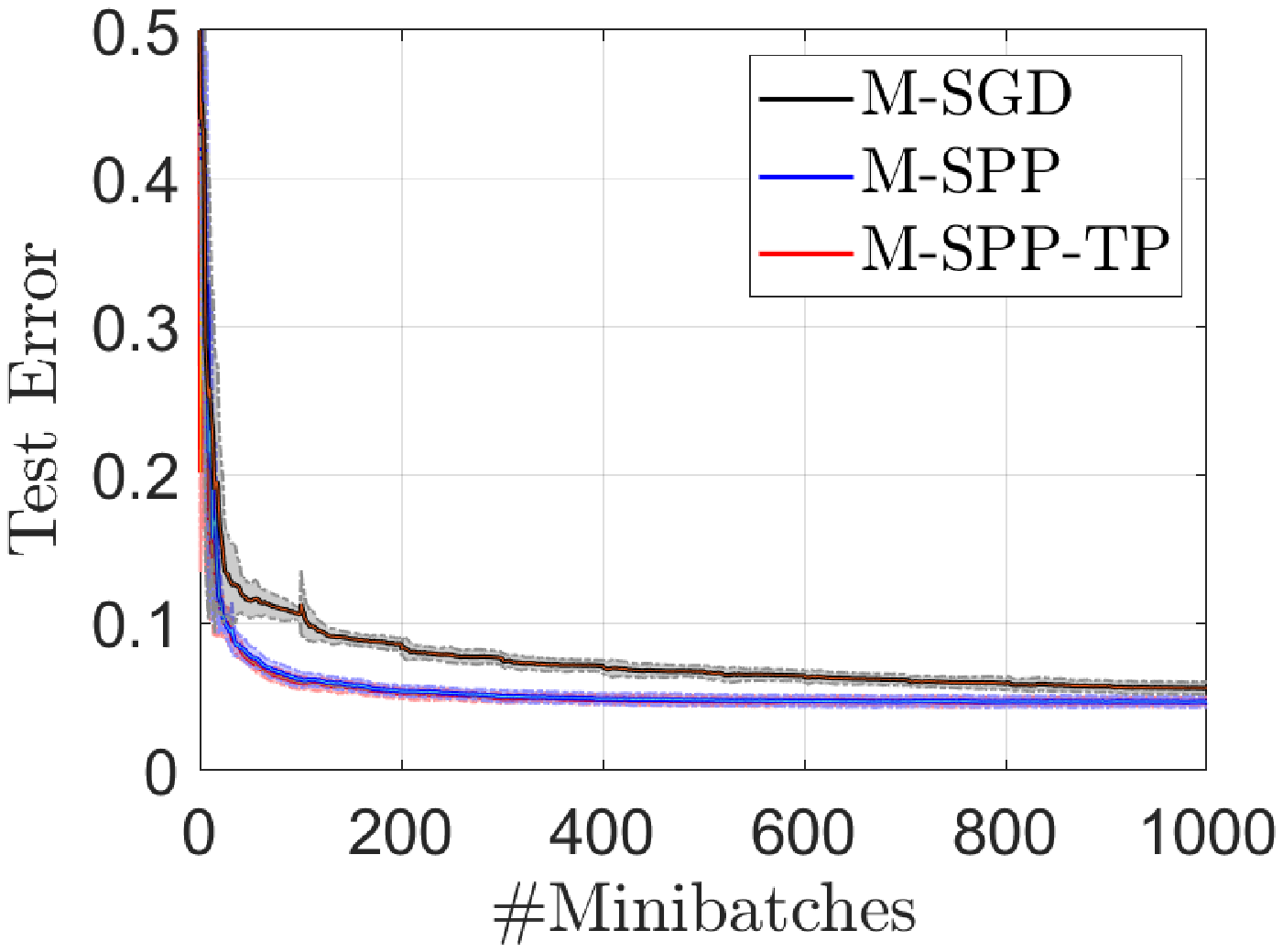}
}
}

\vspace{-0.1in}

\caption{Real-data results on logistic regression: Test error convergence comparison on \texttt{gisette} under varying minibatch size.}
\label{fig:mspp_convergence_T_gistee}
\end{figure}

{In Figure~\ref{fig:mspp_convergence_T_gistee}, we show the evolving curves (error bar shaded in color) of test error with respect to the number of minibatches accessed on \texttt{gisette}, under varying minibatch size $n\in\{\frac{N}{5}, \frac{N}{20}, \frac{N}{100} \}$.} From this set of curves we can observe that:
\begin{itemize}
  \item {Under the same minibatch size, M-SPP and M-SPP-TP converge faster and stabler than M-SGD, especially when the minibatch size is relatively large (see Figure~\ref{fig:mspp_convergence_T_gistee_T_5}). This is as expected because when minibatch size becomes large, M-SGD approaches to gradient descent method while M-SPP approaches ERMs. This answers Question 3 raised at the beginning of the experiment section.}
  \item M-SPP-TP exhibits sharper convergence behavior than M-SPP at the early stage of iteration, especially when the minibatch-size is relatively large. This is consistent with our theoretical results in Theorem~\ref{thrm:fast_rate_smooth} and Theorem~\ref{thrm:fast_rate_smooth_initialization}.
\end{itemize}
{Figure~\ref{fig:mspp_convergence_T_covtype} shows the corresponding results on \texttt{covtype} under $n\in\left\{\frac{N}{20}, \frac{N}{100}, \frac{N}{1000}\right\}$. From this set of results we once again see that M-SPP and M-SPP-TP consistently outperform M-SGD under the same minibatch size, and M-SPP-TP converges faster than M-SPP under relatively larger minibatch size (say, $n=\frac{N}{20}$).}

\begin{figure}
\mbox{\hspace{-0.25in}
\subfigure[$n=N/20$]{
\includegraphics[width=2.3in]{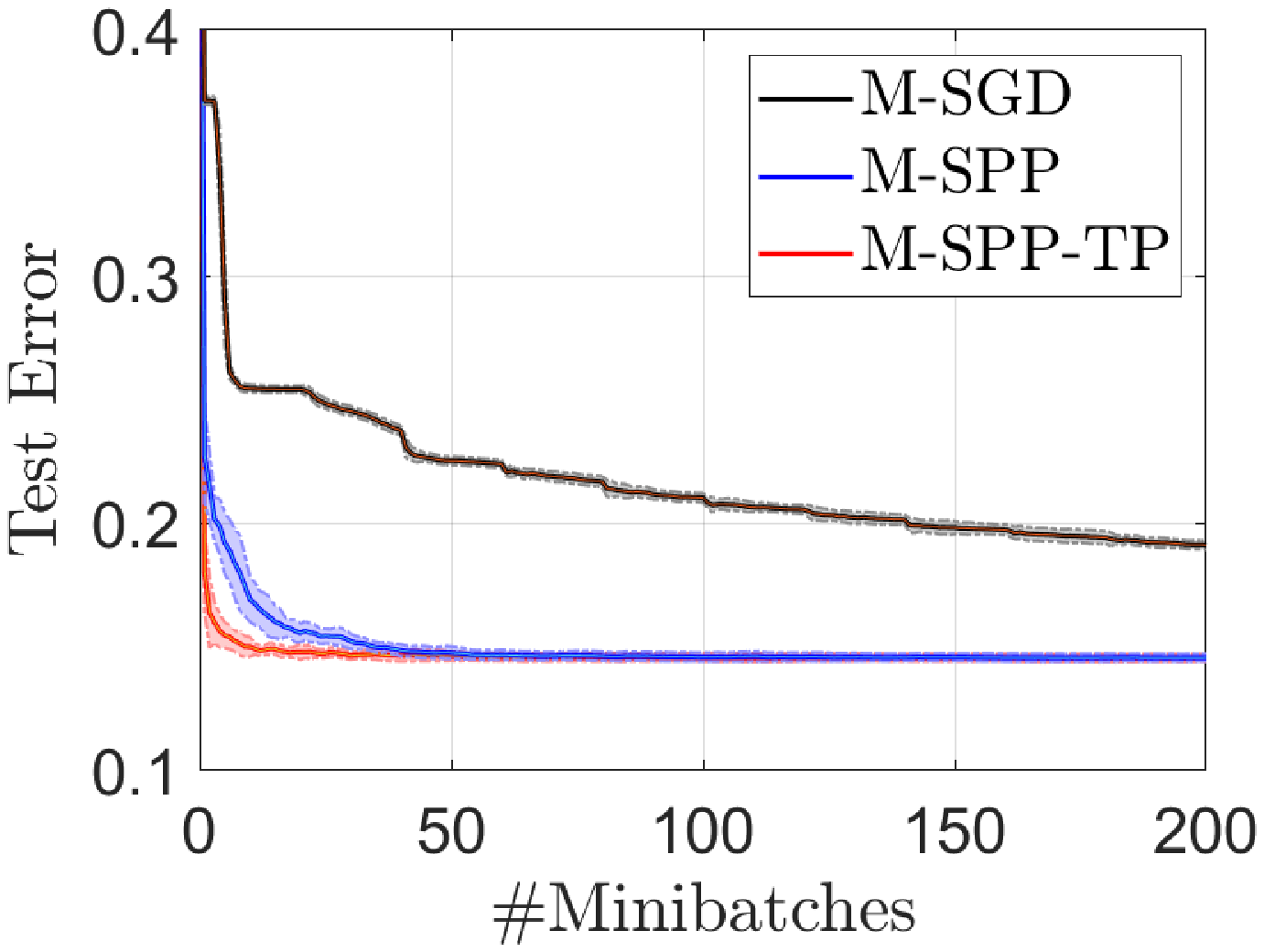}
}\hspace{-0.2in}
\subfigure[$n=N/100$]{
\includegraphics[width=2.3in]{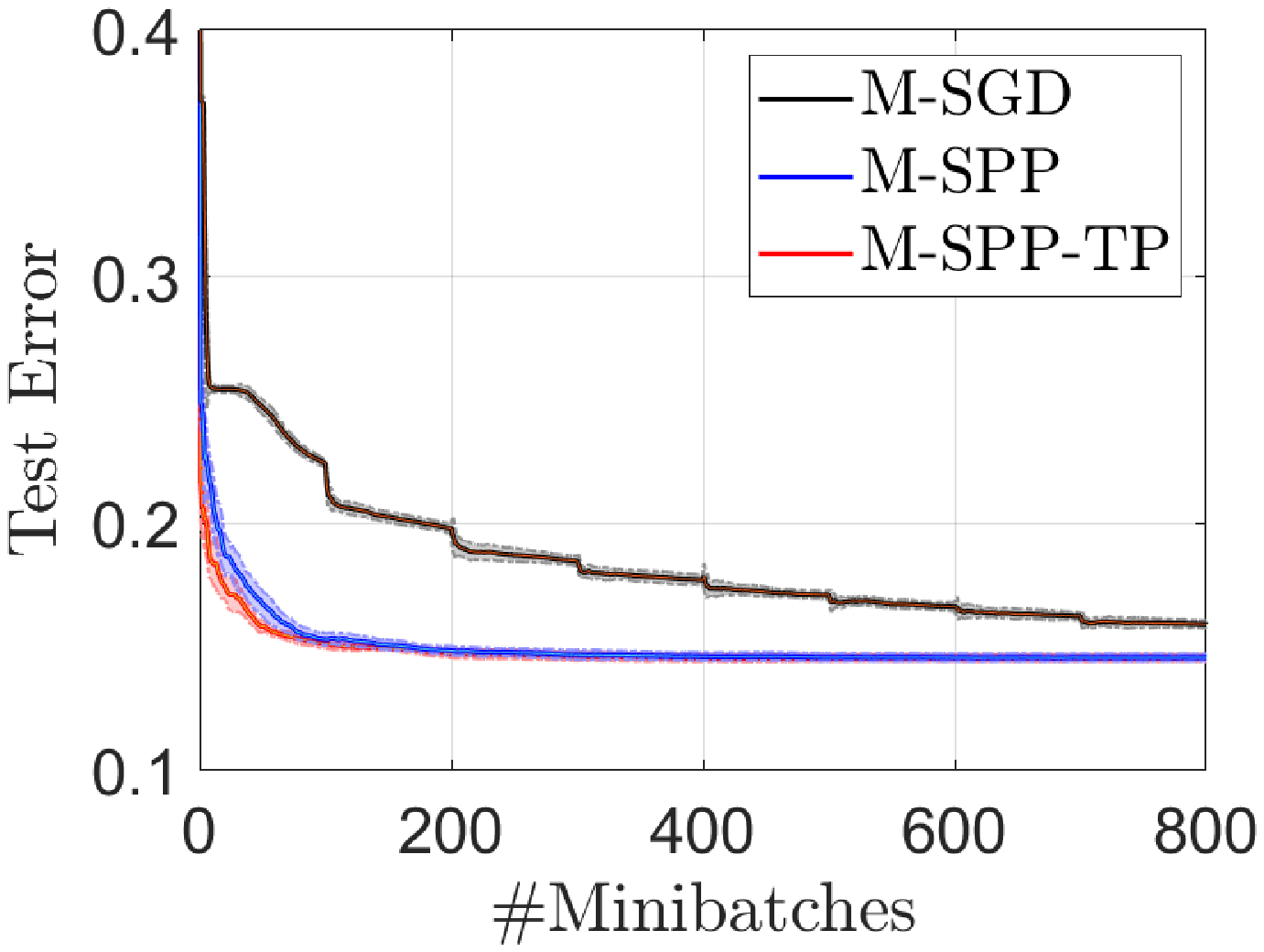}
}\hspace{-0.2in}
\subfigure[$n=N/1000$]{
\includegraphics[width=2.3in]{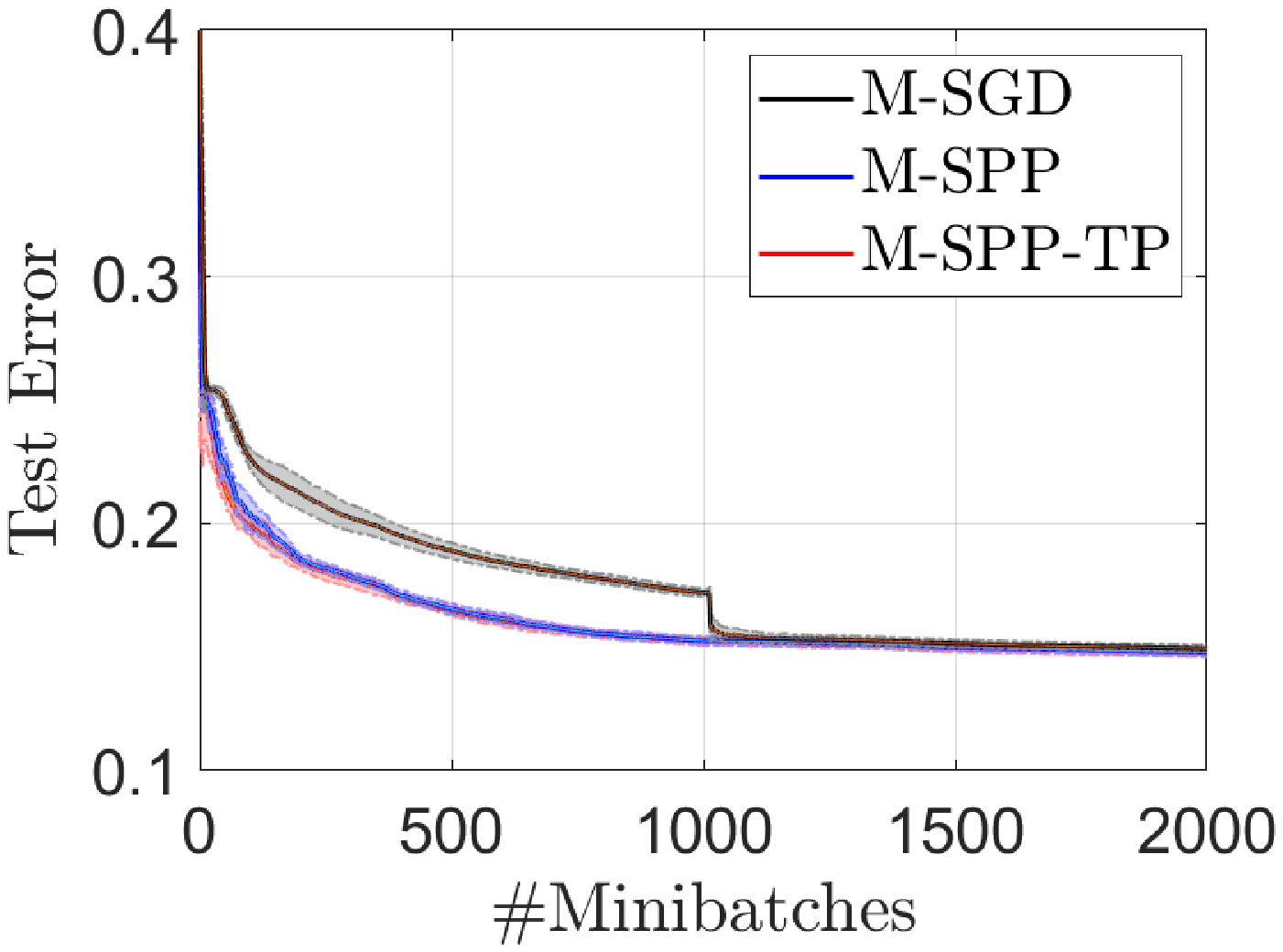}
}}
\caption{Real-data results on logistic regression: Test error convergence comparison on \texttt{covtype.binary} under varying minibatch size.}
\label{fig:mspp_convergence_T_covtype}
\end{figure}

\section{Conclusions and Future Prospects}
\label{sect:conclusion}

In this article, we presented an improved convergence analysis for the minibatch stochastic proximal point methods with smooth and convex losses. Under the quadratic growth condition on population risk, we showed that M-SPP with minibatch-size $n$ and iteration count $T$ converges at a composite rate consisting of an $\mathcal{O}(\frac{1}{T^2})$ \emph{bias} decaying component and an $\mathcal{O}(\frac{1}{N})$ \emph{variance} decaying component. In the small-$n$-large-$T$ case, this result substantially improves the prior relevant results of SPP-type approaches which typically require each instantaneous loss to be Lipschitz and strongly convex. Complementally in the small-$T$-large-$n$ setting, we provide a two-phase acceleration of M-SPP which improves the $\mathcal{O}(\frac{1}{T^2})$ bias decaying rate to $\mathcal{O}\left(\frac{\log(N)}{N^2}\right)$. Perhaps the most interesting theoretical finding is that the (dominant) variance decaying term has a factor dependence on the minimal value of population risk, justifying the sharper convergence behavior of M-SPP in low-noise statistical setting as backed up by our numerical evidence. In addition to the in-expectation risk bounds, we have also derived a near-optimal parameter estimation error bound for a random shuffling variant of M-SPP that holds with high probability over data distribution and in expectation over the random shuffling. To conclude, our theory lays a novel and stronger foundation for understanding the convex M-SPP style algorithms that have gained recent significant attention, both in theory and practice, for large-scale machine learning~\citep{li2014efficient,wang2017efficient,asi2020minibatch}.

\newpage

There are several key prospects for future investigation of our theory:
\begin{itemize}
  \item It is still open to derive near-optimal exponential excess risk bounds for M-SPP that apply to the (suffix) average or last of iterates over training data.
  \item Inspired by the recent progresses made towards understanding M-SPP with momentum acceleration~\citep{deng2021minibatch,chadha2022accelerated}, it is interesting to provide momentum and weakly-convex extensions of our theory for smooth loss functions.
  \item Last but not least, we expect that the theory developed in this article can be extended to the setup of non-parametric learning with minibatch stochastic proximal point methods.
\end{itemize}

%Also, it is interesting to further derive deviation excess risk bounds for random shuffling M-SPP that hold with high probability simultaneously over training data and the internal randomness of algorithm. Last but not least, we expect that the theory developed in this article can be extended to the setup of non-parametric stochastic optimization with minibatch stochastic proximal point methods.
%For future research, it is still open to know if the excess risk bounds established so far are tight, or can be further improved. Also, it is interesting to derive deviation excess risk bounds for random shuffling M-SPP that hold in high probability with respect to the randomness of sampling as well.

\section*{Acknowledgements}

The authors sincerely thank the anonymous referees for their constructive comments. The work of Xiao-Tong Yuan is also funded in part by the National Key Research and Development Program of China under Grant No. 2018AAA0100400 and in part by the Natural Science Foundation of China (NSFC) under Grant No.U21B2049, No.61876090 and No.61936005.

\newpage

\appendix

\section{Proofs for the Results in Section~\ref{sect:analysis_smooth}}
\label{apdsect:analysis_smooth_proofs}
In this section, we present the technical proofs for the main results stated in Section~\ref{sect:analysis_smooth}.

\subsection{Proof of Theorem~\ref{thrm:fast_rate_smooth}}
\label{apdsect:proof_fast_rate_smooth}

Here we prove Theorem~\ref{thrm:fast_rate_smooth} as restated below for convenience.
\FastRateSmooth*

We first present the following lemma which will be used in the proof. It can be viewed as a straightforward extension of the prior result~\cite[Lemma 1]{wang2017memory} to the setup of composite minimization. A proof is included here for the sake of completeness.
\begin{lemma}\label{lemma:function_strong_convexity}
Assume that the loss function $\ell$ is convex with respect to its first argument and the regularization function $r$ is convex. Then for any $w\in \mathcal{W}$, we have
\[
R_{S_t}(w_t) -  R_{S_t}(w) \le \frac{\gamma_t}{2}\left( \|w - w_{t-1} \|^2 - \|w - w_t\|^2 - \|w_t - w_{t-1}\|^2\right).
\]
\end{lemma}
\begin{proof}
Since $\ell$ and $r$ are both convex, $R_{S_t}$ is convex over $\mathcal{W}$. The optimality of $w_t$ implies that for any $w \in \mathcal{W}$ and $\eta \in(0,1)$
\[
\begin{aligned}
&R_{S_t}(w_t) + \frac{\gamma_t}{2} \|w_t - w_{t-1}\|^2 \le R_{S_t}((1-\eta)w_t + \eta w) + \frac{\gamma_t}{2} \|(1-\eta)w_{t} + \eta w - w_{t-1}\|^2 \\
\le& (1-\eta)  R_{S_t}(w_t) + \eta R_{S_t}(w) + \frac{\gamma_t}{2}\left[(1-\eta) \|w_t - w_{t-1}\|^2 + \eta \|w - w_{t-1}\|^2 - \eta(1-\eta)\|w - w_t\|^2 \right],
\end{aligned}
\]
where in the last inequality we have used the definition of the norm $\|\cdot\|$. Rearranging both sides of the above inequality yields
\[
\eta ( R_{S_t}(w_t) -  R_{S_t}(w)) \le  \frac{\eta \gamma_t}{2}\left[ \|w - w_{t-1}\|^2 - (1-\eta)\|w - w_t\|^2 - \|w_t - w_{t-1}\|^2\right],
\]
which then implies (keep in mind that $\eta>0$)
\[
R_{S_t}(w_t) -  R_{S_t}(w) \le  \frac{\gamma_t}{2}\left[ \|w - w_{t-1}\|^2 - (1-\eta)\|w - w_t\|^2 - \|w_t - w_{t-1}\|^2 \right].
\]
Limiting $\eta \rightarrow 0^+$ in the above inequality yields the desired bound.
\end{proof}

\newpage

The following  boundedness result for smooth function is due to~\citet[Lemma 3.1]{srebro2010smoothness}.
\begin{lemma}\label{lemma:key_smooth}
If $g$ is non-negative and $L$-smooth, then $\|\nabla g(w)\| \le \sqrt{2Lg(w)}$.
\end{lemma}
%\begin{proof}
%Since $g$ is non-negative and $L$-smooth, it can been shown that $\|\nabla g(w)\| \le \sqrt{2Lg(w)}$~\cite[Lemma 3.1]{srebro2010smoothness}. Consequently, the desired bound follows  directly according to the definition of smoothness and triangle inequality.
%\end{proof}

Let $\{\mathcal{F}_{t}\}_{t\ge1}$ be the filtration generated by the iterates $\{w_{t}\}_{t\ge1}$ as $\mathcal{F}_t = \sigma\left(w_1, w_2,...,w_t\right)$. With Lemma~\ref{lemma:function_strong_convexity} and Lemma~\ref{lemma:key_smooth} in place, we can further establish the following key lemma that plays a fundamental role in proving Theorem~\ref{thrm:fast_rate_smooth}.

\begin{lemma}\label{lemma:telescope_smooth}
Suppose that the Assumptions~\ref{assump:smooth} holds. Set $\gamma_t \ge \frac{16 L}{n}$. Then we have
\[
\mathbb{E}\left[R(w_t) -  R^*\mid \mathcal{F}_{t-1}\right] \le \gamma_t\left(D^2(w_{t-1}, W^*) - \mathbb{E}\left[D^2(w_t, W^*)\mid \mathcal{F}_{t-1} \right]\right) + \frac{16L}{\gamma_t n} R^*.
\]
\end{lemma}
\begin{proof}
Let us consider a sample set $S_t^{(i)}$ which is identical to $S_t$ except that one of the $z_{i,t}$ is replaced by another random sample $z'_{i,t}$. Denote
\[
w^{(i)}_t = \argmin_{w\in \mathcal{W}} \left\{F^{(i)}_{t}(w):= R_{S^{(i)}_t}(w) + \frac{\gamma_t}{2} \|w - w_{t-1}\|^2\right\},
\]
where $R_{S^{(i)}_t}(w):=\frac{1}{n}\left(\sum_{j\neq i} \ell(w; z_{j,t}) + \ell(w; z'_{i,t})\right) + r(w)$. Then we can show that
\[
\begin{aligned}
&F_t(w^{(i)}_t) - F_t(w_t) \\
=& \frac{1}{n}\sum_{j\neq i} \left(\ell(w^{(i)}_t; z_{j,t}) - \ell(w_t;z_{j,t}) \right) + \frac{1}{n} \left(\ell(w^{(i)}_t;z_{i,t}) - \ell(w_t;z_{i,t})\right) \\
& + r(w^{(i)}_t) - r(w_t) + \frac{\gamma_t}{2} \|w^{(i)}_t - w_{t-1}\|^2 - \frac{\gamma_t}{2} \|w_t - w_{t-1}\|^2 \\
=& F^{(i)}_{t}(w^{(i)}_t) - F^{(i)}_{t}(w_t) + \frac{1}{n} \left(\ell(w^{(i)}_t;z_{i,t}) - \ell(w_t; z_{i,t})\right) - \frac{1}{n} \left(\ell(w^{(i)}_t;z'_{i,t}) - \ell(w_t;z'_{i,t})\right)\\
\le& \frac{1}{n} \left|\ell(w^{(i)}_t;z_{i,t}) - \ell(w_t;z_{i,t})\right| + \frac{1}{n} \left|\ell(w^{(i)}_t; z'_{i,t}) - \ell(w_t; z'_{i,t})\right| \\
\overset{\zeta_1}{\le}&  \frac{\|\nabla \ell(w^{(i)}_t;z_{i,t})\| + \|\nabla\ell(w_t;z'_{i,t})\|}{n} \|w^{(i)}_t - w_t\|\\
\overset{\zeta_2}{\le}& \frac{\sqrt{2L\ell(w^{(i)}_t;z_{i,t})} + \sqrt{2L\ell(w_t;z'_{i,t})}}{n} \|w^{(i)}_t - w_t\|,
\end{aligned}
\]
where ``$\zeta_1$'' is due to the convexity of loss and in ``$\zeta_2$''we have used Lemma~\ref{lemma:key_smooth}. The bound in Lemma~\ref{lemma:function_strong_convexity} implies
\[
F_t(w^{(i)}_t) - F_t(w_t) \ge \frac{\gamma_t}{2} \|w^{(i)}_t - w_t\|^2.
\]
Combining the preceding two inequalities yields
\[
\frac{\gamma_t}{2} \|w^{(i)}_t - w_t\| \le \frac{\sqrt{2L\ell(w^{(i)}_t;z_{i,t})} + \sqrt{2L\ell(w_t;z'_{i,t})}}{n},
\]
which immediately gives
\begin{equation}\label{inequat:proof_exp_risk_smooth_key_1}
\|w^{(i)}_t - w_t\| \le \frac{2\left(\sqrt{2L\ell(w^{(i)}_t;z_{i,t})} + \sqrt{2L\ell(w_t;z'_{i,t})}\right)}{\gamma_t n}.
\end{equation}
Let us now consider the following population risk and empirical risk over $S_t$ with respect to the loss function $\ell$:
\[
 R^\ell(w):=\mathbb{E}_{(x,y) \sim \mathcal{D}} [\ell(w; z)], \quad R^\ell_{S_t}(w):=\frac{1}{n}\sum_{i=1}^n\ell(w; z_{i,t}).
\]
Since $S_t$ and $S_t^{(i)}$ are both i.i.d. samples of the data distribution. It follows that
\[
\begin{aligned}
&\mathbb{E}_{S_t}\left[ R^\ell(w_t)\mid \mathcal{F}_{t-1} \right] = \mathbb{E}_{S_t\cup \{z'_{i,t}\}}\left[ \ell(w_t;z'_{i,t})\mid \mathcal{F}_{t-1} \right]\\
=& \mathbb{E}_{S_t^{(i)}}\left[ R^\ell (w^{(i)}_t) \mid \mathcal{F}_{t-1} \right] = \mathbb{E}_{S_t^{(i)}\cup \{z_{i,t}\}}\left[ \ell(w^{(i)}_t;z_{i,t}) \mid \mathcal{F}_{t-1} \right] \\
.
\end{aligned}
\]
Since the above holds for all $i=1,...,n$, we can further show that
\begin{equation}\label{inequat:proof_exp_risk_smooth_key_2}
\begin{aligned}
&\mathbb{E}_{S_t}\left[ R^\ell(w_t)\mid \mathcal{F}_{t-1} \right] \\
=& \frac{1}{n}\sum_{i=1}^n \mathbb{E}_{S_t^{(i)}\cup \{z_{i,t}\}}\left[ \ell(w^{(i)}_t;z_{i,t})\mid \mathcal{F}_{t-1} \right] = \frac{1}{n}\sum_{i=1}^n \mathbb{E}_{S_t\cup \{z'_{i,t}\}}\left[ \ell(w^{(i)}_t;z_{i,t}) \mid \mathcal{F}_{t-1} \right] \\
=&\frac{1}{n}\sum_{i=1}^n \mathbb{E}_{S_t\cup \{z'_{i,t}\}}\left[ \ell(w_t;z'_{i,t})\mid \mathcal{F}_{t-1}\right] = \frac{1}{n}\sum_{i=1}^n \mathbb{E}_{S^{(i)}_t\cup \{z_{i,t}\}}\left[ \ell(w_t;z'_{i,t})\mid \mathcal{F}_{t-1}\right].
\end{aligned}
\end{equation}
Regarding the empirical case, we find that
\[
\begin{aligned}
&\mathbb{E}_{S_t}\left[ R^\ell_{S_t}(w_t)\mid \mathcal{F}_{t-1}\right] \\
=&  \frac{1}{n}\sum_{i=1}^n \mathbb{E}_{S_t}\left[ \ell(w_t;z_{i,t}) \mid \mathcal{F}_{t-1} \right] = \frac{1}{n}\sum_{i=1}^n \mathbb{E}_{S_t\cup \{z'_{i,t}\}}\left[ \ell(w_t;z_{i,t}) \mid \mathcal{F}_{t-1} \right].
\end{aligned}
\]
Combining the preceding two equalities gives that
\[
\begin{aligned}
&\left|\mathbb{E}_{S_t}\left[ R(w_t) - R_{S_t}(w_t) \mid \mathcal{F}_{t-1} \right] \right| \\
=& \left|\mathbb{E}_{S_t}\left[ R^\ell(w_t) - R^\ell_{S_t}(w_t) \mid \mathcal{F}_{t-1} \right] \right| \\
=& \left| \frac{1}{n}\sum_{i=1}^n \mathbb{E}_{S_t\cup \{z'_{i,t}\}} \left[\ell(w^{(i)}_t;z_{i,t}) - \ell(w_t;z_{i,t}) \mid \mathcal{F}_{t-1} \right] \right| \\
\le& \frac{1}{n}\sum_{i=1}^n  \mathbb{E}_{S_t\cup \{z'_{i,t}\}} \left[\left|\ell(w^{(i)}_t;z_{i,t}) - \ell(w_t;z_{i,t})\right| \mid \mathcal{F}_{t-1} \right] \\
\le& \frac{1}{n}\sum_{i=1}^n  \mathbb{E}_{S_t\cup \{z'_{i,t}\}} \left[ \sqrt{2 L\ell(w^{(i)}_t;z_{i,t})}\|w^{(i)}_t - w_t\| \mid \mathcal{F}_{t-1} \right] \\
\overset{\eqref{inequat:proof_exp_risk_smooth_key_1}}{\le}& \frac{1}{n}\sum_{i=1}^n  \mathbb{E}_{S^{(i)}_t\cup \{z_{i,t}\}} \left[ \frac{4L\ell(w^{(i)}_t;z_{i,t})}{\gamma_t n} + \frac{4L\sqrt{\ell(w^{(i)}_t;z_{i,t})\ell(w_t;z'_{i,t})}}{\gamma_t n}  \mid \mathcal{F}_{t-1} \right]\\
\overset{\zeta_1}{\le}& \left(\frac{L}{\gamma_t n} \right)\frac{1}{n}\sum_{i=1}^n  \mathbb{E}_{S_t\cup \{z'_{i,t}\}} \left[6\ell(w^{(i)}_t;z_{i,t}) + 2\ell(w_t;z'_{i,t}) \mid \mathcal{F}_{t-1} \right] \\
\overset{\eqref{inequat:proof_exp_risk_smooth_key_2}}{=}& \frac{8L}{\gamma_t n} \mathbb{E}_{S_t}\left[ R^\ell(w_t) \mid \mathcal{F}_{t-1} \right] \le \frac{8L}{\gamma_t n} \mathbb{E}_{S_t}\left[ R(w_t) \mid \mathcal{F}_{t-1} \right],
\end{aligned}
\]
where in ``$\zeta_1$'' we have used the fact $a^2 +b^2\ge 2ab$ and the last inequality is due to the fact $r\ge 0$.

Let us now denote $w^*_t=\argmin_{w\in W^*}\|w-w_t\|$. Conditioned on $\mathcal{F}_{t-1}$, taking expectation on both sides of the bound in Lemma~\ref{lemma:function_strong_convexity} for $w=w^*_{t-1}$ yields
\[
\begin{aligned}
&\mathbb{E}_{S_t}\left[R_{S_t}(w_t) -  R^*\mid \mathcal{F}_{t-1}\right] \\
\le& \frac{\gamma_t}{2}\mathbb{E}_{S_t}\left[ \|w^*_{t-1} - w_{t-1} \|^2 - \|w^*_{t-1} - w_t\|^2 - \|w_t - w_{t-1}\|^2\mid \mathcal{F}_{t-1} \right] \\
\le& \frac{\gamma_t}{2} \left(\|w^*_{t-1} - w_{t-1} \|^2 - \mathbb{E}_{S_t}\left[\|w^*_t - w_t\|^2 \mid \mathcal{F}_{t-1} \right]\right).
\end{aligned}
\]
Combining the preceding two inequalities yields
\[
\begin{aligned}
&\mathbb{E}_{S_t}\left[R(w_t) -  R^*\mid \mathcal{F}_{t-1}\right] \\
=& \mathbb{E}_{S_t}\left[R(w_t) - R_{S_t}(w_t) + R_{S_t}(w_t) -  R^*\mid \mathcal{F}_{t-1} \right] \\
\le& \left|\mathbb{E}_{S_t}\left[R(w_t) - R_{S_t}(w_t)\mid \mathcal{F}_{t-1} \right]\right| + \mathbb{E}_{S_t}\left[R_{S_t}(w_t) -  R^* \mid \mathcal{F}_{t-1} \right] \\
\le& \frac{\gamma_t}{2}\left( \|w^*_{t-1} - w_{t-1} \|^2 - \mathbb{E}_{S_t}\left[\|w^*_t - w_t\|^2\mid \mathcal{F}_{t-1} \right]\right) + \frac{8L}{\gamma_t n} \mathbb{E}_{S_t}\left[ R(w_t) \mid \mathcal{F}_{t-1} \right] \\
=& \frac{\gamma_t}{2}\left( \|w_{t-1}^* - w_{t-1} \|^2 - \mathbb{E}_{S_t}\left[\|w^*_t - w_t\|^2\mid \mathcal{F}_{t-1} \right]\right) + \frac{8L}{\gamma_t n} \mathbb{E}_{S_t}\left[ R(w_t) - R^* \mid \mathcal{F}_{t-1} \right] + \frac{8L}{\gamma_t n} \mathbb{E}_{S_t}\left[R^* \right] \\
\le& \frac{\gamma_t}{2}\left( \|w^*_{t-1} - w_{t-1} \|^2 - \mathbb{E}_{S_t}\left[\|w^*_t - w_t\|^2 \mid \mathcal{F}_{t-1} \right]\right) + \frac{1}{2}\mathbb{E}_{S_t}\left[ R(w_t) - R^* \mid \mathcal{F}_{t-1} \right] + \frac{8L}{\gamma_t n} R^*,
\end{aligned}
\]
where in the last inequality we have used the condition $\gamma_t \ge \frac{52 L}{n}$. After rearranging the terms in the above inequality we obtain
\[
\begin{aligned}
\mathbb{E}_{S_t}\left[R(w_t) -  R^* \mid \mathcal{F}_{t-1} \right] \le& \gamma_t\left(\|w^*_{t-1} - w_{t-1} \|^2 - \mathbb{E}_{S_t}\left[\|w^*_t - w_t\|^2\mid \mathcal{F}_{t-1} \right]\right) + \frac{16L}{\gamma_t n} R^* \\
=& \gamma_t\left(D^2(w_{t-1}, W^*) - \mathbb{E}_{S_t}\left[D^2(w_t,W^*)\mid \mathcal{F}_{t-1} \right]\right) + \frac{16L}{\gamma_t n} R^*.
\end{aligned}
\]
This implies the desired bound.
\end{proof}

The following lemma is a direct consequence of Lemma~\ref{lemma:telescope_smooth}.
\begin{lemma}\label{lemma:distance_bound}
Suppose that the Assumptions~\ref{assump:smooth} holds. Set $\gamma_t \ge \frac{16 L}{n}$. Then the following holds for all $t\ge 1$:
\[
\mathbb{E}\left[D^2( w_t,  W^*)\right] \le D^2( w_0,  W^*) + \sum_{\tau = 1}^t\frac{16L}{\gamma_\tau^2 n} R^*.
\]
\end{lemma}
\begin{proof}
Since $R(w_t) \ge R^*$ and $\gamma_t\ge \frac{52L}{n}$, the bound in Lemma~\ref{lemma:telescope_smooth} immediately implies that
\begin{equation}\label{equat:proof_iterstability_key}
\mathbb{E}_{S_t}\left[D^2( w_t,  W^*)\mid \mathcal{F}_{t-1} \right] \le D^2( w_{t-1},  W^*) + \frac{16L}{\gamma^2_t n}  R^*.
\end{equation}
By unfolding the above recurrent from time instance $t$ to zero we obtain that for all $t\ge 1$,
\[
\mathbb{E}\left[D^2( w_t,  W^*)\right] \le D^2( w_0,  W^*) + \sum_{\tau = 1}^t\frac{16L}{\gamma_\tau^2 n} R^*.
\]
This proves the desired bound.
\end{proof}

With all these lemmas in place, we are now ready to prove the main result in Theorem~\ref{thrm:fast_rate_smooth}.

\vspace{0.1in}

\begin{proof}[Proof of Theorem~\ref{thrm:fast_rate_smooth}] \textbf{Part (a):} Note that the condition on $n$ implies $\gamma_t = \frac{\lambda \rho t}{4} \ge \frac{\lambda\rho}{4}\ge \frac{16 L}{n}$. Applying Lemma~\ref{lemma:telescope_smooth} along with the condition $R(w_t) -  R^*\ge \frac{\lambda}{2} D^2(w_t, W^*)$ yields
\[
\begin{aligned}
&(1-\rho)\mathbb{E}\left[R(w_t) -  R^*\mid \mathcal{F}_{t-1} \right] \\
\le& \gamma_t D^2(w_{t-1}, W^*) - \left(\gamma_t + \frac{\lambda\rho}{2}\right)\mathbb{E}\left[D^2(w_t,W^*)\mid \mathcal{F}_{t-1} \right] + \frac{2^4L}{\gamma_t n}R^*\\
\le& \frac{\lambda \rho t}{4} D^2(w_{t-1}, W^*) - \frac{\lambda\rho(t+2)}{4}\mathbb{E}\left[D^2(w_t, W^*) \mid \mathcal{F}_{t-1} \right] + \frac{2^6L}{\lambda \rho t n} R^* \\
\le& \frac{\lambda\rho t}{4} D^2(w_{t-1}, W^*) - \frac{\lambda\rho(t+2)}{4}\mathbb{E}\left[D^2(w_t, W^*)\mid \mathcal{F}_{t-1}\right] + \frac{2^{7}L}{\lambda\rho (t+1) n} R^*,
\end{aligned}
\]
where in the last inequality we have used $\frac{1}{t}\le \frac{2}{t+1}$ for $t\ge 1$. The above inequality implies
\[
\begin{aligned}
&t \mathbb{E}\left[R(w_t) -  R^*\mid \mathcal{F}_{t-1}\right]\\
\le& (t+1) \mathbb{E}\left[R(w_t) -  R^*\mid \mathcal{F}_{t-1}\right] \\
\le& \frac{\lambda \rho t(t+1)}{4(1-\rho)} D^2(w_{t-1}, W^*) - \frac{\lambda\rho(t+1)(t+2)}{4(1-\rho)}\mathbb{E}\left[D^2(w_t, W^*) \mid \mathcal{F}_{t-1} \right] + \frac{2^{7}L}{\lambda n\rho(1-\rho)}R^*.
\end{aligned}
\]
Then based on the law of total expectation and after proper rearrangement we obtain
\begin{equation}\label{inequat:proof_key_recursion_smooth}
\begin{aligned}
&t \mathbb{E}\left[R(w_t) -  R^*\right]\\
\le& \frac{\lambda \rho t(t+1)}{4(1-\rho)} \mathbb{E}\left[D^2(w_{t-1}, W^*)\right] - \frac{\lambda\rho(t+1)(t+2)}{4(1-\rho)}\mathbb{E}\left[D^2(w_t, W^*) \right] + \frac{2^{7}L}{\lambda n\rho(1-\rho)} R^*.
\end{aligned}
\end{equation}
By summing the above inequality from $t = 1,...,T$ and after normalization we obtain
\[
\begin{aligned}
\frac{2}{T(T+1)}\sum_{t=1}^T t \mathbb{E} \left[R(w_t) -  R^*\right] \le& \frac{\lambda\rho}{T(T+1)(1-\rho)}D^2(w_0, W^*) + \frac{2^{8}L}{\lambda\rho(1-\rho)(T+1)n} R^*\\
\le& \frac{2\lambda\rho}{T(T+1)}D^2(w_0, W^*) + \frac{2^{9}L}{\lambda\rho(T+1)n} R^*,
\end{aligned}
\]
where in the last inequality we have used $\rho \le 0.5$. Consider the weighted output $\bar w_T = \frac{2}{T(T+1)}\sum_{t=1}^T t w_t$. In view of the above inequality and the convexity and quadratic growth property of the risk function $R$ we have
\[
\mathbb{E} \left[R(\bar w_T) -  R^*\right] \le \frac{4\rho\left[R(w_0) - R^*\right]}{T(T+1)} + \frac{2^{9}L}{\lambda \rho n(T+1)} R^*,
\]
which then implies the desired bound in part (a).

\vspace{0.1in}

\textbf{Part (b):} Note that $\gamma_t = \frac{\lambda \rho t}{4} + \frac{16 L}{n} \ge \frac{16 L}{n}$ for all $t\ge 1$. According to Lemma~\ref{lemma:distance_bound} we have the following holds for all $t\ge 1$:
\begin{equation}\label{inequat:proof_key_partb_iteration_bound}
\begin{aligned}
&\mathbb{E}\left[D^2( w_t,  W^*)\right] \\
\le& D^2( w_0,  W^*) + \sum_{\tau = 1}^t\frac{16L}{\gamma_\tau^2 n} R^* \le D^2( w_0,  W^*) + \frac{2^{8}L}{\lambda^2\rho^2 n} R^*\sum_{\tau = 1}^t\frac{1}{\tau^2} \le D^2( w_0,  W^*) + \frac{2^{9}L}{\lambda^2\rho^2 n} R^*.
\end{aligned}
\end{equation}
Similar to the argument in part (a), applying Lemma~\ref{lemma:telescope_smooth} along with the quadratic growth condition $R(w_t) -  R^*\ge \frac{\lambda}{2} D^2(w_t, W^*)$ and $\rho\le0.5$ yields
\[
\begin{aligned}
&\frac{1}{2}\mathbb{E}\left[R(w_t) -  R^*\mid \mathcal{F}_{t-1} \right] \\
\le&(1-\rho)\mathbb{E}\left[R(w_t) -  R^*\mid \mathcal{F}_{t-1} \right] \\
\le& \gamma_t D^2(w_{t-1}, W^*) - \left(\gamma_t + \frac{\lambda\rho}{2}\right)\mathbb{E}\left[D^2(w_t,W^*)\mid \mathcal{F}_{t-1} \right] + \frac{2^4L}{\gamma_t n} R^*\\
\le& \frac{\lambda \rho t}{4} D^2(w_{t-1}, W^*) - \frac{\lambda\rho (t+2)}{4}\mathbb{E}\left[D^2(w_t, W^*) \mid \mathcal{F}_{t-1} \right] \\
&+ \frac{16L}{n}\left(D^2(w_{t-1}, W^*) - \mathbb{E}\left[D^2(w_t, W^*) \mid \mathcal{F}_{t-1} \right]\right) + \frac{2^6 L}{\lambda \rho t n} R^*,
\end{aligned}
\]
where in the second inequality we have used $\gamma_t\ge \frac{52L}{n}$, and in the last inequality we have used $\gamma_t\ge \frac{\lambda\rho t}{4}$. Then based on the law of total expectation and after proper rearrangement we have
\[
\begin{aligned}
&\mathbb{E}\left[R(w_t) -  R^*\right]\\
\le& \frac{\lambda \rho t}{2} \mathbb{E}\left[D^2(w_{t-1}, W^*)\right] - \frac{\lambda\rho (t+2)}{2}\mathbb{E}\left[D^2(w_t, W^*) \right] \\
&+ \frac{2^5L}{n}\left(\mathbb{E}\left[D^2(w_{t-1}, W^*)\right] - \mathbb{E}\left[D^2(w_t, W^*)\right]\right) + \frac{2^{7}L}{\lambda t n\rho} R^*,
\end{aligned}
\]
which implies that
\[
\begin{aligned}
&t\mathbb{E}\left[R(w_t) -  R^*\right]\\
\le& (t+1)\mathbb{E}\left[R(w_t) -  R^*\right]\\
\le& \frac{\lambda \rho t(t+1)}{2} \mathbb{E}\left[D^2(w_{t-1}, W^*)\right] - \frac{\lambda\rho (t+1)(t+2)}{2}\mathbb{E}\left[D^2(w_t, W^*) \right] \\
&+ \frac{2^5 L(t+1)}{n}\left(\mathbb{E}\left[D^2(w_{t-1}, W^*)\right] - \mathbb{E}\left[D^2(w_t, W^*)\right]\right) + \frac{2^{7}L(t+1)}{\lambda t n\rho} R^*\\
\le & \frac{\lambda \rho t(t+1)}{2} \mathbb{E}\left[D^2(w_{t-1}, W^*)\right] - \frac{\lambda\rho (t+1)(t+2)}{2}\mathbb{E}\left[D^2(w_t, W^*) \right] \\
&+ \frac{2^6 Lt}{n}\left(\mathbb{E}\left[D^2(w_{t-1}, W^*)\right] - \mathbb{E}\left[D^2(w_t, W^*)\right]\right) + \frac{2^{8}L}{\lambda n\rho} R^*,
\end{aligned}
\]
where in the last inequality we have used the fact $t+1\le 2t$ as $t\ge 1$. By summing the above inequality from $t = 1,...,T$ and after normalization we obtain
\[
\begin{aligned}
&\frac{2}{T(T+1)}\sum_{t=1}^T t \mathbb{E} \left[R(w_t) -  R^*\right] \\
\le& \frac{2\lambda\rho}{T(T+1)}D^2(w_0, W^*) + \frac{2^7L}{nT(T+1)}\sum_{t=1}^T D^2(w_{t-1}, W^*) + \frac{2^{9}L}{\lambda\rho(T+1)n} R^*\\
\le& \frac{2\lambda\rho}{T(T+1)}D^2(w_0, W^*) + \frac{2^7L}{nT(T+1)}\sum_{t=1}^T \left(D^2( w_0,  W^*) + \frac{2^{9}L}{\lambda^2\rho^2 n} R^*\right) + \frac{2^{9}L}{\lambda\rho(T+1)n} R^*\\
=& \left(\frac{2\lambda\rho}{T(T+1)} + \frac{2^7L}{n(T+1)}\right)D^2( w_0,  W^*) + \left(\frac{2^{16}L^2}{\lambda^2\rho^2n^2(T+1)}+\frac{2^{9}L}{\lambda\rho n(T+1)}\right)R^*,
\end{aligned}
\]
where in the last inequality we have used~\eqref{inequat:proof_key_partb_iteration_bound}. Using the convexity and quadratic growth property in the above inequality yields
\[
\mathbb{E} \left[R(\bar w_T) -  R^*\right] \le \left(\frac{4\rho}{T(T+1)} + \frac{2^{8}L}{\lambda n(T+1)}\right)[R(w_0) - R^*]  + \left(\frac{2^{16}L^2}{\lambda^2\rho^2n^2(T+1)}+\frac{2^{9}L}{\lambda\rho n(T+1)}\right)R^*,
\]
which then implies the desired bound in part (b). The proof is concluded.
\end{proof}

\newpage

\subsection{Proof of Theorem~\ref{thrm:fast_rate_smooth_initialization}}
\label{apdsect:proof_fast_rate_smooth_initialization_corollary}

In this subsection we prove Theorem~\ref{thrm:fast_rate_smooth_initialization} which is restated below.
\FastRateSmoothInitialization*

\vspace{-0.2in}

\begin{proof}
\textbf{Part (a):}  In Phase-I, by invoking the first part of Theorem~\ref{thrm:fast_rate_smooth} with $\rho=1/2$ and $T=n/m \ge 1$ (with slight abuse of notation) we get immediately that
\begin{equation}\label{inequat:proof_key_initialbound_smooth}
\mathbb{E}_{S_1} \left[R(w_1) -  R^*\right] \le \frac{2m^2\left[R(w_0) -  R^*\right]}{n^2} + \frac{2^{10}L}{\lambda n}  R^*.
\end{equation}
In Phase-II, conditioned on $\mathcal{F}_1$, summing the recursion form~\eqref{inequat:proof_key_recursion_smooth} from $t = 2,...,T$ with $\rho=1/2$ and proper normalization yields
\[
\begin{aligned}
&\frac{2}{(T-1)(T+2)}\sum_{t=2}^T t \mathbb{E}_{S_{2:t}} \left[R(w_t) -  R^*\mid \mathcal{F}_1 \right] \\
\le& \frac{6\lambda D^2(w_1, W^*)}{(T-1)(T+2)} + \frac{2^{10}L}{\lambda n(T+2)} R^* \le \frac{3\left(R(w_1) -  R^*\right)}{(T-1)(T+2)} + \frac{2^{10}L}{\lambda n(T+2)} R^*,
\end{aligned}
\]
where in the last inequality we have used the quadratic growth property. Consider the weighted average output $\bar w_T = \frac{2}{(T-1)(T+2)}\sum_{t=2}^T t w_t$. Based on the above inequality and law of total expectation we must have
\[
\begin{aligned}
\mathbb{E} \left[R(\bar w_T) -  R^*\right] \le& \frac{6\mathbb{E}_{S_1} \left[R(w_1) -  R^*\right]}{(T-1)(T+2)} + \frac{2^{10}L}{\lambda n(T+2)} R^*\\
 \le& \frac{6\mathbb{E}_{S_1} \left[R(w_1) -  R^*\right]}{T^2} + \frac{2^{102}L}{\lambda nT} R^*\\
\le& \frac{12m^2\left[R(w_0) -  R^*\right]}{n^2T^2} + \frac{2^{13}L}{\lambda n T}  R^* \\
\le & \frac{2^{22}L^2\left[R(w_0) -  R^*\right]}{\lambda^2 n^2T^2} + \frac{2^{13}L}{\lambda nT}  R^*,
\end{aligned}
\]
where we have used the fact $T\ge 2$ in multiple places and in the last but one step we have used~\eqref{inequat:proof_key_initialbound_smooth}. This immediately implies the desired bound in Part (a).

\vspace{0.1in}

\textbf{Part (b):}  In Phase-I, by applying second part of Theorem~\ref{thrm:fast_rate_smooth} (with $\rho=1/2$ and $T=n/m \ge 1$) and preserving the leading terms we obtain that
\begin{equation}\label{inequat:proof_key_initialbound_smooth_b}
\begin{aligned}
\mathbb{E}_{S_1} \left[R(w_1) -  R^*\right] \lesssim& \left(\frac{m^2}{n^2} + \frac{L}{\lambda n}\right)[R(w_0) - R^*]  + \left(\frac{L^2}{\lambda^2mn}+\frac{L}{\lambda n}\right)R^*\\
\lesssim& \frac{L}{\lambda n}[R(w_0) - R^*]  + \frac{L^2}{\lambda^2n}R^*.
\end{aligned}
\end{equation}
In Phase-II, based on the proof argument of the part (b) of Theorem~\ref{thrm:fast_rate_smooth} we can show that the weighted average output $\bar w_T = \frac{2}{(T-1)(T+2)}\sum_{t=2}^T t w_t$ satisfies
\[
\begin{aligned}
\mathbb{E} \left[R(\bar w_T) -  R^*\right] \lesssim & \left(\frac{1}{T^2} + \frac{L}{\lambda nT}\right) \mathbb{E}_{S_1} \left[R(w_1) -  R^*\right] + \left(\frac{L^2}{\lambda^2n^2T}+\frac{L}{\lambda nT}\right)R^*\\
\lesssim& \left(\frac{L}{\lambda nT^2} + \frac{L^2}{\lambda^2 n^2T}\right)[R(w_0) - R^*] + \left(\frac{L^3}{\lambda^3n^2T}+\frac{L^2}{\lambda^2 nT}\right)R^* \\
\lesssim& \frac{L^2}{\lambda^2nT}[R(w_0) - R^*] + \frac{L^3}{\lambda^3nT}R^*,
\end{aligned}
\]
where in the second step we have used~\eqref{inequat:proof_key_initialbound_smooth_b}. This proves the desired bound in Part (b).
\end{proof}

\subsection{Proof of Theorem~\ref{thrm:fast_rate_smooth_convex}}
\label{apdsect:proof_fast_rate_smooth_convex}

In this subsection, we prove Theorem~\ref{thrm:fast_rate_smooth_convex} as following restated.
\FastRateSmoothConvex*

\begin{proof}
Since $\gamma_t \equiv \gamma \ge  \frac{16 L}{n}$, the bound in Lemma~\ref{lemma:telescope_smooth} is valid. Based on law of total expectation and by summing that inequality from $t = 1,...,T$ with proper normalization we obtain
\[
\begin{aligned}
&\frac{1}{T}\sum_{t=1}^T \mathbb{E} \left[R(w_t) -  R^*\right]\le \frac{\gamma}{T}D^2(w_0, W^*) + \frac{16L}{\gamma n} R^*.
\end{aligned}
\]
Consider $\bar w_T = \frac{1}{T}\sum_{t=1}^T w_t$. In view of the above inequality and convexity of $R$ we have
\[
\mathbb{E} \left[R(\bar w_T) -  R^*\right] \le \frac{\gamma}{T}D^2(w_0, W^*) + \frac{16 L}{\gamma n}  R^*.
\]
This proves the first desired bound. The second bound follows immediately by substituting $\gamma =\sqrt{\frac{T}{n}}+ \frac{16 L}{n} > \frac{16 L}{n}$ into the above bound. The proof is concluded.
\end{proof}

\subsection{On the (Iteration) Stability of M-SPP}
\label{sect:robustness}

In this appendix subsection, we further provide a sensitivity analysis of M-SPP to the choice of regularization modulus $\{\gamma_t\}_{t\ge 1}$, under the following notion of iteration stability essentially introduced by~\citet{asi2019importance,asi2019stochastic}.
\begin{definition}
A stochastic optimization algorithm generating iterates $\{w_t\}_{t\ge 1}$ for minimizing the population risk $R(w)$ is staid to be stable if
\[
\sup_{t\ge 1} D(w_t, W^*) < \infty, \ \ \ \text{with probability 1}.
\]
\end{definition}
Before presenting the main results on the iteration stability of M-SPP, we first recall the Robbins-Siegmund nonnegative almost supermartingale convergence lemma which is typically used for establishing the stability and convergence of stochastic optimization methods including SPP~\citep{asi2019stochastic}.
\begin{lemma}[\cite{robbins1971convergence}]\label{lemma:robbins-siegmund}
Consider four sequences of nonnegative random variables $\{U_t\},\{V_t\}, \{\alpha_t\}, \{\beta_t\}$ that are measurable over a filtration $\{\mathcal{F}_t\}_{t\ge 0}$. Suppose that $\sum_t \alpha_t<\infty$, $\sum_t \beta_t < \infty$, and
\[
\mathbb{E}[U_{t+1}\mid \mathcal{F}_t] \le (1+\alpha_t) U_t + \beta_t - V_t.
\]
Then there exits $U_{\infty}$ such that $U_t \xrightarrow{a.s.} U_{\infty}$ and $\sum_t V_t < \infty$ with probability 1.
\end{lemma}

The following proposition shows that the sequence of estimation error $\{\|w_t - w^*\|\}$ is non-divergent in expectation and it converges to some finite value and is bounded with probability $1$.
\begin{proposition}\label{prop:stability}
Suppose that the Assumptions~\ref{assump:smooth} holds. Assume that $\gamma_t \ge \frac{16 L}{n}$ and $\sum_{t\ge 1} L \gamma_t^{-2}<\infty$. Then we have the following hold:
\begin{itemize}[leftmargin=0.68cm]
  \item[(a)] $\mathbb{E}\left[D( w_t,  W^*)\right] <\infty$;
  \item[(b)] $D( w_t,  W^*)$ converges to some finite value and $\sup_{t\ge 1} D( w_t,  W^*) < \infty$ with probability $1$.
\end{itemize}
\end{proposition}
\begin{proof}
Applying Lemma~\ref{lemma:distance_bound} yields that for all $t\ge 1$
\[
\mathbb{E}\left[D^2( w_t,  W^*)\right] \lesssim D^2( w_0,  W^*) + \sum_{\tau = 1}^t\frac{L}{\gamma_\tau^2 n} R^* < \infty,
\]
where we have used the given conditions on $\gamma_t$. This proves the part (a). To show the part (b), invoking Lemma~\ref{lemma:robbins-siegmund} with $\alpha_t=V_t\equiv0$ and $\beta_t=\frac{16 L}{\gamma^2_t n} R^*$ to~\eqref{equat:proof_iterstability_key} yields $D( w_t,  W^*)$ converges to some finite value and thus $\sup_{t\ge 1} D( w_t,  W^*) < \infty$ almost surely.
\end{proof}
\begin{remark}
Proposition~\ref{prop:stability} shows that in contrast to minibatch SGD, the choice of $\gamma_t$ in M-SPP is insensitive to the gradient scale of loss functions for generating a non-divergent sequence of estimation errors.
\end{remark}

\section{Proofs for the Results in Section~\ref{sect:analysis_smooth_inexact}}

In this section, we present the technical proofs for the main results stated in Section~\ref{sect:analysis_smooth_inexact}.

\subsection{Proof of Theorem~\ref{thrm:fast_rate_smooth_inexact}}
\label{apdsect:proof_fast_rate_smooth_inexact}

In this subsection, we prove Theorem~\ref{thrm:fast_rate_smooth_inexact} which is restated below.
\FastRateSmoothInexact*

\emph{Preliminaries.} In what follows, we denote by $\tilde w_t:=\argmin_{w \in \mathcal{W}} F_t(w)$ the exact solution of the inner-loop minibatch ERM optimization, which plays the same role as $w_t$ in Section~\ref{sect:analysis_smooth}. We first present the following lemma that upper bounds the discrepancy between the inexact minimizer $w_t$ and the exact minimizer $\tilde w_t$.
\begin{lemma}\label{lemma:function_strong_convexity_inexact}
Assume that the loss function $\ell$ is convex with respect to its first argument and $r$ is convex. Then for any $w\in \mathcal{W}$, we have
\[
\| w_t -  \tilde w_t\| \le \sqrt{\frac{2\epsilon_t}{\gamma_t}}.
\]
\end{lemma}
\begin{proof}
Using arguments identical to those of Lemma~\ref{lemma:function_strong_convexity} we can show that for all $w \in \mathcal{W}$,
\begin{equation}\label{inequat:function_strong_convexity_inexact_proof_key1}
R_{S_t}(\tilde w_t) -  R_{S_t}(w) \le  \frac{\gamma_t}{2}\left( \|w -  w_{t-1} \|^2 - \|w -  \tilde w_t\|^2 - \| \tilde w_t -  w_{t-1}\|^2\right).
\end{equation}
Setting $w = w_t$ in the above yields
\[
\frac{\gamma_t}{2} \| w_t -  \tilde w_t\|^2 \le F_t( w_t) -   F_t(\tilde w_t) \le \epsilon_t,
\]
which directly implies $\| w_t -  \tilde w_t\| \le \sqrt{2\epsilon_t/\gamma_t}$. This proves the second desired bound.
\end{proof}

The following lemma as an extension of Lemma~\ref{lemma:telescope_smooth} to the inexact setting.

\begin{lemma}\label{lemma:telescope_smooth_inexact}
Suppose that the Assumptions~\ref{assump:smooth},~\ref{assump:quadratic_growth} and~\ref{assump:lipschitz} hold. Assume that $\gamma_t \ge \frac{19 L}{n}$. Then the following bound holds for any $\rho\in(0,1)$:
\[
\begin{aligned}
\mathbb{E}\left[R(w_t) -  R^*\mid \mathcal{F}_{t-1} \right] \le& \gamma_t\left(D^2(w_{t-1},W^*) - \mathbb{E}\left[\left(1- \frac{\rho \lambda}{2\gamma_t}\right)D^2(w_t,W^*)\mid \mathcal{F}_{t-1} \right]\right) \\
& + \frac{19 L}{\gamma_t n} R^* +\left(3 n + \frac{4\gamma_t}{\rho\lambda} \right)\epsilon_t + 3G\sqrt{\frac{2\epsilon_t}{\gamma_t}}.
\end{aligned}
\]
Alternatively, for any $w^*\in W^*$, under Assumptions~\ref{assump:smooth} and~\ref{assump:lipschitz} we have
\[
\begin{aligned}
\mathbb{E}\left[R(w_t) -  R^*\mid \mathcal{F}_{t-1}\right] \le& \gamma_t\left(\|w_{t-1} - w^*\|^2 - \mathbb{E}\left[\|w_t - w^*\|^2 \mid \mathcal{F}_{t-1} \right]\right) + \frac{19L}{\gamma_t n} R^* \\
& + 3n\epsilon_t +  \left( 2\sqrt{2\gamma_t}\mathbb{E}\left[\|w_t - w^*\| \mid \mathcal{F}_{t-1} \right] + 3G\sqrt{\frac{2}{\gamma_t}}\right)\sqrt{\epsilon_t}.
\end{aligned}
\]
\end{lemma}
\begin{proof}
Let us decompose $\mathbb{E}\left[R(w_t) -  R^*\mid \mathcal{F}_{t-1}\right]$ into the following three terms:
\[
\begin{aligned}
&\mathbb{E}\left[R(w_t) -  R^*\mid \mathcal{F}_{t-1}\right] \\
=& \underbrace{\mathbb{E}\left[R(w_t) -  R(\tilde w_t)\mid \mathcal{F}_{t-1} \right]}_A + \underbrace{\mathbb{E}\left[R(\tilde w_t) -  R_{S_t}(\tilde w_t)\mid \mathcal{F}_{t-1} \right]}_B + \underbrace{\mathbb{E}\left[R_{S_t}(\tilde w_t) -  R^*\mid \mathcal{F}_{t-1} \right]}_C.
\end{aligned}
\]
We next bound these three terms respectively. To bound the term $A$, we can show that
\[
\begin{aligned}
&|A|:=|\mathbb{E}\left[R(w_t) -  R(\tilde w_t) \mid \mathcal{F}_{t-1} \right]| \\
=& \left|\mathbb{E}\left[R^\ell(w_t) -  R^\ell(\tilde w_t)\mid \mathcal{F}_{t-1}\right] + \mathbb{E}\left[r(w_t) -  r(\tilde w_t)\right] \mid \mathcal{F}_{t-1} \right| \\
\le& \mathbb{E}\left[\mathbb{E}_z |\ell(w_t;z) - \ell(\tilde w_t;z)| \mid \mathcal{F}_{t-1} \right] + \mathbb{E}\left[|r(w_t) -  r(\tilde w_t)| \mid \mathcal{F}_{t-1} \right]\\
\overset{\zeta_1}{\le}& \mathbb{E}\left[\mathbb{E}_z \left[\sqrt{2L \ell(w_t;z)} \|w_t - \tilde w_t\| \right]\mid \mathcal{F}_{t-1}\right] +  \mathbb{E}\left[G\|w_t -  \tilde w_t\|\mid \mathcal{F}_{t-1}\right] \\
\le& \mathbb{E}\left[\mathbb{E}_z \left[\frac{L}{\gamma_t n} \ell(w_t;z) + \frac{\gamma_t n}{2}\|w_t - \tilde w_t\|^2\right]\mid \mathcal{F}_{t-1}\right] +  \mathbb{E}\left[G\|w_t -  \tilde w_t\|\mid \mathcal{F}_{t-1}\right] \\
=& \mathbb{E}\left[\frac{L}{\gamma_t n} R^\ell(w_t) \mid \mathcal{F}_{t-1}\right] + \mathbb{E}_{S_t}\left[\frac{\gamma_t n}{2}\|w_t - \tilde w_t\|^2 + G\|w_t -  \tilde w_t\|\mid \mathcal{F}_{t-1}\right] \\
\le& \mathbb{E}\left[\frac{L}{\gamma_t n} R(w_t) \mid \mathcal{F}_{t-1} \right] + n\epsilon_t + G\sqrt{\frac{2\epsilon_t}{\gamma_t}},
\end{aligned}
\]
where in ``$\zeta_1$'' we have used the convexity of loss and Lemma~\ref{lemma:key_smooth} and the Assumption~\ref{assump:lipschitz} and in the last inequality we have used $r>0$ and the perturbation bound of Lemma~\ref{lemma:function_strong_convexity_inexact}.

To bound the term $B$, using about the same proof arguments as for Lemma~\ref{lemma:telescope_smooth} we can show that
\[
\begin{aligned}
B:=&\mathbb{E}\left[R(\tilde w_t) -  R_{S_t}(\tilde w_t) \mid \mathcal{F}_{t-1} \right] \\
\le& \frac{8L}{\gamma_t n} \mathbb{E}\left[R(\tilde w_t)\mid \mathcal{F}_{t-1}\right] \\
=& \frac{8L}{\gamma_t n} \mathbb{E}\left[R(\tilde w_t) - R(w_t)\right] + \frac{8L}{\gamma_t n} \mathbb{E}\left[R(w_t) \mid \mathcal{F}_{t-1} \right]\\
\le& \frac{1}{2}|A| + \frac{8L}{\gamma_t n}\mathbb{E}\left[R(w_t) \mid \mathcal{F}_{t-1} \right],
\end{aligned}
\]
where we have used the condition on minibatch size $\gamma_t$.

To bound the term $C$, based on the definition of $\tilde w_t$ and by invoking Lemma~\ref{lemma:function_strong_convexity} with $w=w^*_{t-1}$ we can verify that
\[
\begin{aligned}
C:=& \mathbb{E}\left[R_{S_t}(\tilde w_t) -  R^* \mid \mathcal{F}_{t-1} \right] \\
 \le& \frac{\gamma_t}{2} \mathbb{E}\left[ \|w^*_{t-1} - w_{t-1} \|^2 - \|w^*_{t-1} - \tilde w_t\|^2 - \|\tilde w_t - w_{t-1}\|^2\mid \mathcal{F}_{t-1}\right] \\
\le& \frac{\gamma_t}{2} \mathbb{E}\left[ \|w^*_{t-1} - w_{t-1} \|^2 - \|w^*_{t-1} - \tilde w_t\|^2 \mid \mathcal{F}_{t-1}\right] \\
=&  \frac{\gamma_t}{2}\left( \|w^*_{t-1} - w_{t-1} \|^2 - \mathbb{E}\left[\|w^*_{t-1} - w_t + w_t - \tilde w_t\|^2\mid \mathcal{F}_{t-1}\right]\right) \\
=& \frac{\gamma_t}{2}\left( \|w^*_{t-1} - w_{t-1} \|^2 - \mathbb{E}\left[\|w^*_{t-1} - w_t \|^2 + 2\langle w^*_{t-1} - w_t, w_t - \tilde w_t\rangle+ \|w_t - \tilde w_t\|^2\mid \mathcal{F}_{t-1}\right]\right) \\
\le& \frac{\gamma_t}{2}\left( \|w^*_{t-1} - w_{t-1} \|^2 - \mathbb{E}\left[\left(1- \frac{\rho \lambda}{2\gamma_t}\right)\| w^*_{t-1} - w_t\|^2 -  \frac{2\gamma_t}{\rho\lambda}\| w_t -\tilde w_t\|^2 \mid \mathcal{F}_{t-1}\right]\right) \\
\le& \frac{\gamma_t}{2}\left( \|w^*_{t-1} - w_{t-1} \|^2 - \mathbb{E}\left[\left(1- \frac{\rho \lambda}{2\gamma_t}\right)\| w^*_t - w_t\|^2 \mid \mathcal{F}_{t-1} \right]\right) + \frac{2\gamma_t \epsilon_t}{\rho\lambda} \\
=& \frac{\gamma_t}{2}\left( D^2(w_{t-1}, W^*) - \mathbb{E}\left[\left(1- \frac{\rho \lambda}{2\gamma_t}\right)D^2(w_t, W^*) \mid \mathcal{F}_{t-1} \right]\right) + \frac{2\gamma_t \epsilon_t}{\rho\lambda} .
\end{aligned}
\]
Combining the above three bounds yields
\[
\begin{aligned}
&\mathbb{E}\left[R(w_t) -  R^* \mid \mathcal{F}_{t-1}\right] = A +B + C \\
\le& \frac{3}{2}|A| + \frac{8L}{\gamma_t n} \mathbb{E}\left[R(w_t)\mid \mathcal{F}_{t-1}\right] \\
&+ \frac{\gamma_t}{2}\left( D^2(w_{t-1}, W^*) - \mathbb{E}\left[\left(1- \frac{\rho \lambda}{2\gamma_t}\right)D^2(w_{t}, W^*) \mid \mathcal{F}_{t-1} \right]\right) + \frac{2\gamma_t \epsilon_t}{\rho\lambda} \\
\le& \mathbb{E}\left[\frac{3 L}{2\gamma_t n} R(w_t) \mid \mathcal{F}_{t-1} \right] + \frac{3 n}{2}\epsilon_t + \frac{3G}{2}\sqrt{\frac{2\epsilon_t}{\gamma_t}} + \frac{8L}{\gamma_t n}\mathbb{E}\left[R(w_t)\mid \mathcal{F}_{t-1}\right] \\
&+ \frac{\gamma_t}{2}\left( D^2(w_{t-1}, W^*) - \mathbb{E}\left[\left(1- \frac{\rho \lambda}{2\gamma_t}\right)D^2(w_{t}, W^*) \mid \mathcal{F}_{t-1} \right]\right) + \frac{2\gamma_t \epsilon_t}{\rho\lambda} \\
\le& \frac{\gamma_t}{2}\left( D^2(w_{t-1}, W^*) - \mathbb{E}\left[\left(1- \frac{\rho \lambda}{2\gamma_t}\right)D^2(w_{t}, W^*) \mid \mathcal{F}_{t-1} \right]\right) + \frac{9.5 L}{\gamma_t n}\mathbb{E}\left[R(w_t) \mid \mathcal{F}_{t-1} \right]\\
 &+ \left(\frac{3 n}{2} + \frac{2\gamma_t}{\rho\lambda}\right)\epsilon_t + \frac{3G}{2}\sqrt{\frac{2\epsilon_t}{\gamma_t}} \\
=& \frac{\gamma_t}{2}\left( D^2(w_{t-1}, W^*) - \mathbb{E}\left[\left(1- \frac{\rho \lambda}{2\gamma_t}\right)D^2(w_{t}, W^*) \mid \mathcal{F}_{t-1} \right]\right) + \frac{9.5 L}{\gamma_t n}\mathbb{E}\left[R^*\right] +  \frac{9.5 L}{\gamma_t n}\mathbb{E}\left[R(w_t) - R^* \mid \mathcal{F}_{t-1} \right] \\
& + \left(\frac{3 n}{2} + \frac{2\gamma_t}{\rho\lambda}\right)\epsilon_t + \frac{3G}{2}\sqrt{\frac{2\epsilon_t}{\gamma_t}}\\
\le& \frac{\gamma_t}{2}\left( D^2(w_{t-1}, W^*) - \mathbb{E}\left[\left(1- \frac{\rho \lambda}{2\gamma_t}\right)D^2(w_{t}, W^*) \mid \mathcal{F}_{t-1} \right]\right) + \frac{9.5 L}{\gamma_t n}\mathbb{E}\left[R^*\right] + \frac{1}{2}\mathbb{E}\left[R(w_t) - R^* \mid \mathcal{F}_{t-1} \right] \\
& + \left(\frac{3 n}{2} + \frac{2\gamma_t}{\rho\lambda}\right)\epsilon_t + \frac{3G}{2}\sqrt{\frac{2\epsilon_t}{\gamma_t}},
\end{aligned}
\]
where in the last inequality we have used the condition $\gamma_t \ge \frac{19 L}{n}$. After rearranging the terms in the above inequality we obtain the first desired bound.

To derive the second bound, for any fixed $w^*\in W^*$, we note that the term $C$ can be alternatively bounded as
\[
\begin{aligned}
C \le& \frac{\gamma_t}{2}\left( \|w^* - w_{t-1} \|^2 - \mathbb{E}\left[\|w^* - w_t \|^2 + 2\langle w^* - w_t, w_t - \tilde w_t\rangle+ \|w_t - \tilde w_t\|^2 \mid \mathcal{F}_{t-1} \right]\right) \\
\le& \frac{\gamma_t}{2}\left( \|w^* - w_{t-1} \|^2 - \mathbb{E}\left[\| w^* - w_t\|^2 - 2\|w_t - w^*\|\| w_t - \tilde w_t\| \mid \mathcal{F}_{t-1} \right]\right) \\
\le& \frac{\gamma_t}{2}\left( \|w^* - w_{t-1} \|^2 - \mathbb{E}\left[\| w^* - w_t\|^2 \mid \mathcal{F}_{t-1} \right]\right) + \sqrt{2\gamma_t\epsilon_t}\mathbb{E}\left[\| w^* - w_t\| \mid \mathcal{F}_{t-1} \right].
\end{aligned}
\]
Similar to the proof of the first bound, we can derive that
\[
\begin{aligned}
&\mathbb{E}_{S_t}\left[R(w_t) -  R^* \mid \mathcal{F}_{t-1} \right] = A +B + C \\
\le& \frac{3}{2}|A| + \frac{8 L}{\gamma_t n} \mathbb{E}\left[R(w_t) \mid \mathcal{F}_{t-1} \right] + \frac{\gamma_t}{2}\left( \|w^* - w_{t-1} \|^2 - \mathbb{E}\left[\| w^* - w_t\|^2 \mid \mathcal{F}_{t-1} \right]\right) + \sqrt{2\gamma_t\epsilon_t}\mathbb{E}\left[\| w^* - w_t\| \mid \mathcal{F}_{t-1} \right] \\
\le& \frac{\gamma_t}{2}\left( \|w^* - w_{t-1} \|^2 - \mathbb{E}\left[\| w^* - w_t\|^2 \mid \mathcal{F}_{t-1} \right]\right) + \frac{9.5 L}{\gamma_t n} R^* \\
& + \frac{1}{2}\mathbb{E}\left[R(w_t) - R^* \mid \mathcal{F}_{t-1} \right] + \frac{3 n}{2}\epsilon_t + \sqrt{2\gamma_t\epsilon_t}\mathbb{E}\left[\| w^* - w_t\| \mid \mathcal{F}_{t-1}\right] + \frac{3G}{2}\sqrt{\frac{2\epsilon_t}{\gamma_t}}.
\end{aligned}
\]
After rearranging the terms in the above inequality we obtain the second desired bound.
\end{proof}

With the above preliminary results in hand, we are now in the position to prove the main result of Theorem~\ref{thrm:fast_rate_smooth_inexact}.

\vspace{0.1in}

\begin{proof}[Proof of Theorem~\ref{thrm:fast_rate_smooth_inexact}]
Since by assumption $R(w_t) -  R^*\ge \frac{\lambda}{2} D^2(w_t, W^*)$ and $\gamma_t = \frac{\lambda \rho t}{4} \ge \frac{\lambda\rho}{4}\ge \frac{19 L}{n}$, based on the first bound in Lemma~\ref{lemma:telescope_smooth_inexact} we can show that
\[
\begin{aligned}
&(1 - 2\rho) \mathbb{E}\left[R(w_t) -  R^* \mid \mathcal{F}_{t-1} \right] \\
\le& \gamma_t D^2(w_{t-1}, W^*) - \left(\gamma_t + \frac{\rho\lambda}{2}\right)\mathbb{E}\left[D^2(w_t, W^*)\mid \mathcal{F}_{t-1}\right] + \frac{19 L}{\gamma_t n} R^* +\left(3n + \frac{4\gamma_t}{\rho\lambda} \right)\epsilon_t + 3G\sqrt{\frac{2\epsilon_t}{\gamma_t}} \\
\le& \frac{\lambda\rho t}{4} D^2(w_{t-1}, W^*) - \frac{\rho\lambda(t+2)}{4} \mathbb{E}\left[D^2(w_t, W^*) \mid \mathcal{F}_{t-1} \right] + \frac{76L}{\lambda\rho n t} R^* +\left(3n  + t \right)\epsilon_t + 6G\sqrt{\frac{2\epsilon_t}{\lambda\rho t}}.
\end{aligned}
\]
Now suppose that $\epsilon_t\le\frac{\epsilon}{nt^4}$ for some $\epsilon\in[0,1]$. Since $\rho \le 1/4$, the above implies
\[
\begin{aligned}
&\mathbb{E} \left[R(w_t) -  R^* \mid \mathcal{F}_{t-1} \right] \\
\le& \frac{\lambda\rho t}{2} D^2(w_{t-1}, W^*) - \frac{\rho\lambda(t+2)}{2} \mathbb{E}\left[D^2(w_t, W^*) \mid \mathcal{F}_{t-1} \right] + \frac{152 L}{\lambda\rho n t} R^* + \left( \frac{6}{t^4}  + \frac{2}{t^3}  + 12G\sqrt{\frac{2}{\lambda\rho t^5}}\right)\sqrt{\epsilon}.
\end{aligned}
\]
The above inequality then implies
\[
%\begin{equation}\label{inequat:proof_key_recursion_inexact}
\begin{aligned}
 t \mathbb{E}\left[R(w_t) -  R^* \mid \mathcal{F}_{t-1} \right] \le& (t+1) \mathbb{E}\left[R(w_t) -  R^* \mid \mathcal{F}_{t-1} \right] \\
\le& \frac{\lambda \rho t(t+1)}{2} D^2(w_{t-1}, W^*) - \frac{\lambda\rho(t+1)(t+2)}{2}\mathbb{E}\left[D^2(w_t, W^*) \mid \mathcal{F}_{t-1} \right] \\
& + \frac{304L}{\lambda\rho n} R^* + \left(\frac{12}{t^3}  + \frac{4}{t^2}  + \frac{24G}{t}\sqrt{\frac{2}{\lambda\rho t}}\right)\sqrt{\epsilon},
\end{aligned}
%\end{equation}
\]
where we have used the fact $\frac{t+1}{t}\le 2$ for $t\ge1$. In view of the law of total expectation, summing the above inequality from $t = 1,...,T$ with natural normalization yields
\[
\begin{aligned}
& \frac{2}{T(T+1)}\sum_{t=1}^T t \mathbb{E} \left[R(w_t) -  R^* \right]  \\
\le& \frac{2\lambda \rho}{T(T+1)}D^2(w_0, W^*) + \frac{608 L}{\lambda\rho(T+1)n} R^* + \frac{\sqrt{\epsilon}}{T(T+1)}\left(64  + 192G\sqrt{\frac{2}{\lambda\rho}}\right) \\
\le& \frac{4\rho}{T(T+1)}(R(w_0) - R^*) + \frac{608L}{\lambda\rho(T+1)n} R^* + \frac{\sqrt{\epsilon}}{T(T+1)}\left( 64  + 192G\sqrt{\frac{2}{\lambda\rho}}\right),
\end{aligned}
\]
which then immediately leads to the desired bound. The proof is concluded.
\end{proof}

\subsection{Proof of Theorem~\ref{thrm:fast_rate_smooth_convex_inexact}}
\label{apdsect:proof_fast_rate_smooth_convex_inexact}

In this subsection, we prove Theorem~\ref{thrm:fast_rate_smooth_convex_inexact} as following restated.
\FastRateSmoothConvexInexact*

The following lemma, which can be proved by induction~\citep[see, e.g., ][]{schmidt2011convergence}, will be used to prove the main result.
\begin{lemma}\label{lemma:key_sequence_bound}
Assume that the nonnegative sequence $\{u_\tau\}_{\tau\ge1}$ satisfies the following recursion for all $t\ge 1$:
\[
u^2_t \le S_t + \sum_{\tau=1}^t \alpha_\tau u_\tau,
\]
with $\{S_\tau\}_{\tau\ge1}$ an increasing sequence, $S_0\ge u_0^2$ and $\alpha_\tau\ge 0$ for all $\tau$. Then, the following bound holds for all $t\ge 1$:
\[
u_t \le \sqrt{S_t} + \sum_{\tau=1}^t \alpha_\tau.
\]
\end{lemma}

\newpage

The following lemma gives an upper bound on the expected estimation error $\mathbb{E}\left[\|w^*_0 - w_t\|\right]$.
\begin{lemma}\label{lemma:key_estimation_error_bound}
Under the conditions of Theorem~\ref{thrm:fast_rate_smooth_convex_inexact}, the following bound holds for all $t\ge1$:
\[
\mathbb{E}\left[\|w_t - w^*_0\|\right] \le \|w_0 - w^*_0\| + \sqrt{\frac{t}{\gamma}R^*} + \frac{6tG}{\gamma}.
\]
\end{lemma}
\begin{proof}
Recall that $w^*_0=\argmin_{w\in W^*}\|w_0 - w\|$. Since $\gamma_t \equiv \gamma \ge \frac{19 L}{n}$, the second bound in Lemma~\ref{lemma:telescope_smooth_inexact} is valid. For any $t\in [T]$, by summing that inequality with $w^*=w^*_0$ from $\tau = 1,...,t$ we obtain
\begin{equation}\label{equat:proof_convex_inexact_key}
\begin{aligned}
&\sum_{\tau=1}^t \mathbb{E} \left[R(w_\tau) -  R^*\right] + \gamma \mathbb{E} \left[\|w_t - w^*_0\|^2\right] \\
\le& \gamma \|w_0 - w^*_0\|^2 + \frac{19 L}{\gamma n} tR^* + 3n \sum_{\tau=1}^t \epsilon_\tau + \sum_{\tau=1}^t \left( 2\sqrt{2\gamma}\mathbb{E} \left[\| w^*_0 - w_\tau\|\right] + 3G\sqrt{\frac{2}{\gamma}}\right)\sqrt{\epsilon_\tau}.
\end{aligned}
\end{equation}
Dropping the non-negative term $\sum_{\tau=1}^t \mathbb{E}_{S_{[\tau]}} \left[R(w_\tau) -  R^*\right]$ from the above inequality yields
\[
\begin{aligned}
&\underbrace{\mathbb{E} \left[\|w_t - w^*_0\|^2\right]}_{u_t^2} \\
\le& \|w_0 - w^*_0\|^2 + \frac{19 L}{\gamma^2 n} tR^* + \frac{3n}{\gamma} \sum_{\tau=1}^t \epsilon_\tau + \sum_{\tau=1}^t \left( 2\sqrt{\frac{2}{\gamma}}\mathbb{E}\left[\| w^*_0 - w_\tau\|\right] + 3G\frac{\sqrt{2}}{\gamma\sqrt{\gamma}}\right)\sqrt{\epsilon_\tau} \\
\overset{\zeta_1}{\le}& \|w_0 - w^*_0\|^2 + \frac{t}{\gamma}R^* + \sum_{\tau=1}^t \left(\frac{3n}{\gamma}\epsilon_{\tau} +\frac{3G\sqrt{2}}{\gamma\sqrt{\gamma}}\sqrt{\epsilon_{\tau}}\right) + \sum_{\tau=1}^t \left(2\sqrt{\frac{2\epsilon_\tau}{\gamma}}\sqrt{\mathbb{E}\left[\| w^*_0 - w_\tau\|^2\right]}\right) \\
\le& \underbrace{\|w_0 - w^*_0\|^2 + \frac{t}{\gamma} R^* + \sum_{\tau=1}^t \frac{4G\sqrt{2\epsilon_\tau}}{\gamma\sqrt{\gamma}}}_{S_t} + \sum_{\tau=1}^t \left(\underbrace{2\sqrt{\frac{2\epsilon_\tau}{\gamma}}}_{\alpha_\tau}\underbrace{\sqrt{\mathbb{E} \left[\| w^*_0 - w_{\tau}\|^2\right]}}_{u_\tau}\right),
\end{aligned}
\]
where in ``$\zeta_1$'' we have used $\gamma \ge \frac{19 L}{n}$ and the basic inequality $\mathbb{E}^2[X]\le \mathbb{E}[X^2]$, and in the last inequality we have used the condition $\epsilon_\tau\le \frac{2G^2}{9n^2\gamma}$ for all $\tau\ge 1$. By invoking Lemma~\ref{lemma:key_sequence_bound} to the above recursion form we can derive that for all $t\ge 1$,
\[
\begin{aligned}
\sqrt{\mathbb{E} \left[\|w_t - w^*_0\|^2\right]} \le& \sqrt{\|w_0 - w^*_0\|^2 + \frac{t}{\gamma} R^* + \sum_{\tau=1}^t \frac{4G\sqrt{2\epsilon_\tau}}{\gamma}} +\sum_{\tau=1}^t 2\sqrt{\frac{2\epsilon_\tau}{\gamma}} \\
\le& \|w_0 - w^*_0\| + \sqrt{\frac{t}{\gamma}R^*} + \sum_{\tau=1}^t \sqrt{\frac{4G\sqrt{2\epsilon_\tau}}{\gamma\sqrt{\gamma}}} +\sum_{\tau=1}^t 2\sqrt{\frac{2\epsilon_\tau}{\gamma}} \\
%\le& \|w_0 - w^*\| + \sqrt{\left(\frac{52L}{\gamma^2 n} + \frac{128\kappa^4L^2}{\gamma^3n^2}\right)TR^*} + \sum_{t=1}^T 3\sqrt{\frac{2\epsilon_t}{\gamma}} \\
\le& \|w_0 - w^*_0\| + \sqrt{\frac{t}{\gamma}R^*} + \frac{6Gt}{\gamma},
\end{aligned}
\]
where the last inequality is due to the condition $\epsilon_\tau\le \frac{2G^2}{9\gamma}$ for all $\tau\ge 1$. The above inequality then directly implies the desired bound for all $t\in [T]$.
\end{proof}

Now we are ready to prove the main result of Theorem~\ref{thrm:fast_rate_smooth_convex_inexact}.
\begin{proof}[Proof of Theorem~\ref{thrm:fast_rate_smooth_convex_inexact}]
Dropping non-negative term $\gamma \mathbb{E} \left[\|w_t - w^*\|^2\right]$ in~\eqref{equat:proof_convex_inexact_key} followed by natural normalization yields
\[
\begin{aligned}
&\frac{1}{T}\sum_{t=1}^T \mathbb{E} \left[R(w_t) -  R^*\right] \\
\le& \frac{\gamma}{T} \|w_0 - w^*_0\|^2 + \frac{19L}{\gamma n} R^* + \frac{3n}{ T} \sum_{t=1}^T \epsilon_t + \frac{1}{T}\sum_{t=1}^T \left( 2 \sqrt{2\gamma}\mathbb{E}\left[\| w_t - w^*_0 \|\right] + 3G\sqrt{\frac{2}{\gamma}}\right)\sqrt{\epsilon_t}\\
\overset{\zeta_1}{\le} & \frac{\gamma}{T} \|w_0 - w^*_0\|^2 + \frac{19 L}{\gamma n} R^* + \frac{3n}{T} \sum_{t=1}^T \epsilon_t \\
& + \frac{1}{T}\sum_{t=1}^T \left( 2\sqrt{2}\left(\sqrt{\gamma}\|w_0 - w^*_0\| + \sqrt{tR^*} + \frac{6Gt}{\sqrt{\gamma}}\right) + 3G\sqrt{\frac{2}{\gamma}}\right)\sqrt{\epsilon_t}\\
\overset{\zeta_2}{\le}& \frac{\gamma}{T} \|w_0 - w^*_0\|^2 + \frac{19 L}{\gamma n} R^* + \frac{1}{T}\sum_{t=1}^T \left(3n\epsilon_t + 2\sqrt{2\gamma\epsilon_t}\|w_0 - w^*_0\| + 2\sqrt{2tR^*\epsilon_t} + \frac{15\sqrt{2\epsilon_t}Gt}{\sqrt{\gamma}} \right)\\
\overset{\zeta_3}{\le}& \frac{\gamma}{T} \|w_0 - w_0^*\|^2 + \frac{19 L}{\gamma n} R^* \\
&+ \frac{1}{T}\sum_{t=1}^T \left(3n\epsilon_t + \frac{\gamma\|w_0-w^*_0\|^2}{t^2} + 2t^2\epsilon_t + \frac{2LR^*}{\gamma n} + \frac{\gamma n t\epsilon_t}{L} +  \frac{15\sqrt{2\epsilon_t}Gt}{\sqrt{\gamma}} \right)\\
\le& \frac{3\gamma}{T} \|w_0 - w_0^*\|^2 + \frac{21 L}{\gamma n} R^* + \frac{1}{T}\sum_{t=1}^T \left(3n\epsilon_t + 2t^2\epsilon_t +  \frac{\gamma n t\epsilon_t}{L} +  \frac{15\sqrt{2\epsilon_t}Gt}{\sqrt{\gamma}} \right),
\end{aligned},
\]
where ``$\zeta_1$''follows from Lemma~\ref{lemma:key_estimation_error_bound}, ``$\zeta_2$'' is due to $t\ge 1$ and ``$\zeta_3$'' is due to $ab\le (a^2+b^2)/2$. Now consider $\epsilon_t\le \frac{\epsilon}{n^2 t^5}$ for some $\epsilon\in[0,1]$. Then it follows from the preceding inequality that
\[
\begin{aligned}
&\frac{1}{T}\sum_{t=1}^T \mathbb{E} \left[R(w_t) -  R^*\right] \\
\le& \frac{3\gamma}{T} \|w_0 - w_0^*\|^2 + \frac{21 L}{\gamma n} R^* + \frac{1}{T}\sum_{t=1}^T \left(\frac{3}{n t^5} + \frac{2}{nt^3} +  \frac{\gamma}{nLt^4} +  \frac{15 \sqrt{2}G}{nt^{1.5}\sqrt{\gamma}} \right)\sqrt{\epsilon} \\
\le& \frac{3\gamma}{T} \|w_0 - w_0^*\|^2 + \frac{21 L}{\gamma n} R^* + \frac{1}{T}\left(\frac{6}{n} + \frac{4}{n} +  \frac{2\gamma}{nL} +  \frac{45 \sqrt{2}G}{n\sqrt{\gamma}} \right)\sqrt{\epsilon}.
\end{aligned}
\]
Let $\bar w_T = \frac{1}{T}\sum_{t=1}^T w_t$. Combined with the convexity of $R$, the above inequality implies
\[
\mathbb{E} \left[R(\bar w_T) -  R^*\right] \lesssim \frac{\gamma}{T}D^2(w_0, W^*) + \frac{L}{\gamma n} R^* + \left( \frac{1}{nT} +  \frac{\gamma}{LnT} +  \frac{G}{nT\sqrt{\gamma}} \right)\sqrt{\epsilon}.
\]
This proves the first bound. Substituting $\gamma =\sqrt{\frac{T}{n}}+ \frac{19L}{n} >  \frac{19 L}{n}$ into the above bound and preserving the leading terms yields the following second desired bound:
\[
\mathbb{E} \left[R(\bar w_T) -  R^*\right] \lesssim \left(\frac{1}{\sqrt{nT}}+ \frac{L}{nT}\right)D^2(w_0, W^*) + \frac{L}{\sqrt{nT}} R^*+ \left(\frac{L+G}{\sqrt{nT}} + \frac{1}{nT} \right)\sqrt{\epsilon}.
\]
The proof is concluded. 
\end{proof}

\section{Proofs for the Results in Section~\ref{sect:high_probability}}

In this section, we present the proofs for the high probability estimation error bounds stated in Section~\ref{sect:high_probability}.

\subsection{Proof of Proposition~\ref{prop:uniform_stability_withoutreplacement}}
\label{apdsect:proof_uniform_stability_withoutreplacement}

In this subsection, we prove Proposition~\ref{prop:uniform_stability_withoutreplacement} as below restated .
\UniformStabilityWithoutReplacement*

We first need to show the following preliminary result which is about the expansion property of M-SPP update when performed over identical or different minibatches.
\begin{lemma}\label{lemma:stability_key}
Suppose that Assumptions~\ref{assump:smooth} holds and the loss function $\ell$ is bounded in the interval $[0,M]$. From $w_0=w'_0$, let us define the sequences $\{w_t\}_{t\in[T]}$ and $\{w'_t\}_{t\in [T]}$ that are respectively generated over $\{S_t\}_{t\in [T]}$ and $\{S'_t\}_{t\in [T]}$ according to
\[
\begin{aligned}
F_t(w_t) \le& \min_{w\in \mathcal{W}} \left\{F_t(w):=R_{S_t}(w) + \frac{\gamma_t}{2} \|w-w_{t-1}\|^2\right\} + \epsilon_t, \\
F'_{t}(w'_t) \le& \min_{w\in \mathcal{W}} \left\{F'_{t}(w):=R_{S'_t}(w) + \frac{\gamma_t}{2} \|w - w'_{t-1}\|^2\right\} + \epsilon_t.
\end{aligned}
\]
Assume that either $S_t=S'_t$ or $S_t \doteq S'_t$ for all $t\in [T]$. Let $\beta_t = \mathbf{1}_{\left\{S_t\neq S'_t\right\}}$. Then the following bound holds for all $t \in [T]$,
\[
\|w_t - w'_t\| \le \sum_{\tau=1}^t \left\{ \beta_\tau \frac{4\sqrt{2LM}}{n\gamma_\tau} + 2\sqrt{\frac{2\epsilon_\tau}{\gamma_\tau}} \right\}.
\]
\end{lemma}
\begin{proof}
Let $w^*_t=\argmin_w F_t(w)$ and $w'^{*}_t=\argmin_w F'_t(w)$. It follows from Lemma~\ref{lemma:function_strong_convexity} that
\[
\begin{aligned}
R_{S_t}(w^*_t) -  R_{S_t}(w'^*_t) \le& \frac{\gamma_t}{2}\left( \|w'^*_t - w_{t-1} \|^2 - \|w'^*_t - w^*_t\|^2 - \|w^*_t - w_{t-1}\|^2\right) \\
R_{S'_t}(w'^*_t) -  R_{S'_t}(w^*_t) \le& \frac{\gamma_t}{2}\left( \|w^*_t - w'_{t-1} \|^2 - \|w'^*_t - w^*_t\|^2 - \|w'^*_t - w'_{t-1}\|^2\right).
\end{aligned}
\]
Summing both sides of the above two inequalities yields
\[
\begin{aligned}
&R_{S_t}(w^*_t)-  R_{S_t}(w'^*_t) + R_{S'_t}(w'^*_t) -  R_{S'_t}(w^*_t)   \\
\le& \frac{\gamma_t}{2}\left( \|w'^*_t - w_{t-1} \|^2- \|w^*_t - w_{t-1}\|^2 + \|w^*_t - w'_{t-1} \|^2 - \|w'^*_t - w'_{t-1}\|^2- 2\|w'^*_t - w^*_t\|^2 \right) \\
=& \frac{\gamma_t}{2}\left( 2\langle w^*_t - w'^*_t, w_{t-1} - w'_{t-1} \rangle - 2\|w'^*_t - w^*_t\|^2 \right) \\
\le& \frac{\gamma_t}{2} \left(\|w_{t-1} - w'_{t-1}\|^2 - \|w^*_t - w'^*_t\|^2 \right)
\end{aligned}.
\]
We need to distinguish the following two complementary cases.

\textbf{Case I: $S_t = S'_t$.} In this case, the previous inequality immediately leads to
\[
\|w^*_t - w'^*_t\| \le \|w_{t-1} - w'_{t-1}\|.
\]
By using triangle inequality and Lemma~\ref{lemma:function_strong_convexity_inexact} we obtain
\begin{equation}\label{inequat:distance_expansion_1}
\|w_t - w'_t\| \le \| w_t -  w^*_t\| + \|w^*_t - w'^*_t\| + \| w'_t -  w'^*_t\| \le \|w_{t-1} - w'_{t-1}\| + 2\sqrt{\frac{2\epsilon_t}{\gamma_t}}.
\end{equation}

\textbf{Case II: $S_t$ and $S'_t$ differ in a single element.} In this case, we have
\[
\begin{aligned}
&\|w^*_t - w'^*_t\|^2 \\
\le& \|w_{t-1} - w'_{t-1}\|^2 + \frac{2}{\gamma_t} \left( R_{S_t}(w'^*_t)-  R_{S_t}(w^*_t) + R_{S'_t}(w^*_t) -  R_{S'_t}(w'^*_t) \right) \\
=& \|w_{t-1} - w'_{t-1}\|^2 + \frac{2}{\gamma_t} \left( R^\ell_{S_t}(w'^*_t)-  R^\ell_{S_t}(w^*_t) + R^\ell_{S'_t}(w^*_t) -  R^\ell_{S'_t}(w'^*_t) \right) \\
=& \|w_{t-1} - w'_{t-1}\|^2 + \frac{2}{\gamma_t} \left( \frac{1}{|S_t|} \sum_{z\in S_t} (\ell(w'^*_t;z) - \ell(w^*_t;z) +  \frac{1}{|S'_t|} \sum_{z\in S'_t} (\ell(w^*_t;z) - \ell(w'^*_t;z)) \right) \\
\le& \|w_{t-1} - w'_{t-1}\|^2 + \frac{4\sqrt{2LM}}{n\gamma_t} \|w^*_t - w'^*_t\|.
\end{aligned}
\]
where in the last inequality we have used $\ell(\cdot;\cdot)$ is $\sqrt{2LM}$-Lipschitz with respect to its first argument which is implied by Lemma~\ref{lemma:key_smooth}, and $S_t$ and $S'_t$ differ in a single element as well. Since $x^2 \le y^2 + a x $ implies $x \le y+ a$ for all $x,y,a>0$, we can derive from the above that
\[
\|w^*_t - w'^*_t\| \le \|w_{t-1} - w'_{t-1}\| + \frac{4\sqrt{2LM}}{n\gamma_t}.
\]
Then based on triangle inequality and Lemma~\ref{lemma:function_strong_convexity_inexact} we have
\begin{equation}\label{inequat:distance_expansion_2}
\begin{aligned}
\|w_t - w'_t\| \le \| w_t -  w^*_t\| + \|w^*_t - w'^*_t\| + \| w'_t -  w'^*_t\| \le \|w_{t-1} - w'_{t-1}\| + \frac{4\sqrt{2LM}}{n\gamma_t} + 2\sqrt{\frac{2\epsilon_t}{\gamma_t}}.
\end{aligned}
\end{equation}
Let $\beta_t = \mathbf{1}_{\left\{S_t\neq S'_t\right\}}$ where $\mathbf{1}_{\{C\}}$ is the indicator function of the condition $C$. Based on the recursion forms~\eqref{inequat:distance_expansion_1} and~\eqref{inequat:distance_expansion_2} and the condition $w_0=w'_0$ we can show that for all $t \in [T]$
\[
\|w_t - w'_t\| \le \sum_{\tau=1}^t \left\{ \frac{4\beta_\tau\sqrt{2LM}}{n\gamma_\tau} + 2\sqrt{\frac{2\epsilon_\tau}{\gamma_\tau}} \right\},
\]
which is the desired bound.
\end{proof}

Now we are in the position to prove the main result in Proposition~\ref{prop:uniform_stability_withoutreplacement}.
\vspace{0.1in}

\begin{proof}[Proof of Proposition~\ref{prop:uniform_stability_withoutreplacement}]
Consider a fixed pair of minibatch sets $S\doteq S'$.

\vspace{0.1in}

\textbf{Part (a):} Let $\{w_t\}_{t\in[T]}$ and $\{w'_t\}_{t\in [T]}$ be two solution sequences that are respectively generated over $\{S_t\}_{t\in [T]}$ and $\{S'_t\}_{t\in [T]}$ by Algorithm~\ref{alg:mspp}. At each time instance $t$, define random variable $\beta_t := \mathbf{1}_{\left\{S_t\neq S'_t\right\}}$. Since by assumption $S$ and $S'$ differ only in a single minibatch, there must exist one and only one $t\in [T]$ such that $\beta_t=1$ and $\beta_{j}=0$ for all $j\in [T], j\neq t$. Then in the worst case of $\beta_\tau=1$ for $\tau=\argmin_{i\in [t]} \gamma_i$, it follows from Lemma~\ref{lemma:stability_key} that for all $t\in [T]$,
\[
\|w_t - w'_t\| \le \frac{4\sqrt{2LM}}{n\min_{i\in [t]}\gamma_i} +\sum_{i=1}^t 2\sqrt{\frac{2\epsilon_i}{\gamma_i}} \le \frac{4\sqrt{2LM}}{n\min_{i\in [T]}\gamma_i} +\sum_{i=1}^T 2\sqrt{\frac{2\epsilon_i}{\gamma_i}}.
\]
Then the convex combination nature of $\bar w_T$ and $\bar w'_T$ implies that
\[
\|\bar w_T - \bar w'_T\| \le \frac{\sum_t \gamma_t \|w_t - w'_t\|}{\sum_{t} \gamma_t}\le \frac{4\sqrt{2LM}}{n \min_{t\in [T]}\gamma_t} +\sum_{t=1}^T 2\sqrt{\frac{2\epsilon_t}{\gamma_t}}.
\]
The desired result follows immediately as the above bound holds for any pair $\{S, S'\}$.

\vspace{0.1in}

\textbf{Part (b):} Recall that $\{\xi_t\}_{t\in [T]}$ are the uniform random indices for iteratively selecting data minibatches from $S$ and $S'$. Let $\{w_t\}_{t\in[T]}$ and $\{w'_t\}_{t\in [T]}$ be two solution sequences that are respectively generated over $\{S_{\xi_t}\}_{t\in [T]}$ and $\{S'_{\xi_t}\}_{t\in [T]}$ by Algorithm~\ref{alg:mspp_swr}. Define random variable $\beta_t := \mathbf{1}_{\left\{S_{\xi_t}\neq S'_{\xi_t}\right\}}$. Since by assumption $S$ and $S'$ differ only in a single minibatch, under without-replacement sampling scheme, there must exist one and only one $t\in [T]$ such that $\beta_t=1$ and $\beta_{j}=0$ for all $j\in [T], j\neq t$. Let us define the event $\mathcal{E}_t:=\{\beta_t=1 \text{ and } \beta_{j\neq t, j\in [T]}=0\}$ for all $t\in [T]$. Then the uniform randomness of $\xi_t$ implies that
\[
R\left( \mathcal{E}_t \right) = \frac{1}{T}, \quad t\in [T].
\]
Given $t\in [T]$, suppose that $\mathcal{E}_\tau$ occurs for some $\tau\in [t]$. Then it follows from Lemma~\ref{lemma:stability_key} that
\[
\|w_t - w'_t\| \le \frac{4\sqrt{2LM}}{n\gamma_\tau} +\sum_{i=1}^t 2\sqrt{\frac{2\epsilon_i}{\gamma_i}}.
\]
Suppose that $\mathcal{E}_\tau$ occurs for some $\tau\in \{t+1, t+2, ..., T\}$, again it follows from Lemma~\ref{lemma:stability_key} that
\[
\|w_t - w'_t\| \le \sum_{i=1}^t 2\sqrt{\frac{2\epsilon_i}{\gamma_i}}.
\]
Then we have
\[
\begin{aligned}
\mathbb{E}_{\xi_{[t]}}\left[\|w_t - w'_t\|\right] =& \sum_{\tau=1}^T R\left( \mathcal{E}_\tau \right) \left[\|w_t - w'_t\|\mid \mathcal{E}_\tau\right] \\
\le& \sum_{\tau=1}^t \left\{\frac{4\sqrt{2LM}}{nT\gamma_t} + \sum_{i=1}^t \frac{2}{T}\sqrt{\frac{2\epsilon_i}{\gamma_i}} \right\} + \sum_{\tau=t+1}^T \left\{\sum_{i=1}^t\frac{2}{T}\sqrt{\frac{2\epsilon_t}{\gamma_t}} \right\}\\
=& \sum_{\tau=1}^t \left\{ \frac{4\sqrt{2LM}}{nT\gamma_\tau} + 2\sqrt{\frac{2\epsilon_\tau}{\gamma_\tau}} \right\} \le \sum_{t=1}^T \left\{ \frac{4\sqrt{2LM}}{nT\gamma_t} + 2\sqrt{\frac{2\epsilon_t}{\gamma_t}} \right\}.
\end{aligned}
\]
It follows that
\[
 \mathbb{E}_{\xi_{[T]}}\left[\|\bar w_T - \bar w'_T\|\right] \le \frac{\sum_t \gamma_t \mathbb{E}_{\xi_{[t]}}\left[\|w_t - w'_t\|\right] }{\sum_{t} \gamma_t}\le \sum_{t=1}^T \left\{ \frac{4\sqrt{2LM}}{nT\gamma_t} + 2\sqrt{\frac{2\epsilon_t}{\gamma_t}} \right\}.
\]
The desired result follows immediately as the above bound holds for any pair $\{S, S'\}$.
\end{proof}

\subsection{Proof of Theorem~\ref{thrm:fast_rate_smooth_high_probability}}
\label{apdsect:proof_fast_rate_smooth_high_probability}

In this subsection, we prove Theorem~\ref{thrm:fast_rate_smooth_high_probability} that is restated below.
\FastRateSmoothHighProbability*

To show this result, we need to use the following restated McDiarmid's inequality~\citep{McDiarmid1989SurveysIC} which is also known as bounded difference inequality.
\begin{lemma}[McDiarmid's/Bounded differences inequality]
Let $X_1, X_2, ...,X_N$ be independent random variables valued in $\mathcal{X}$. Suppose that the function $h: \mathcal{X}^N \mapsto \mathbb{R}$ satisfies the bounded differences property, i.e., the following inequality holds for any $i\in [N]$ and any $x_1,...,x_N, x'_i$:
\[
|h(x_1,...,x_{i-1}, x_i, x_{i+1}, ..., x_N) - h(x_1,...,x_{i-1}, x'_i, x_{i+1},...,x_N)| \le c_i.
\]
Then for any $\varepsilon>0$,
\[
\mathbb{P}\left( h(X_1,...,X_N) -  \mathbb{E}\left[h(X_1,...,X_N) \right] \ge \varepsilon \right) \le \exp\left(-\frac{2\varepsilon^2}{\sum_{i=1}^N c_i^2}\right).
\]
\end{lemma}
Now we are ready to prove Theorem~\ref{thrm:fast_rate_smooth_high_probability}.

\vspace{0.1in}

\begin{proof}[Proof of Theorem~\ref{thrm:fast_rate_smooth_high_probability}]
Let $S=\{S_t\}_{t\in [T]}$ and $S'=\{S'_t\}_{t\in [T]}$ be two sets of data minibatches such that $S \doteq S'$. Then according to Proposition~\ref{prop:uniform_stability_withoutreplacement} the weighted average output $\bar w_T$ and $\bar w'_T$ respectively generated by Algorithm~\ref{alg:mspp_swr} over $S$ and $S'$ satisfy
\[
\sup_{S, S'} \mathbb{E}_{\xi_{[T]}}\left[\|\bar w_T - \bar w'_T\|\right] \le \sum_{t=1}^T \left\{ \frac{4\sqrt{2LM}}{nT\gamma_t} + 2\sqrt{\frac{2\epsilon_t}{\gamma_t}} \right\} \le \sum_{t=1}^T \left\{ \frac{5\sqrt{2LM}}{nT\gamma_t} \right\} \le \frac{20\sqrt{2LM}(1+\log(T))}{\lambda \rho nT},
\]
where in the last but one inequality we have used the condition $\epsilon_t \le \frac{LM}{4n^2T^2 \gamma_t}=\frac{LM}{\lambda\rho N^2 t}$. It follows from the triangle inequality and the above bound that
\[
\sup_{S, S'} \mathbb{E}_{\xi_{[T]}}\left[\left|D(\bar w_T,W^*) - D( \bar w'_T, W^*)\right| \right]\le\sup_{S, S'} \mathbb{E}_{\xi_{[T]}}\left[\|\bar w_T - \bar w'_T\|\right] \le \frac{20\sqrt{2LM}(1+\log(T))}{\lambda \rho n T}.
\]
Since $\xi_{[T]}$ are independent on $S$, as a direct consequence of applying McDiarmid's inequality with $c_i\equiv c = \frac{20\sqrt{2LM}(1+\log(T))}{\lambda \rho n T}$ to $h(S):=D(\bar w_T,  W^*)$, we can show that with probability at least $1-\delta$ over the randomness of $S$,
\[
\mathbb{E}_{\xi_{[T]}}\left[D(\bar w_T,  W^*) - \mathbb{E}_{S}\left[\mathbb{E}_{\xi_{[T]}}\left[D(\bar w_T,  W^*)\right] \right]\right] \le c\sqrt{\frac{nT\log(1/\delta)}{2}} = \frac{20\sqrt{LM\log(1/\delta)}(1+\log(T))}{\lambda \rho \sqrt{nT}}.
\]
We next derive a bound for $\mathbb{E}_{S}\left[D(\bar w_T,  W^*)\right]$. In view of Jensen's inequality and the quadratic growth property of $F$ we have
\[
\begin{aligned}
\mathbb{E}_{S}\left[\mathbb{E}_{\xi_{[T]}}\left[D(\bar w_T,  W^*)\right]\right] =&\mathbb{E}_{\xi_{[T]}}\left[\mathbb{E}_{S}\left[D(\bar w_T,  W^*)\right]\right] \\
\le& \mathbb{E}_{\xi_{[T]}}\left[\sqrt{\mathbb{E}_{S}\left[D^2(\bar w_T,  W^*)\right]}\right] \\
\le& \mathbb{E}_{\xi_{[T]}}\left[\sqrt{\frac{2}{\lambda}\mathbb{E}_{S}\left[R(\bar w_T) - R^* \right]}\right] \\
\lesssim & \mathbb{E}_{\xi_{[T]}}\left[\sqrt{\frac{\rho\left[R(w_0) - R^*\right]}{\lambda T^2} + \frac{L}{\lambda^2 \rho nT} R^* + \frac{\sqrt{\epsilon}}{\lambda T^2}\left(\frac{L}{\lambda\rho} + G\sqrt{\frac{1}{\lambda\rho}}\right)}\right]\\
=&\sqrt{\frac{\rho\left[R(w_0) - R^*\right]}{\lambda T^2} + \frac{L}{\lambda^2 \rho nT} R^* + \frac{\sqrt{\epsilon}}{\lambda T^2}\left(\frac{L}{\lambda\rho} + G\sqrt{\frac{1}{\lambda\rho}}\right)},
\end{aligned}
\]
where in the last inequality we have invoked Theorem~\ref{thrm:fast_rate_smooth_inexact}. Therefore, based on the previous two inequalities we obtain that with probability at least $1-\delta$ over $S$,
\[
\begin{aligned}
&\mathbb{E}_{\xi_{[T]}}\left[D(\bar w_T,  W^*) \right] \\
\lesssim& \frac{\sqrt{LM\log(1/\delta)}\log(T)}{\lambda \rho \sqrt{nT}} + \sqrt{\frac{\rho\left[R(w_0) - R^*\right]}{\lambda T^2} + \frac{L}{\lambda^2 \rho nT} R^* + \frac{\sqrt{\epsilon}}{\lambda T^2}\left(\frac{L}{\lambda\rho} + G\sqrt{\frac{1}{\lambda\rho}}\right)},
\end{aligned}
\]
which gives the desired bound.
\end{proof}

\subsection{Proof of Theorem~\ref{thrm:fast_rate_smooth_convex_high_probability}}
\label{apdsect:proof_fast_rate_smooth_convex_high_probability}

Here we prove the following restated Theorem~\ref{thrm:fast_rate_smooth_convex_high_probability}.
\FastRateSmoothConvexHighProbability*

We need the following lemma essentially from~\citet[Corollary 8]{bousquet2020sharper} that gives a near-tight generalization bound for a learning algorithm that is uniformly stable with respect to loss function.

\begin{lemma}[\citet{bousquet2020sharper}]\label{lemma:sharper_generalization}
Suppose that a learning algorithm $\mathcal{A}_w$, parameterized by $w$, satisfies $|\ell(\mathcal{A}_{w_S}(x),y) - \ell(\mathcal{A}_{w_{S'}}(x),y)| \le \varrho$ for any $(x,y)\in \mathcal{X}\times \mathcal{Y}$ and $S \doteq S'$. Assume the loss function satisfies $0\le \ell(y',y) \le M$ for all $y,y'$.
Then for any $\delta\in(0,1)$, with probability at least $1-\delta$ over an i.i.d. data set $S$ of size $N$,
\[
\left|R(\mathcal{A}_{w_S}) - R_S(\mathcal{A}_{w_S})\right| \lesssim \varrho \log(N) \log\left(\frac{1}{\delta}\right) + M\sqrt{\frac{\log\left(1/\delta\right)}{N}}.
\]
\end{lemma}

With this lemma in place, we can prove the main result in Theorem~\ref{thrm:fast_rate_smooth_convex_high_probability}
\begin{proof}[Proof of Theorem~\ref{thrm:fast_rate_smooth_convex_high_probability}]
Let $S=\{S_t\}_{t\in [T]}$ and $S'=\{S'_t\}_{t\in [T]}$ be two sets of data minibatches satisfying $S \doteq S'$. Note that $\gamma_t\equiv\gamma = \sqrt{\frac{T}{n}}$. Then according to Proposition~\ref{prop:uniform_stability_withoutreplacement} the average output $\bar w_T$ and $\bar w'_T$ respectively generated by Algorithm~\ref{alg:mspp} over $S$ and $S'$ satisfy
\[
\sup_{S, S'} \|\bar w_T - \bar w'_T\| \le \frac{4\sqrt{2LM}}{n\gamma} +\sum_{t=1}^T 2\sqrt{\frac{2\epsilon_t}{\gamma}} \le \frac{5\sqrt{2LM}}{n\gamma} = \frac{5\sqrt{2LM}}{\sqrt{nT}}.
\]
where in the last but one inequality we have used the condition $\epsilon_t\le \frac{LM}{4nT^2\sqrt{N}}$. It follows that
\[
|\ell(\bar w_T;z) - \ell(\bar w'_T;z)| \le \sqrt{2ML}\|\bar w_T- \bar w'_T\| \le \frac{10LM}{\sqrt{nT}},
\]
where we have used $\ell(\cdot;\cdot)$ is $\sqrt{2LM}$-Lipschitz with respect to its first argument (which is implied by Lemma~\ref{lemma:key_smooth}). In view of Assumption~\ref{assump:lipschitz} we have
\[
|r(\bar w_T) - r(\bar w'_{T})| \le G \|\bar w_T - \bar w'_T\| \le \frac{5G\sqrt{2LM}}{\sqrt{nT}}.
\]
This preceding two inequalities indicate that M-SPP is $\frac{10 LM+5G\sqrt{2LM}}{\sqrt{nT}}$-uniformly stable with respect to the composite loss function $\ell+r$. By invoking Lemma~\ref{lemma:sharper_generalization} to M-SPP we obtain that
\[
\left|R(w_S) - R_S(w_S)\right| \lesssim \frac{(LM + G \sqrt{LM})\log(nT)}{\sqrt{nT}} \log\left(\frac{1}{\delta}\right) + M\sqrt{\frac{\log\left(1/\delta\right)}{nT}}.
\]
The proof is concluded.
\end{proof}

\bibliography{refs_scholar}
\bibliographystyle{plainnat}

\end{document}